\documentclass[11pt,draftclsnofoot,twoside,onecolumn,a4paper]{IEEEtran}
\IEEEoverridecommandlockouts
\usepackage{graphicx}
\usepackage{epstopdf}
\DeclareGraphicsExtensions{.eps,.pdf} 
\graphicspath{ {figures/} }
\usepackage{cite}
\usepackage{subfigure}
\usepackage{amssymb,amsmath}
\usepackage{algorithm}
\usepackage{algorithmic}
\usepackage{color}
\usepackage{multirow}
\usepackage{multicol}
\usepackage{cases}
\usepackage{setspace}
\usepackage{amsthm}
\usepackage{caption}
\usepackage{balance}
\usepackage[mathscr]{euscript}
 \let\mathscr\relax
\usepackage[scr]{rsfso}

\allowdisplaybreaks

\newtheorem{theorem}{\bf {Theorem}}

\newtheorem{remark}{{\bf{Remark}}}
\newtheorem{definition}{\bf {Definition}}
\newtheorem{lemma}{\bf {Lemma}}
\newtheorem{assumption}{\bf {Assumption}}

\newtheorem{proposition}{{\bf Proposition}}

\newcommand{\SNR}{{\mathsf{SNR}}}

\newcommand{\st}{{\mathrm{s.t.}}}

\newcommand\argmin{\operatornamewithlimits{argmin}}

\renewcommand{\algorithmicrequire}{\textbf{Input:}}

\newcommand{\ds}{\displaystyle}

     \usepackage[compact]{titlesec}
    \titlespacing{\section}{0pt}{2ex}{1ex}
    \titlespacing{\subsection}{0pt}{1ex}{0ex}
    \titlespacing{\subsubsection}{0pt}{0.5ex}{0ex}

\newcommand{\ul}{\mathtt{ul}}
\newcommand{\dl}{\mathtt{dl}}

\newcommand{\FedAvg}{\mathsf{FedAvg}}
\newcommand{\FedFog}{\mathsf{FedFog}}

\usepackage{capt-of}
\usepackage{dblfloatfix}
\usepackage{varwidth}
\makeatletter
\g@addto@macro\normalsize{%
 \setlength\abovedisplayskip{2pt}
 \setlength\belowdisplayskip{2pt}
 \setlength\abovedisplayshortskip{2pt}
 \setlength\belowdisplayshortskip{2pt}
}

\usepackage{titlesec}
\titlespacing{\section}{0pt}{2pt}{0pt}

\begin{document}
\bstctlcite{IEEEexample:BSTcontrol}

\title{\huge{FedFog: Network-Aware Optimization of  Federated Learning over Wireless Fog-Cloud Systems}}
\author{
	\IEEEauthorblockN{Van-Dinh Nguyen,  Symeon Chatzinotas,  Bj$\ddot{\text{o}}$rn Ottersten, and Trung Q. Duong
 	\vspace{-50pt}
}
\\
\thanks{V.-D. Nguyen, S. Chatzinotas, and B. Ottersten are with the Interdisciplinary Centre for Security, Reliability and Trust (SnT), University of Luxembourg, L-1855 Luxembourg City,
Luxembourg (e-mail: \{dinh.nguyen, symeon.chatzinotas, bjorn.ottersten\}@uni.lu). }
\thanks{T. Q. Duong  is with the School of Electronics, Electrical Engineering
and Computer Science, Queen's University Belfast, Belfast BT7 1NN,
United Kingdom (e-mail: trung.q.duong@qub.ac.uk).}
\thanks{This work was supported in part by the ERC AGNOSTIC project, ref. H2020/ERC2020POC/957570.}
	}

\maketitle

\begin{abstract}

Federated learning (FL) is capable of performing large distributed machine learning tasks across multiple edge users by periodically aggregating trained local parameters. {\color{black}To address key challenges of enabling FL over a wireless fog-cloud system (e.g., non-i.i.d. data, users' heterogeneity), we first propose an efficient FL algorithm based on  Federated Averaging (called $\FedFog$) to perform the local aggregation of gradient parameters at fog servers and global training update at the cloud}. Next, we employ $\FedFog$ in wireless fog-cloud systems by investigating a novel network-aware FL optimization problem that strikes the balance between the global loss and completion time. An iterative algorithm is then developed to obtain a precise measurement of the system performance, which helps design an efficient stopping criteria to output an appropriate number of global rounds. To mitigate the straggler effect, we propose a flexible user aggregation strategy that trains fast users first to obtain a certain level of accuracy before allowing slow users to join the global training updates. Extensive numerical results using several real-world FL tasks are provided to verify the theoretical convergence of $\FedFog$. We also show that the proposed co-design of FL and communication is essential to substantially improve resource utilization while achieving comparable accuracy of the learning model.
\end{abstract}

\section{Introduction} \label{Introduction} 
Nowadays, Internet-connected devices are often equipped with advanced sensors that allow them to collect and store large amounts of data locally. This combined with the high computing capability of edge devices promotes the fog-cloud computing paradigm which brings data processing, storage, and intelligent control to the proximity of the network edge \cite{ChiangIoT16}. Besides, emerging technologies (e.g., autonomous driving, industry automation) are  relying heavily on data-driven machine learning (ML) approaches to enable near real-time applications \cite{Cisco2019Whitepaper}. However, traditional ML models, which require that all local data is sent to a centralized server for model training, may not be practical due to high  \textit{round-trip}   delays, energy constraints and privacy-sensitive concerns of edge devices. Fortunately, distributed ML is practically  suited for the fog-cloud computing, which aims at leveraging the advantages of the increasing storage and computing capabilities of edge devices to train ML models  while keeping device datasets local.

Federated learning (FL) is an emerging distributed ML framework that can address many challenges in implementing ML over networks \cite{ParkProIEEE19,park2020communicationefficient}. The most widely used and effective FL algorithm  is  Federated Averaging ($\FedAvg$) \cite{McMahan2017}. The FL optimization is commonly solved by an iterative procedure, where each iteration includes local training update and global aggregation. In particular, in each global round, edge devices compute local updates based on their available datasets, typically using gradient descent methods, which are sent back to the sever for global aggregation. Then, the server updates the new global model and broadcasts it to all devices to start the next global round of training. The main advantage of $\FedAvg$ over traditional distributed ML algorithms is that each device runs a series of local updates before communicating with the server. This process results  in less global updates and reduced communication costs \cite{McMahan2017,konevcny2016federated,wang2020local}. 

\subsection{Related Works}
In this section, we focus on  the literature review  of  FL algorithm and wireless FL performance optimization. 

\noindent\textbf{FL and challenges}. Inspired by $\FedAvg$, FL has attracted considerable attention in recent years in the ML community (see a comprehensive survey in \cite{kairouz2019advances}). Several works have attempted  to address the main challenges of FL (e.g., non independently and identically distribute (non-i.i.d.) data among devices and resource constraints) for improving communication-efficiency \cite{StichICLR2019,ZhangJMLR13,WangJSAC2019}, incentive  mechanism \cite{KangIoT19}, privacy-preserving \cite{DuchiACM18,mcmahan2018learning}, and guaranteeing fairness \cite{LiICLRFair2020} and robustness \cite{xie2019slsgd}. However, these works mainly focused on characterizing and optimizing the FL performance on over-simplified and/or unrealistic communication models, and the impact of wireless factors on FL is often not taken into account.

\noindent\textbf{Wireless FL}. Recently, there has been an increasing effort in designing communications protocols and computational aspects of the FL implementation in wireless networks. To improve the convergence speed of FL algorithms,   scheduling policies were proposed in \cite{YangTCOM2020}, where only a portion of users (UEs) are  scheduled for updates at each global round. The authors in \cite{ZhengJSAC20} investigated the impact of  various quantization and transmission methods for wireless FL. Various join FL model and radio resource allocation schemes have been proposed in \cite{MChenFL2019,DinhFL2019,KYang2018} to minimize either the global loss function or the training time. While many works have focused on  efficient wireless communications between the server and UEs to support FL \cite{VuCellfreeML2019,NguyenIoTFL2020,ZYang2019}, it is still challenging to employ them in distributed environments due to the causal setting (i.e., the loss value and associated  costs of future rounds are not available in advance). Notably, Mahmoudi \textit{et al. }\cite{MahmoudiICC20} developed an iterative distributed algorithm which characterizes the end-to-end delay as the per-iteration cost. {\color{black} Liu \textit{et al. }\cite{LuminICC2020}  proposed a client-edge-cloud hierarchical FL algorithm, called HierFAVG, which allows  performing partial model aggregation at edge servers to  improve communication-efficiency. However, the realistic cost function taking into account completion time was not considered in these works, and also,  computation and wireless factors (i.e., the  transmit  power,  UEs’  CPU  clock  speed  and  bandwidth  coefficients) were not jointly optimized.}

\noindent\textbf{Hierarchical FL-supported  fog-cloud networks}. Despite its  potential, there have been only a few attempts to improve the resource utilization  of FL-supported  fog-cloud networks in the literature. Specifically, the authors in \cite{HosseinalipourComMag20} recently introduced a new architecture, called Fog Learning (FogL), which leverages the multi-layer network structure of fog computing to handle ML tasks. This work was extended in \cite{hosseinalipour2020multistage} to develop a multi-stage hybrid model training, which incorporates  multi-stage parameter relaying across networks layers. The work in \cite{SahaIoT20} proposed $\mathsf{FogFL}$ to reduce energy consumption and communication latency, where  fog nodes act as local aggregators to share location-based information for applications with the similar environment. By taking into account both the computational and communication costs, it was shown in \cite{TuINFOCOM20} that the network-aware optimization greatly reduces the cost of model training. {\color{black}Very recently, the authors in \cite{JietPDS21} and \cite{SiqiTWC20} studied a general cost optimization of energy consumption and delay minimization within one global iteration. However, the model training was not jointly optimized in the global cost minimization. Wen \textit{et al.} \cite{WanliTWC22} proposed a joint  design of the scheduling and resource allocation scheme in an hierarchical federated edge learning, allowing a subset of helpers to upload their updated gradients in each round of the model training. A privacy-preserving FL in fog computing was also proposed in \cite{ZhouIoTJ}, requiring each device to meet different privacy  to resist data attacks. Here our focus is the effects of communication and edge devices' computation capability  in  realistic causal settings.}

\noindent\textbf{Straggler effect in FL} is a major bottleneck in implementing FL over wireless networks, i.e., when a user has poor channel quality and significantly low computation capability, resulting in higher training time. The promising approaches to  mitigating the straggler effect include user sampling \cite{LiICLR2020,McMahan2017,NguyenIoTFL2020} and user selection \cite{ChenTWC2020,VuCellfreeML2020,XiaTWCFL2020}, which require only a subset  of users to participate in the training process. However, to the best of our knowledge, these works neither consider a co-design between model training and communication nor simultaneously minimize  associated training costs.

\subsection{Research Gap and Main Contributions}
{\color{black}Despite the potential benefits offered by FL, there are still several inherent challenges in implementing hierarchical FL over  wireless fog-cloud networks, including but not limited to high communication costs, heterogeneity of edge devices (both datasets and computational capabilities), limited wireless resources and  straggler effects. Though in-depth results of optimizing communication for FL were presented in \cite{YangTCOM2020,ZhengJSAC20,MChenFL2019, DinhFL2019,KYang2018,VuCellfreeML2019,NguyenIoTFL2020,ZYang2019,VuCellfreeML2020,JietPDS21,SiqiTWC20}, they are not very practical for the actual implementation because the unique characteristics of the federated environment have not been fully addressed in these works. Moreover, it is often considered that the communication and model training are optimized separately  \cite{VuCellfreeML2020,VuCellfreeML2019,DinhFL2019,NguyenIoTFL2020,LuminICC2020,JietPDS21,SiqiTWC20}. Joint learning and communication \cite{MChenFL2019} is done in the sense  that all  local losses are available at the server in advance, which violates the FL principle. We show in this paper that communication and FL model training should be optimized on different time scales in each global round. In \cite{VuCellfreeML2019,DinhFL2019,ZYang2019,MChenFL2019,NguyenIoTFL2020}, the upper bound on the convergence of FL algorithms is characterized by providing the trade-off between the convergence rate and  number of global aggregations, which is commonly known in  ML literature, but not taking into the cost of model training. Although the higher the number of global aggregations, the lower the training loss that can be obtained, the associated cost increases significantly. This phenomenon promotes  a co-design of hierarchical FL and communication that strikes a good balance between the accuracy of the learning model and the running cost. In particular, the co-design should provide an adequate number of global rounds with minimal completion time while still guaranteeing the comparable accuracy of the FL model.

In this paper, we propose a novel network-aware optimization framework to enable hierarchical FL over a cloud-fog system, taking into account all the issues mentioned previously. In the considered system, the cloud server (CS) and fog servers (FSs) do not have access to the UEs' local datasets, thus preserving data privacy. The main goal is to minimize the global loss function and  completion time in a single framework, two prime objectives in  FL algorithms, which are conflicting. A direct application of an offline algorithm to solve such a problem is inapplicable as it requires complete information at  the beginning of the training process, which is impractical in federated settings. Towards a  realistic causal setting, we decompose the network-aware optimization problem into two sub-problems, namely the hierarchical FL  and  resource allocation, which are executed in different time slots. The $\FedFog$ is capable of performing a flexible user aggregation which allows fewer UEs to participate in the training process in each round, resulting in low completion time. The convergence of $\FedAvg$ with partial user participation has been studied in \cite{LiICLR2020}. Our main contributions are summarized as follows:}
\begin{itemize}
    {\color{black}\item We first propose an efficient FL algorithm for a fog-cloud system (called $\FedFog$) based on  $\FedAvg$ framework \cite{McMahan2017}. Acting as a participant, each FS plays the role of a local aggregator that collects the  local gradient parameters trained at UEs and then forwards them to CS for global training update, reducing network traffic of  backhaul links between FSs and CS. Compared to \cite{LuminICC2020}, we explicitly provide the new convergence upper bound of $\FedFog$  with a learning rate decay, taking into account  non-i.i.d. data and  stochastic noise of the random sampling of mini-batchs. 
    \item We formulate a novel network-aware FL  problem for  wireless fog-cloud systems by jointly optimizing the transmit power, UEs' CPU clock speed and bandwidth coefficients
    whose goal is to minimize  the global loss and completion time in a single framework while meeting UEs' energy constraints. To solve the resource allocation sub-problem of of join computation and communication, we develop a simple yet efficient path-following procedure based on inner approximation (IA) framework \cite{Marks:78}, in which newly convex approximated functions are derived to tackle nonconvex constraints.
    \item We characterize the discrete convex property of the general cost function to design a stopping criteria to produce  a desirable number of global rounds without an additional cost. We then propose the network-aware  optimization  algorithm that solves the FL and communication problems in a distributed fashion.}
   \item To further mitigate the straggler effect, we relax the objective function to   favor strong UEs. Then, we determine a time threshold that allows collecting only local  gradient  parameters  of strong  UEs. Once a certain accuracy level is reached, the time threshold is increased to allow slower UEs to join the training process, thereby  reducing  completion time without compromising the learning accuracy.
   \item We empirically evaluate the performance of the  proposed algorithms using  real  datasets. The results show that the  $\FedFog$-based algorithms can improve network resource utilization while achieving good performance in terms of convergence rate and  accuracy of the learning model.
\end{itemize}

\textit{Paper Organization and Mathematical Notations:}
The rest of this paper is organized as follows. Preliminaries and definitions are described in Section \ref{PreliminariesDefinitions}. The proposed
$\FedFog$ and the expected convergence rate are given in Section \ref{sec:FedFog}. The network-aware optimization algorithms are presented in Section \ref{sec:FedFogOpti}. Numerical results are analyzed in Section \ref{sec:Numericalresults}, while  conclusions are draw in Section \ref{sec:Conclusion}.
The main notations and symbols are summarized in Table \ref{table:notation}.
\begin{table}[t]
\centering
	\captionof{table}{Summary of Main Notations and Symbols}
	\label{table:notation}
 	\fontsize{11}{10}\selectfont
	\vspace{-5pt}
	\scalebox{0.85}{
		\begin{tabular}{l|l}
			\hline
			$F(\mathbf{w})$ \& $F_{ij}(\mathbf{w})$ & Global loss function and local loss function of UE $(i,j)$ \\
			$\mathbf{w}^g$ \& $\mathbf{w}^g_{ij,\ell}$ &  Global model at round $g$ and local model of UE $(i,j)$ \\ 
			& at global round $g$ and local iteration $\ell$\\
			$\mathcal{D}_{ij}$ \& $D_{ij}$ & Set of data samples and number of samples of UE $(i,j)$ \\
			$\mathcal{I}$ \& $\mathcal{J}$ & Sets of fog nodes and UEs, respectively \\
			$\mathcal{J}_i$ \& $J_i$ & Set and number of UEs at fog node $i$\\
			$\mathcal{G}$ \& $\mathcal{L}$ & Sets of global rounds and local iterations, respectively\\
			$\mathcal{S}(g)$ \& $S(g)$ & Set and number of UEs selected at round $g$\\
			$g$ \& $\ell$ & Indexes of global round and local iteration \\
			$\mathcal{B}_{ij,\ell}^g$ & The mini-batch with size $B$ \\
			$\eta^g$& The step size (learning rate) at global round $g$\\
			$\mathbf{w}^*$ & Optimal global model\\
			$p_{ij}(g)$ & The transmit power coefficient of UE $(i,j)$ at  round $g$\\
			$\mathsf{f}_{ij}(g)$  & The CPU clock speed of UE $(i,j)$ at  round $g$\\
			$\beta_{ij}(g)$ & The UL bandwidth fraction allocated  UE $(i,j) $ at  round $g$\\
			\hline
			$\left<\mathbf{x},\mathbf{y}\right>$ & Inner product of  vectors $\mathbf{x}$ and $\mathbf{y}$ \\
			$\|\mathbf{x}\|$ \& $|z|$ &  Euclidean
			norm of vector $\mathbf{x}$ and absolute value of $z$, respectively \\
			$\mathbf{h}^{H}$  & Hermitian transpose of vector $\mathbf{h}$\\
			$\nabla F(\cdot)$ & Gradient of the function $F(\cdot)$\\
			$\mathbb{E}\{\cdot\}$ & Expectation of a random variable\\
			$\mathcal{CN}(0,\sigma^2)$ & Circularly-symmetric complex Gaussian random variable \\
			& with zero mean and variance  $\sigma^2$\\
			$\mathbb{C}$ \& $\mathbb{R}$ & Sets of complex and real numbers, respectively\\
			\hline		   				
		\end{tabular}
	}
\end{table}

\section{Preliminaries and Definitions} \label{PreliminariesDefinitions}
\subsection{Wireless Fog-Cloud Computing Model}
\begin{figure}[!ht]
	\centering
	\includegraphics[width=0.6\columnwidth,trim={0cm 0.0cm 0cm 0.0cm}]{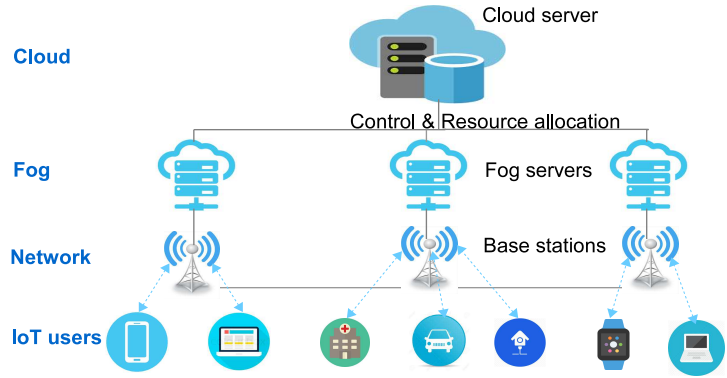}
		\vspace{-2pt}
	\caption{A generic architecture of wireless fog-cloud systems.}
	\label{fig:systemmodel}
\end{figure}

A generic fog-cloud computing architecture  consists of three layers \cite{ZhouIoTJ,LeeTWC2019}, as illustrated in Fig. \ref{fig:systemmodel}.
\begin{itemize}
    \item \textit{Cloud layer} contains large-scale cloud data centers (CDCs) equipped with powerful processing units,  providing off-premise computing services to IoT users (or UEs for short). In the context of FL, CS is mainly responsible to collect  local models generated by  IoT UEs to produce a new global model, which is then sent back to them through FSs to start a new training  round.
    \item \textit{Fog layer} comprises a set of FSs  deployed close to IoT UEs that can perform local data processing services. Each FS acts as a local aggregator to exchange the model between CS and UEs, which is connected to a base station (BS) via the wired backhaul link, while BS communicates to its UEs through wireless links.
    \item \textit{User layer} deploys a large number of IoT UEs  
 collected measurement data from the environment that is used to train ML algorithms. Since IoT UEs are placed in the vicinity of their FSs, allowing each UE to transmit the trained parameters to its FS for local aggregations that can obtain low-latency high quality-of-experience (QoE) of IoT UEs. 
\end{itemize}

We consider a fog-cloud network consisting of a set $\mathcal{I}\triangleq\{1,2,\cdots,I\}$ of $I$ fog servers (each fog server already associated with one BS) and a set $\mathcal{J}\triangleq\{1,2,\cdots,J\}$ of $J$ UEs. We assume that FS  $i\in\mathcal{I}$ serves a separate set $\mathcal{J}_i\triangleq\{1,2,\cdots, J_i\}$ of $J_i$ UEs with $J=\sum_{i\in\mathcal{I}}J_i$, and each UE is associated with one FS only. In addition,  \textit{one learning process} (i.e., the entire implementation of an $\FedFog$ algorithm until convergence) requires $G$ global rounds, each with the same number of local updates, $L$. Let us denote  $\mathcal{G}\triangleq\{0,1,\cdots,G-1\}$ and $\mathcal{L}\triangleq\{0,1,\cdots,L-1\}$ as the sets of $G$ global rounds and $L$ local iterations, respectively. The number of global and local iterations may depend on the specific ML application. 

 We denote  user $j\in\mathcal{J}_i$ associated with FS $i\in\mathcal{I}$ by UE ($i,j$), which collects a local input data set $\mathcal{D}_{ij}$ of $D_{ij} = |\mathcal{D}_{ij}|$ data samples. Considering non-i.i.d. distributed data across the network, we assume that $\mathcal{D}_{ij}\cap \mathcal{D}_{i'j'} =\emptyset,\ \forall (i,j)\neq (i',j')$. Each element $\mathbf{x}_d\in\mathcal{D}_{ij}$  is  an input sample vector with $q>1$ features. In a typical learning algorithm, we use $(\mathbf{x}_d,y_d)$ to express the data sample $d$, where $y_{d}\in\mathbb{R}$ is the output (label) for the sample $\mathbf{x}_{d}$. 

\begin{remark}
{\color{black}In general, it is possible for each UE to forward its local model to more than one FSs to reduce the UL communication delay. However, this will likely lead to a case where  some UEs may not be willing to share their proprietary information (i.e., computational capability and battery level) to strange FSs that will be used to optimize  system performance. In the proposed scheme, FS $i$ can be seen as a trusted fog server to a group of $J_i$ UEs based on their  \textit{prior} agreement, which further alleviates  privacy concerns of data sharing. In many FL applications (e.g., healthcare industry, FinTech, insurance sector and IoT), FSs are often deployed by private organizations to keep their privacy and data preserved. }
\end{remark}

\subsection{Federated Learning Model}\label{Sec:PreliminariesDefinitionsFL}
\begin{definition}
Throughout the paper, the model produced by FSs and UEs is referred to as ``local aggregation model'' and ``local model,'' respectively, while that averaged at CS is called ``global  model.''
\end{definition}
\noindent\textbf{Local loss function:} Following
the commonly used FL framework \cite{McMahan2017,park2020communicationefficient}, the main goal  is to jointly learn the \textit{global model parameter} $\mathbf{w}\in\mathbb{R}^{q}$ (e.g., a neural network or support vector machine \cite{ShwartzCUP2014,NguyenIoTFL2020}) that
produces  the output $y_d$ given the input sample $\mathbf{x}_d$ through the local loss function $f(\mathbf{w}, \mathbf{x}_d,y_d)$ \cite{park2020communicationefficient}. On the  local data set, the loss function at UE $(i,j)$  can be generally defined as
\begin{IEEEeqnarray}{rCl}\label{eq:locallossfuntions}
F_{ij}(\mathbf{w}|\mathcal{D}_{ij}) &\triangleq& \frac{1}{D_{ij}}\sum_{d\in\mathcal{D}_{ij}}f(\mathbf{w}, \mathbf{x}_d,y_d).
\end{IEEEeqnarray}

\noindent\textbf{The federated learning problem:} Since the overall data distributions on UEs are unknown, we consider the empirical loss function across the entire network data set $\mathcal{D} = \cup_{ i,j}\mathcal{D}_{ij}$, defined as
\begin{IEEEeqnarray}{rCl}\label{eq:totalloss1}
F(\mathbf{w}|\mathcal{D})\triangleq \frac{\sum_{i\in\mathcal{I}}\sum_{j\in\mathcal{J}_i}F_{ij}(\mathbf{w}|\mathcal{D}_{ij})}{J}.
\end{IEEEeqnarray}
The aim of the common FL algorithms is to find the optimal model $\mathbf{w}^*$ that minimizes the global loss value of the following  optimization problem:
\begin{equation}\label{eq:globalmini}
    \mathbf{w}^* = \arg\min_{\mathbf{w}\in\mathbb{R}^q} F(\mathbf{w}|\mathcal{D}).
\end{equation}
To achieve this in a distributed fashion, problem \eqref{eq:globalmini} is separately decomposed into $J$ independent sub-problems that can be solved locally at UEs. Here, we  directly adopt $\FedAvg$ \cite{McMahan2017}  which iteratively minimizes the local loss \eqref{eq:totalloss1} using gradient descent technique before communicating with CS for global update. The key steps of $\FedAvg$ training procedure used for the considered  fog-cloud system can be summarized as follows:
\begin{enumerate}
    \item \textbf{Global model downloading}: At the start of global round $g$, CS broadcasts the latest global model $\mathbf{w}^g$ to all FSs, and then  FS $i$ broadcasts $\mathbf{w}^g$ to $J_i$ UEs, $\forall i\in\mathcal{I}$.
    \item \textbf{Local training update}: UE $(i,j)$ sets  $\mathbf{w}^g_{ij,0}:=\mathbf{w}^g$ and then updates the local parameters for $L>1$ local iterations as
    \begin{equation}\label{eq:localtraiingupdates}
        \mathbf{w}^g_{ij,\ell+1} := \mathbf{w}^g_{ij,\ell} - \eta^g\nabla F_{ij}(\mathbf{w}^g_{ij,\ell}|\mathcal{D}_{ij}),\ \ell = 0,1,\cdots, L-1
    \end{equation}
    where $\eta^g > 0$ is the step size (learning rate), which is often decreased over time.
    \item \textbf{Local model uploading}: UE $(i,j)$ with $j\in\mathcal{J}_i$ sends $\mathbf{w}^g_{ij,L}$ back to FS $i, \forall i\in\mathcal{I}$, and then FS $i$ forwards  $\mathbf{w}^g_{ij,L}$ to CS for averaging.
    
    \item \textbf{Global training update at CS}: CS performs global training update as
    \begin{equation}
        \mathbf{w}^{g+1} := \frac{\sum_{i\in\mathcal{I}}\sum_{j\in\mathcal{J}_i}\mathbf{w}^g_{ij,L}}{J}.
    \end{equation}
    
    \item Set $g:=g+1$ and repeat Steps 1-4  until convergence.
\end{enumerate}

Compared to prior works \cite{DinhFL2019,MChenFL2019,ZYang2019,VuCellfreeML2019,KYang2018,NguyenIoTFL2020}, one of the key challenges here is how  FSs can efficiently convey trained local models  to CS while ensuring convergence. We note that  forwarding  all local models from FSs to CS (e.g., \cite{VuCellfreeML2019,VuCellfreeML2020} in the context of cell-free networks) will increase the backhaul overhead traffic , which becomes prohibitive in a large-scale system. The proposed approach for local aggregations will be discussed in the sequel. 

\section{$\FedFog$: Proposed  Federated Learning Algorithm Design}\label{sec:FedFog}
We first develop the $\FedFog$ algorithm (Section \ref{sub:FedFogAlg}) and then provide a detailed convergence analysis (Section \ref{sub:FedFogConver}). 

\subsection{Proposed $\FedFog$ Algorithm Design}\label{sub:FedFogAlg}
\begin{figure}[!ht]
	\centering
	\includegraphics[width=0.8\columnwidth,trim={0cm 0.0cm 0cm 0.0cm}]{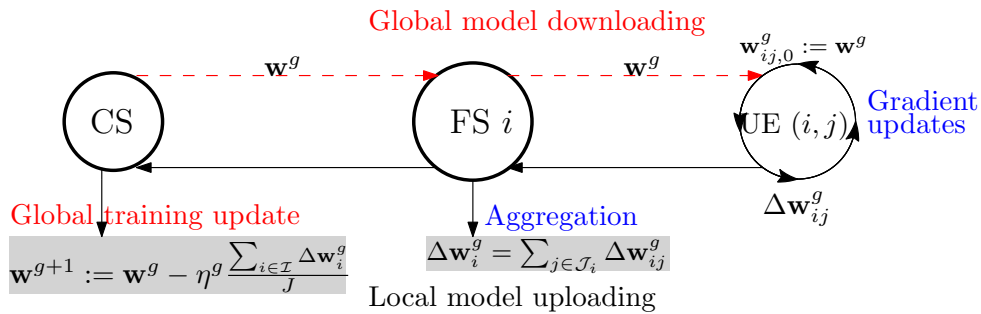}
	\vspace{-2pt}
	\caption{One $\FedFog$ update between global rounds $g$ and $g+1$. }
	\label{fig:FLprocess}
\end{figure}

Similar to  the extension of $\FedAvg$ to  the  fog-cloud system presented in Section \ref{Sec:PreliminariesDefinitionsFL}, the training procedure of the proposed  $\FedFog$ algorithm  is detailed in Fig. \ref{fig:FLprocess}.  With the latest global model $\mathbf{w}^g$ at round $g$, UE ($i,j$) computes $L$ gradient updates on the local data. FSs aggregate the received local gradient parameters, and then convey them to CS to update the new global model $\mathbf{w}^{g+1}$, which is then sent back to  UEs through FSs to begin a new global  round. The dimension of learning parameter vectors is the same for every layer.

\noindent\textbf{Local gradient update}: {\color{black}In $\FedFog$, UEs will send their gradient parameters to FSs and CS, instead of the trained local models in \eqref{eq:localtraiingupdates}. When the amount of (local) training data is large, it is often impractical for IoT UEs to compute local updates with the full batch gradient using deterministic  gradient  descent (DGD). Unlike DGD in \cite{LuminICC2020},  we use stochastic gradient descent (SGD)  method to  compute the gradient on mini-batches}. Let $\mathcal{B}_{ij,\ell}^{g}$ be the mini-batch with size $B = |\mathcal{B}_{ij,\ell}^{g}|$,  randomly sampled from $\mathcal{D}_{ij}$ of UE $(i,j)$ at the $\ell$-th local iteration of round $g$. The local parameter estimates in \eqref{eq:localtraiingupdates} for UE $(i,j)$ is revised accordingly as
 \begin{equation}\label{eq:localtraiingupdatesMinibatch}
        \mathbf{w}^g_{ij,\ell+1} := \mathbf{w}^g_{ij,\ell} - \eta^g\nabla F_{ij}(\mathbf{w}^g_{ij,\ell}|\mathcal{B}_{ij,\ell}^g),\ \ell = 0,1,\cdots, L-1
    \end{equation}
where the stochastic gradient is computed by
 \begin{equation}\label{eq:gradientupdatesMinibatch}
       \nabla F_{ij}(\mathbf{w}^g_{ij,\ell}|\mathcal{B}_{ij,\ell}^g) := \frac{1}{B}\sum_{d\in\mathcal{B}_{ij,\ell}^g}\nabla f(\mathbf{w}^g_{ij,\ell}, \mathbf{x}_d,y_d).
    \end{equation}
For unbiased estimate of gradient, the condition $\mathbb{E}\{\nabla F_{ij}(\mathbf{w}^g_{ij,\ell}|\mathcal{B}_{ij,\ell}^g)\} = \nabla F_{ij}(\mathbf{w}^g_{ij,\ell}|\mathcal{D}_{ij})$ should be satisfied. {\color{black}In particular, the mini-batch for each UE changes for every local iteration but its size is fixed during the whole training process.}  The total stochastic gradient updates of UE $(i,j)$ at round $g$ is
 \begin{IEEEeqnarray}{rCl}\label{eq:totalgradientupdatesMinibatch}
\Delta\mathbf{w}_{ij}^g \triangleq \sum_{\ell\in\mathcal{L}}\nabla F_{ij}(\mathbf{w}^g_{ij,\ell}|\mathcal{B}_{ij,\ell}^g)
\end{IEEEeqnarray}   
which  also implies that $\mathbf{w}^g_{ij,L} - \mathbf{w}^g_{ij,0} = -\eta^g\Delta\mathbf{w}_{ij}^g$.

\noindent\textbf{Local aggregation and global update}: After $L$ local updates, FS $i$ will periodically aggregate  gradient parameters  as:
\begin{equation}\label{eq:Aggregation}
    \Delta\mathbf{w}_i^{g} :=  \sum_{j\in\mathcal{J}_i} \sum_{\ell\in\mathcal{L}}\nabla F_{ij}(\mathbf{w}^g_{ij,\ell}|\mathcal{B}_{ij,\ell}^g) = \sum_{j\in\mathcal{J}_i}\Delta\mathbf{w}_{ij}^g. 
\end{equation}
 All FSs then send their aggregated gradient parameters to CS for averaging (i.e., global training update), as commonly done:  
\begin{IEEEeqnarray}{rCl}\label{eq:globaltraining}
    \mathbf{w}^{g+1}  := \mathbf{w}^{g} -\eta^g  \frac{\sum_{i\in\mathcal{I}}\sum_{j\in\mathcal{J}_i}\Delta\mathbf{w}_{ij}^g }{J}
    =  \mathbf{w}^{g} -  \eta^g\frac{\sum_{i\in\mathcal{I}}\Delta\mathbf{w}^{g}_{i} }{J}. 
\end{IEEEeqnarray}
 Once $\mathbf{w}^{g+1}$ is calculated, it will be sent back to  UEs to start  a new global round. The proposed $\FedFog$ algorithm  is summarized in Algorithm \ref{alg:FedFog}. We have the following remarks:

\begin{algorithm}[t]
	\begin{algorithmic}[1]
		\fontsize{9}{9}\selectfont
		\protect\caption{$\FedFog$: Proposed Federated
Learning for  Fog-Cloud Systems}
		\label{alg:FedFog}
		\STATE \textbf{Input:}  $L$, $G$, $I$, $J_i$, and  $\mathcal{D}_{ij}$,    $\forall i,j$
		\STATE \textbf{Initial parameters at CS:} Initialize the global model $\mathbf{w}^{0}$ and learning rate $\eta^{0}$
		\FOR{ $g=0,1,\cdots,G-1$}
		
			 \STATE CS broadcasts $\mathbf{w}^g$ to all FSs

	     \FOR{$i\in\mathcal{I}$ \textit{in parallel}}
	         \STATE  FS $i$ broadcasts $\mathbf{w}^g$ to $J_i$ UEs 
	         \FOR{$j\in\mathcal{J}_i$ \textit{in parallel}}
	         
	            \STATE  Overwrite $\mathbf{w}^g_{ij,0}:=\mathbf{w}^g$
	             
	             \FOR{$\ell=0,1,\cdots,L-1$}

			      \STATE UE $(i,j)$ randomly samples a new mini-batch $\mathcal{B}_{ij,\ell}^g$ with size $B$ and computes the gradient $\nabla F_{ij}(\mathbf{w}^g_{ij,\ell}|\mathcal{B}_{ij,\ell}^g)$
			      
			      \ENDFOR
			      \STATE UE $(i,j)$ sends $\Delta\mathbf{w}_{ij}^g \triangleq \sum_{\ell\in\mathcal{L}}\nabla F_{ij}(\mathbf{w}^g_{ij,\ell}|\mathcal{B}_{ij,\ell}^g)$ to FS $i$ 
			\ENDFOR
			   \STATE  FS $i$ aggregates all gradient parameters $\Delta\mathbf{w}_i^{g} :=   \sum_{j\in\mathcal{J}_i}\Delta\mathbf{w}_{ij}^g$ and then forwards it  to CS for averaging
			\ENDFOR
		   \STATE CS performs global training update $\mathbf{w}^{g+1} := \mathbf{w}^{g} -  \eta^g\frac{\sum_{i\in\mathcal{I}}\Delta\mathbf{w}^{g}_{i} }{J}$
		\ENDFOR
		
\STATE \textbf{Output:}	Final global model $\mathbf{w}^G$
		\end{algorithmic} 
\end{algorithm}

\begin{itemize}
    \item Similar to the literature on FL, our method does not require UEs  to transfer their raw data to FSs and CS,  improving privacy of training data and eliminating communication overhead.
    \item In the proposed  $\FedFog$, the main challenge is to compute the global training update at CS, as given in  \eqref{eq:globaltraining}. In FL with one server and multiple UEs \cite{NguyenIoTFL2020, MChenFL2019,DinhFL2019}, local models will be  uploaded directly from UEs to the server. This, however, is prohibitive in fog-cloud systems, where CS is often located  away from end UEs. The reasons are two-fold: $i)$ It may require extremely high energy consumption of UEs to transmit their local models via wireless links, even not possible to reach the main server due to UEs' limited-battery; $2)$  It may also cause heavy-communication burdens and high-latency communications due to a very large number of uploading parameters and long-distance transmission. In addition, if each FS naively forwards all  received gradient parameters to CS then it will induce very high network traffic, especially with a large number of UEs.   In this regard, the  local aggregation  at FSs given in \eqref{eq:Aggregation}  will result in a learning parameter vector with the same dimension for every layer while still guaranteeing the theoretical performance.
    \item In general, a very large value of $L$ will cause the local models  to converge only to an optimal solution of their local loss functions \cite{ZhangJMLR13}, while a very small value of $L$ will  result in  high communication costs. In this paper, an
appropriate value of $L$ will be numerically evaluated, and  it is assumed to be predefined in our design. 
  \item The optimal value $G^*$ of global rounds is generally unknown in federated settings. In Section \ref{sec:FedFogOpti}, we will consider a stopping criteria based on gradient parameters received at CS which will help  not only reduce  completion time but also save UEs' energy consumption.
\end{itemize}

\subsection{Convergence Analysis}\label{sub:FedFogConver}
To facilitate the analysis, we make the following common assumptions and additional definition to the loss function.
\begin{assumption}\label{assp:1}
  For  $F_{ij}(\mathbf{w})$, $\forall i\in\mathcal{I}$ and $j\in\mathcal{J}_i$, we assume that: $i)$ $F_{ij}(\mathbf{w})$  is $\lambda$-strongly convex, i.e. $F_{ij}(\mathbf{w}) \geq F_{ij}(\bar{\mathbf{w}}) + \langle \nabla F_{ij}(\bar{\mathbf{w}}),\mathbf{w}-\bar{\mathbf{w}}\rangle+\frac{\lambda}{2}\|\mathbf{w}-\bar{\mathbf{w}}\|^2, \forall \mathbf{w}, \bar{\mathbf{w}}$;  and $ii)$ $F_{ij}(\mathbf{w})$ is $\mu$-smooth ($\mu$-Lipschitz gradient), i.e. $\|\nabla F_{ij}(\mathbf{w}) - \nabla F_{ij}(\bar{\mathbf{w}})\| \leq \mu\|\mathbf{w} - \bar{\mathbf{w}}\|, \forall \mathbf{w}, \bar{\mathbf{w}}$.
\end{assumption}

\begin{assumption}[Bounded variance]\label{assp:2}
For any $\ell\in\mathcal{L}$ and $g\in\mathcal{G}$, the variance of the stochastic gradients at UE $(i,j)$ is bounded as: $\mathbb{E}\bigl\{\bigl\|\nabla F_{ij}(\mathbf{w}^g_{ij,\ell}|\mathcal{B}_{ij,\ell}^g)-\nabla F_{ij}(\mathbf{w}^g_{ij,\ell}|\mathcal{D}_{ij})\bigl\|^2\bigl\} \leq \gamma_{ij}^2, \forall i\in\mathcal{I}, j\in\mathcal{J}_i.$
\end{assumption}

\begin{assumption}\label{assp:3}
For any $\ell\in\mathcal{L}$ and $g\in\mathcal{G}$, let $\delta$ be an upper
bound of the expected squared norm of stochastic gradients, i.e. $\mathbb{E}\bigl\{\bigl\|\nabla F_{ij}(\mathbf{w}^g_{ij,\ell}|\mathcal{B}_{ij,\ell}^g)\bigl\|^2\bigl\} \leq \delta^2, \forall i\in\mathcal{I}, j\in\mathcal{J}_i.$
\end{assumption}
Assumption \ref{assp:1} is standard (see \cite{MChenFL2019,DinhFL2019,StichICLR2019,LiICLR2020,WangJSAC2019}) used for the squared-SVM, logistic regression,  and softmax classifier, while Assumptions \ref{assp:2} and \ref{assp:3} have been used in \cite{StichICLR2019,LiICLR2020,ruan2020flexible} to quantify the sampling noise.  
\begin{definition}\label{def:1} Let Assumption \ref{assp:1} hold. We quantify the heterogeneity of the data distribution between UE $(i,j)$ and other UEs by defining $\varepsilon_{ij} \triangleq F_{ij}(\mathbf{w}^*|\mathcal{D}_{ij}) - F_{ij}^*$, where $F_{ij}^*$ is the minimum local loss of UE $(i,j)$. It is clear that $\varepsilon_{ij}$ is finite for strongly convex loss function, and $\varepsilon_{ij}=0$ if data distribution of clients are i.i.d.
\end{definition}

 We first introduce additional notation and then provide key lemmas to support the proof of convergence. {\color{black}Motivated by \cite{ruan2020flexible,LiICLR2020}, let $\bar{\mathbf{w}}^g_{\ell+1}$ be the average of one local update from all UEs, i.e., $\bar{\mathbf{w}}^g_{\ell+1} \triangleq \frac{1}{J}\sum_{i\in\mathcal{I}}\sum_{j\in\mathcal{J}_i}\bigl(\mathbf{w}^g_{ij,\ell} - \eta^g\nabla F_{ij}(\mathbf{w}^g_{ij,\ell}|\mathcal{B}_{ij,\ell}^g)\bigr)$}. For $\nabla F(\bar{\mathbf{w}}^g_{\ell})\triangleq  \frac{1}{J}\sum_{i\in\mathcal{I}}\sum_{j\in\mathcal{J}_i}$ $\nabla F_{ij}(\mathbf{w}^g_{ij,\ell}|\mathcal{B}_{ij,\ell}^g)$, we have $\bar{\mathbf{w}}^g_{\ell+1}=\bar{\mathbf{w}}^g_{\ell} - \eta_g\nabla F(\bar{\mathbf{w}}^g_{\ell})$. It is clear that $\bar{\mathbf{w}}^{g+1} = \mathbf{w}^{g+1}, \forall g$  since it is accessible to all UEs, but not for $\bar{\mathbf{w}}^g_{\ell},\forall \ell$. In what follows, we use $\nabla F_{ij}(\mathbf{w}^g_{ij,\ell})$ and $\nabla  \bar{F}_{ij}(\mathbf{w}^g_{ij,\ell})$ to denote $\nabla F_{ij}(\mathbf{w}^g_{ij,\ell}|\mathcal{B}_{ij,\ell}^g)$ and $\nabla F_{ij}(\mathbf{w}^g_{ij,\ell}|\mathcal{D}_{ij})$, respectively, for simplicity. It is clear that $\nabla  \bar{F}_{ij}(\mathbf{w}^g_{ij,\ell}) = \mathbb{E}\{\nabla  {F}_{ij}(\mathbf{w}^g_{ij,\ell})\}$. We now provide some intermediate results, whose proofs are given in Appendix \ref{app: lemma123}.
\begin{lemma}[The expected upper bound of the variance  of the stochastic gradients]\label{lema:1}
Let Assumption \ref{assp:2} hold. We have
\begin{equation}
    \mathbb{E}\Bigr\{\bigr\|\sum_{i\in\mathcal{I}}\sum_{j\in\mathcal{J}_i} \frac{1}{J}\bigl(\nabla F_{ij}(\mathbf{w}^g_{ij,\ell})-\nabla  \bar{F}_{ij}(\mathbf{w}^g_{ij,\ell})\bigl) \bigl\|^2\Bigr\} \leq \frac{\sum_{i\in\mathcal{I}}\sum_{j\in\mathcal{J}_i}\gamma_{ij}^2}{J^2},\ \forall \ell, g.
\end{equation}
\end{lemma}
\begin{lemma}\label{lema:2}
Let Assumption \ref{assp:3} hold. The expected upper bound of the divergence of $\{\mathbf{w}^g_{ij,\ell}\}$ is given as
\begin{equation}
    \frac{1}{J}\mathbb{E}\Bigr\{\sum_{i\in\mathcal{I}}\sum_{j\in\mathcal{J}_i}\bigr\| \bar{\mathbf{w}}^g_{\ell} - \mathbf{w}^g_{ij,\ell} \bigl\|^2\Bigr\} \leq (L-1)L\eta_g^2\delta^2,\ \forall \ell, g.
\end{equation}
\end{lemma}
\begin{lemma}\label{lema:3} Let Assumptions \ref{assp:1}-\ref{assp:3} hold. From Definition \ref{def:1} and if $\eta_g\leq 1/4\mu$, the expected upper bound of $\mathbb{E}\{\|\bar{\mathbf{w}}^g_{\ell+1} - \mathbf{w}^*\|^2\}$ is given as
\begin{align}
  \mathbb{E}\{\|\bar{\mathbf{w}}^g_{\ell+1} - \mathbf{w}^*\|^2\} \leq  (1-0.5\lambda \eta_g )\mathbb{E}\{\|\bar{\mathbf{w}}^g_{\ell}  - \mathbf{w}^*\|^2\} 
   + \eta_g^2\bar{\Omega}^g_{\ell}  +\frac{2\eta_g}{J}\sum_{i\in\mathcal{I}}\sum_{j\in\mathcal{J}_i}\mathbb{E}\{\bigl(\bar{F}_{ij}(\mathbf{w}^*)-\bar{F}_{ij} (\bar{\mathbf{w}}^g_{\ell})  \bigr)\}
\end{align}
where $\bar{\Omega}^g_{\ell}\triangleq(2+\lambda/4\mu)(L-1)L\delta^2 +\frac{1}{J^2} \sum_{i\in\mathcal{I}}\sum_{j\in\mathcal{J}_i}\gamma_{ij}^2 + 6\mu \frac{1}{J}\sum_{i\in\mathcal{I}}\sum_{j\in\mathcal{J}_i}\varepsilon_{ij}$.
\end{lemma}
Combining with the results from Lemmas \ref{lema:1}-\ref{lema:3}, the convergence of $\FedFog$ is stated in the following theorem.
\begin{theorem}\label{theo:1}
Let Assumptions \ref{assp:1}-\ref{assp:3} hold. Given the optimal global model $\mathbf{w}^*$, $\bar{Q}^0=\mathbb{E}\{\|\mathbf{w}^0 - \mathbf{w}^*\|^2\}$, and the diminishing learning rate $\eta_g = \frac{16}{\lambda(g+1+\psi)}$, we can obtain the expected  upper bound
of $\FedFog$ after G global rounds as
\begin{align}
    \mathbb{E}\{\|\mathbf{w}^G- \mathbf{w}^*\|^2\} \leq \frac{\max\bigl\{\psi^2\bar{Q}^0, \frac{16^2}{\lambda^2}G\Theta \bigr\}}{(G+\psi)^2}
\end{align}
where   $\Theta\triangleq 2L^2\delta^2+(2+\lambda/4\mu)(L-1)L\delta^2 + \frac{L\sum_{i\in\mathcal{I}}\sum_{j\in\mathcal{J}_i}\gamma_{ij}^2}{J^2} + 6\mu L \frac{1}{J}\sum_{i\in\mathcal{I}}\sum_{j\in\mathcal{J}_i}\varepsilon_{ij}$ and
$\psi = \max\bigl\{\frac{64\mu}{\lambda},\ 4L \bigr\}$.
\end{theorem}
\begin{proof}
Please see Appendix \ref{app: theo1}.
\end{proof}
\noindent {\color{black}It can be seen that with an appropriate diminishing learning rate, the optimal global model is obtained after a sufficient large number of global rounds, i.e., $\underset{G\rightarrow\infty}{\mathrm{lim}}\mathbb{E}\{\|\mathbf{w}^G- \mathbf{w}^*\|^2\}\propto \underset{G\rightarrow\infty}{\mathrm{lim}}\frac{1}{G}\rightarrow0$.} 

\noindent \textbf{Flexible user aggregation:} To reduce  completion time, it is also of practical interest to perform a flexible user aggregation  which allows fewer UEs to participate in the training process in each round. The convergence of $\FedAvg$ with partial user participation has been
studied in \cite{LiICLR2020}. The large variance due to non-i.i.d. data can be controlled by fine-tuned learning rates. We will detail the flexible aggregation  strategy over wireless fog-cloud systems in Section \ref{sec:FedFogOptiD}.

\section{Network-Aware Optimization of $\FedFog$ over Wireless Fog-Cloud Systems}\label{sec:FedFogOpti}
\begin{figure}[!ht]
	\centering
		\includegraphics[width=0.8\columnwidth,trim={0cm 0.0cm 0cm 0.0cm}]{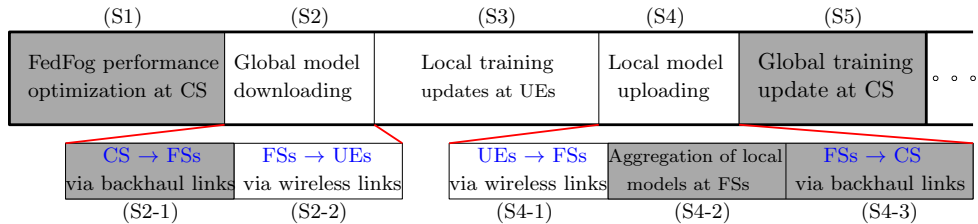}
		\vspace{-5pt}
	\caption{Illustration of the proposed scheme to support $\FedFog$ in one global round over wireless fog-cloud systems.}
	\label{fig:SchemesupportedFL}
\end{figure}
The proposed wireless fog-cloud scheme to support $\FedFog$ consists of five main steps, as illustrated in Fig. \ref{fig:SchemesupportedFL}. Compared to traditional FL algorithms over wireless communications systems \cite{DinhFL2019,MChenFL2019,ZYang2019,NguyenIoTFL2020}, the additional steps of $\FedFog$'s performance optimization (Step 1), model downloading/uploading between CS and FSs (mini-steps S2-1 and S4-3), and aggregation of local models at FSs (mini-step S4-2) are added to enable $\FedFog$ over wireless fog-cloud systems. Unlike \cite{DinhFL2019,MChenFL2019,ZYang2019}, the resource allocation algorithm (S1) is done in each round since the future training loss has not been revealed.

\subsection{Network and Computation Costs Model}
{\color{black}In this section, we focus on minimizing the completion time of $\FedFog$ and the global loss in \eqref{eq:globalmini}. The communication delay between FSs (or BSs) and CS (i.e., S2-1 and S4-3) to exchange a single model update in both DL and UL is neglected since in many practical networks, they can be connected through  backhaul links with sufficiently large capacities (i.e., high-speed optical ones). For example,  advanced optical fiber is considered as  the most  viable solution to meet the extremely low latency requirements of backhaul links between FSs and CS, i.e., down to 150 microseconds  \cite{DotschBell2014}. In addition, FSs and CS  are often equipped with much higher computational power than UEs to execute tasks \cite{ChenTWC2018,DuTCOM2018}. Therefore, the latency of execution time  of  the  resource allocation algorithm  (S1),  aggregate local models  at FSs  (S4-2) and global training update at CS (S5) is  ignored.}

{\color{black}\textit{1) Communication model:} We assume that BS $i\in\mathcal{I}$ (installed FS $i$) is equipped with $K_i$ antennas to serve $J_i$ single-antenna UEs via a shared wireless medium. Let us denote by $\mathbf{h}_{ij}^\dl(g)\in\mathbb{C}^{K_i\times 1}$ and $\mathbf{h}_{ij}^\ul(g)\in\mathbb{C}^{K_i\times 1}$ the channel (column) vectors between BS $i$ and UE $(i,j)$ in downlink (DL) and uplink (UL), respectively. The channel vector $\mathbf{h}_{ij}^x(g)$ with $x\in\{\dl,\ul\}$ is modeled as $\mathbf{h}_{ij}^x(g)=\sqrt{\varphi_{ij}(g)}\bar{\mathbf{h}}_{ij}^x(g)$, which accounts for both the effects of large-scale fading $\varphi_{ij}(g)$ (e.g.,  path loss and shadowing) with 
a low degree of mobility and  small-scale fading $\bar{\mathbf{h}}_{ij}^x(g)\sim\mathcal{CN}(0,\mathbf{I}_{K_i})$ and remains unchanged during round $g$, but
changes independently from one round to another. Let $W^\dl$, $W^\ul$, and $N_0$ be the DL and UL system bandwidths (Hz), and noise power spectral density (dBm/Hz) at receivers, respectively. The communication delay of UL can be predominant over DL since BSs are typically equipped with much higher power budget than UEs as well as they have  high bandwidth used for data broadcasting. Therefore, we allocate the DL bandwidth equally to BS $i$ as $W^\dl_i = W^\dl/I, \forall i\in\mathcal{I}$. 
Adopting frequency-division multiple access (FDMA) for UL wireless links from UEs to BSs, we denote the UL bandwidth allocated to UE $(i,j)$ at round $g$ by $\beta_{ij}(g)W^\ul$, satisfying $\sum_{i\in\mathcal{I}}\sum_{j\in\mathcal{J}_i}\beta_{ij}(g) \leq 1, \forall g$. Under FDMA, the optimal receiver (i.e., maximum ratio combining, $\mathbf{h}_k^H/\|\mathbf{h}_k\|$) is used for the UL transmission. Let $\mathsf{S}_{\dl}$ and $\mathsf{S}_{\ul}$ denote the data size (in bits) of the updated parameters in  DL and UL, respectively. In this paper, we have $\mathsf{S}_{\dl}  < \mathsf{S}_{\ul}$ since UEs are required to additionally send the local loss value to CS to design a stopping criteria.

 Each BS $i\in\mathcal{I}$ broadcasts the latest global model to all $J_i$ UEs using physical-layer multicasting, where the broadcasting transmission rate of each BS is simply determined by its slowest user \cite{KimTWC2011}. The DL achievable rate (a lower bound on the ergodic rate) in bits/s of UE $(i,j)$ to download the latest  global model from BS $i$ at round $g$ can be computed as\footnote{A more sophisticated beamformer is beyond the scope of this paper.}:
\begin{equation}
r_{ij}^\dl(g) = W^\dl_i\log\bigl(1 + \underset{j\in\mathcal{J}_i}{\min}\,\SNR_{ij}^\dl(g) \bigr) 
\end{equation}
where the signal-to-noise ratio (SNR) is  $\SNR_{ij}^\dl(g)=\frac{P_{i}^{\max}\mathbb{E}\{\|\mathbf{h}^\dl_{ij}(g)\|^2\}}{W^{\dl}N_0} = \frac{P_{i}^{\max}K_i\varphi_{ij}(g)}{W^{\dl}N_0}$, and $P_{i}^{\max}$ denotes the maximum transmit power at BS $i$. Here, we consider a  worst-case SNR against the noise power, i.e., $W_i^{\dl}N_0 \leq W^{\dl}N_0$. 
%
%
The DL communication delay of UE $(i,j)$ at round $g$ is 
\begin{equation}
t^{\dl,co}_{ij}(g) = \frac{\mathsf{S}_{\dl}}{r_{ij}^\dl(g)}
\end{equation}
which is the same for all UE $(i,j)$ associated with BS $i$. Similarly, the UL communication delay for UE $(i,j)$ to transmit the trained parameters  to BS $i$ using FDMA is  
\begin{equation}t^{\ul,co}_{ij}(g)= \frac{\mathsf{S}_{\ul}}{r_{ij}^\ul(g)},\ \text{with}\ r_{ij}^\ul(g) = \beta_{ij}(g)W^{\ul}\log\bigl(1 + \SNR_{ij}^\ul(g)\bigr) \end{equation}
where  $\ds\SNR_{ij}^\ul(g) = \frac{p_{ij}(g)K_i\varphi_{ij}(g)}{W^{\ul}N_0}$, and $p_{ij}(g)$ is the transmit power coefficient of UE $(i,j)$ during the local model uploading phase, subject to the power constraint $p_{ij}(g) \leq P_{ij}^{\max}$. 
}

\textit{2) Computation and energy consumption models at UEs:}
Denote by $c_{ij}$ the number of CPU cycles required
for executing 1 data bit of UE $(i,j)$, which  assumes to be known  \textit{a prior} by an offline measurement\cite{MiettinenUSENIX10}. Let $\mathsf{f}_{ij}(g)$  be the CPU clock speed of UE $(i,j)$,  which can be chosen in the range $[\mathsf{f}_{ij}^{\min},\mathsf{f}_{ij}^{\max} ]$. Then, the computation delay for  local training updates at UE $(i,j)$  over $L$  local iterations at round $g$ can be expressed as
\begin{equation}\label{eq:timecomp}t_{ij}^{cp}(g)= L\frac{c_{ij}\mathsf{S}_B}{\mathsf{f}_{ij}(g)}
\end{equation}
where $\mathsf{S}_{B}$  denotes the  mini-batch  size.

The total energy consumed by UE $(i,j)$ at round $g$  can be formulated as:
\begin{equation}\label{eq:Ecomp}
E_{ij}(g) = \underbrace{p_{ij}(g)t^{\ul,co}_{ij}(g)}_{E_{ij}^{co}(g)}  +  \underbrace{L\frac{\theta_{ij}}{2}c_{ij}\mathsf{S}_{B}\mathsf{f}^2_{ij}(g)}_{E_{ij}^{cp}(g)} 
\end{equation}
where $E_{ij}^{co}(g)$ is the energy consumption required to transmit the trained local model via the UL, and $E_{ij}^{cp}(g)$ is the energy consumed for local executions; The constant $\theta_{ij}/2$ represents the average switched capacitance and the average activity factor of UE $(i,j)$ \cite{BurdVLSI1996,GerardsJSAC15}.

\subsection{Network-Aware Optimization Problem}
We assume that UEs  communicate asynchronously with BSs \cite{NguyenIoTFL2020,ZYang2019}. {\color{black}The delay of  one global round (say, round $g$)  of  $\FedFog$  is}
\begin{equation}
    T(g) = \underset{\forall i,j}{\max}\bigr\{t^{\dl,co}_{ij}(g) + t_{ij}^{cp}(g) + t^{\ul,co}_{ij}(g)\bigl\}.
\end{equation}
The completion time for implementing $\FedFog$ over $G$ global rounds is thus $T_{{\Sigma }}=\sum_{g\in\mathcal{G}} T(g)$. The simplest way to compute $T$ is to set a sufficiently large value of $G$ \cite{MChenFL2019,ZYang2019,NguyenIoTFL2020}, which ensures the convergence of $\FedFog$. However, such a solution may require redundant transmissions of training models between UEs and CS, resulting in  an extra cost. Therefore, it is necessary  to design a stopping criteria to output an optimal value $G^* \leq G$, which helps to achieve a lower communication cost  $T^{*}_{\Sigma }=\sum_{g=0}^{G^*-1} T(g)$. Inspired by \cite{MahmoudiICC20}, we solve \eqref{eq:globalmini} by $\FedFog$, taking into account the iterative costs. \cite{MahmoudiICC20}. 

\textbf{Cost function and performance measure:}  The general cost function should capture both the global loss function and the completion time, which are two prime objectives  in  FL algorithms, to provide a more precise measurement on the system performance. We first introduce a cost function over $G$ global rounds based on the multi-objective optimization method \cite{Marler2004}:
\begin{equation}\label{eq:utilityfunction}
    C(G) \triangleq \alpha\frac{F(\mathbf{w}^G)}{F_0} + (1-\alpha)\frac{\sum_{g\in\mathcal{G}} T(g)}{T_0}
\end{equation}
where $\alpha\in[0,\ 1]$ is  the priority parameter. In addition,    $F_0 > 0$ and $T_0 >0$ denote the references of the loss value and the completion time, respectively, which are used to deal with the different dimensions of the two quantities.  We can see that minimizing both  the global  loss  function  and  the  completion  time are conflicting. In particular, the lower the  completion  time, the higher the global loss, resulting in  low accuracy of the learning model. Therefore, the priority parameter $\alpha$ is imposed to create a trade-off between the two objective functions. The higher the value of $\alpha$, the higher the completion time that the $\FedFog$ is willing to spend to achieve a better  accuracy of the learning model.
{\color{black} 
\begin{remark}
We note that minimizing the cost function $C(G)$ requires complete information about the network across $G$ global rounds, which is obviously unaffordable in  hierarchical FL-supported wireless fog-cloud networks. In particular, to compute the cost function $C(G)$ in \eqref{eq:utilityfunction} over $G$ global rounds in an offline manner,  we need to have the sequences $\{F(\mathbf{w}^g)\}_{\forall g}$ and $\{ T(g)\}_{\forall  g}$ in advance, which is unrealistic since  the future values  at round $g+1$ (i.e., $F(\mathbf{w}^{g+1})$ and $T(g+1)$) are not revealed  at the beginning of  round $g$. This calls for an alternating procedure, solely based on the network information in each global round. The training time will be accumulated after each global round.
\end{remark}}

Based on the above discussions, we consider the following minimization problem of   joint  learning  and  communication at round $g$:
\begin{subequations} \label{eq:OP1}
	\begin{IEEEeqnarray}{cl}
		\underset{\mathbf{w}^g,\mathbf{p}(g), \mathsf{\mathbf{f}}(g), \boldsymbol{\beta}(g)}{\mathrm{minimize}} &\quad   C(g) \triangleq \alpha\frac{F(\mathbf{w}^g)}{F_0} + (1-\alpha)\frac{\sum_{g'=0}^{g} T(g')}{T_0}\label{eq:OP1a} \\
		\st & \quad  E_{ij}(g)    \leq  \mathcal{E}^{\max},\ \forall i\in\mathcal{I}, j\in\mathcal{J}_i,                        \label{eq:OP1b}\\
		     & \quad \mathsf{SNR}_{ij}^{\ul}(g) \geq \SNR^{\min},\ \forall  i\in\mathcal{I}, j\in\mathcal{J}_i,                        \label{eq:OP1c}\\
				& \quad p_{ij}(g) \leq P_{ij}^{\max},\ \forall i\in\mathcal{I}, j\in\mathcal{J}_i,                        \label{eq:OP1d}\\
			    & \quad \mathsf{f}_{ij}^{\min}\leq \mathsf{f}_{ij}(g) \leq \mathsf{f}_{ij}^{\max},\ \forall i\in\mathcal{I}, j\in\mathcal{J}_i,   \label{eq:OP1e}\\
				& \quad \sum_{i\in\mathcal{I}}\sum_{j\in\mathcal{J}_i}\beta_{ij}(g) \leq 1\label{eq:OP1f}
	\end{IEEEeqnarray}
\end{subequations}
where $\mathbf{p}(g)\triangleq\{p_{ij}(g)\}_{\forall i,j}$, $\mathsf{\mathbf{f}}(g)\triangleq\{\mathsf{f}_{ij}(g)\}_{\forall i,j}$ and ${\boldsymbol{\beta}}(g)\triangleq\{\beta_{ij}(g)\}_{\forall i,j}$. Constraints \eqref{eq:OP1b} and \eqref{eq:OP1c} indicate the maximum energy consumption requirement $\mathcal{E}_{\max}$ and the minimum SNR requirement $\SNR^{\min}$ for performing one round of $\FedFog$, respectively.  As discussed previously,  \eqref{eq:OP1d}, \eqref{eq:OP1e} and \eqref{eq:OP1f} are the transmit power, CPU-frequency and bandwidth constraints for  UE $(i,j)$, respectively. 

{\color{black}We can see that it is not possible to minimize the two quantities in \eqref{eq:OP1a} simultaneously since they are optimized on different time slots. Intuitively, problem \eqref{eq:OP1} can be decomposed into two sub-problems at round $g$ as follows:
\begin{equation}\label{eq:OP1FL}
    \min_{\mathbf{w}^g\in\mathbb{R}^q} F(\mathbf{w}|\mathcal{D})
\end{equation}
which is the FL problem solved by Algorithm \ref{alg:FedFog}, solely based on the UEs' local datasets, and 
\begin{equation}\label{eq:OP1RA}
    \underset{\mathbf{p}(g), \mathsf{\mathbf{f}}(g), \boldsymbol{\beta}(g)}{\mathrm{minimize}}\   C(g)|\ \st\ \eqref{eq:OP1b}-\eqref{eq:OP1f}
\end{equation}
which is the resource allocation sub-problem of joint computation and communication resources solved at CS for a given $F(\mathbf{w}^g)$ obtained in the last global round. We note that minimizing the cost function $C(g)$ in \eqref{eq:OP1RA} is equivalent to minimizing the delay of one global round, i.e. $T(g)$, since  the total delay of the previous round is already revealed at round $g$.
}

\subsection{Proposed Path-Following Algorithm to Solve \eqref{eq:OP1RA}}
 In what follows, we treat the loss function $F(\mathbf{w}^g)$ as a constant and rewrite \eqref{eq:OP1a} equivalently as
\begin{equation}
    C(g)= (1-\alpha)\frac{T(g)}{T_0} + \bar{C}(g) 
\end{equation}
where $\bar{C}(g) \triangleq \alpha\frac{F(\mathbf{w}^g)}{F_0} + (1-\alpha)\frac{\sum_{g'=0}^{g-1} T(g')}{T_0}$ is also a constant at round $g$ since
$\sum_{g'=0}^{g-1} T(g')$ is already computed in the previous rounds.


Problem \eqref{eq:OP1RA} is nonconvex due to the non-concavity of \eqref{eq:OP1a} and non-convexity of \eqref{eq:OP1b}. By introducing new variables $t(g)$ and $\boldsymbol{\tau}(g)\triangleq\{\tau_{ij}(g)\}_{\forall i,j}$, problem \eqref{eq:OP1RA} is rewritten  as
\begin{subequations} \label{eq:OP2}
	\begin{IEEEeqnarray}{cl}
		\underset{\substack{\mathbf{p}(g), \mathsf{\mathbf{f}}(g), \boldsymbol{\beta}(g) \\ t(g), \boldsymbol{\tau}(g)  }}{\mathrm{minimize}} &\quad   C(g) \triangleq (1-\alpha)\frac{t(g)}{T_0} + \bar{C}(g)\label{eq:OP2a} \\
		\st & \quad t^{\dl,co}_{ij}(g) + L\frac{c_{ij}\mathsf{S}_B}{\mathsf{f}_{ij}(g)} + \frac{\mathsf{S}_{\ul}}{\tau_{ij}(g)} \leq t(g),\ \forall  i\in\mathcal{I}, j\in\mathcal{J}_i,                        \label{eq:OP2b}\\
		&\quad r_{ij}^{\ul}(g) \geq \tau_{ij}(g),\ \forall  i\in\mathcal{I}, j\in\mathcal{J}_i,                        \label{eq:OP2c}\\
				& \quad  \mathsf{S}_{\ul}\frac{p_{ij}(g) }{\tau_{ij}(g)}  +  L\frac{\theta_{ij}}{2}c_{ij}\mathsf{S}_{B}\mathsf{f}_{ij}^2(g)   \leq  \mathcal{E}^{\max},\ \forall i\in\mathcal{I}, j\in\mathcal{J}_i,                        \label{eq:OP2d}\\
		         & \quad p_{ij}(g) \geq \SNR^{\min}\frac{W^{\ul}N_0}{K_i\varphi_{ij}(g)},\ \forall  i\in\mathcal{I}, j\in\mathcal{J}_i,                        \label{eq:OP2e}\\
				& \quad \eqref{eq:OP1d}, \eqref{eq:OP1e}, \eqref{eq:OP1f}\label{eq:OP2f}
	\end{IEEEeqnarray}
\end{subequations}
where \eqref{eq:OP2d} and \eqref{eq:OP2e} are transformed  from \eqref{eq:OP1b} and \eqref{eq:OP1c}, respectively.  The equivalence between \eqref{eq:OP1RA} and \eqref{eq:OP2} is due to the fact that constraints \eqref{eq:OP2b} and \eqref{eq:OP2c} must hold with equality at optimum for at least some of the slowest UEs. In problem \eqref{eq:OP2}, the nonconvex parts include \eqref{eq:OP2c} and \eqref{eq:OP2d}, which can be convexified by IA framework \cite{Marks:78}.

Let us treat \eqref{eq:OP2c} first. We  make the variable change $\ds\tilde{\beta}_{ij}(g)=\frac{1}{\beta_{ij}(g)} \geq 1, \forall i,j$ to  equivalently rewrite \eqref{eq:OP2c} as
\begin{subnumcases}{\eqref{eq:OP2c}\\ \Leftrightarrow}
\frac{1}{\tilde{\beta}_{ij}(g)}\log\bigl(1 + \frac{1}{\omega_{ij}(g)}\bigr) \geq \frac{\tau_{ij}(g)}{W^{\ul}}\label{eq:OP2c1}\\
\frac{p_{ij}(g)K_i\varphi_{ij}(g)}{W^{\ul}N_0} \geq \frac{1}{\omega_{ij}(g)}\label{eq:OP2c2}
\end{subnumcases}
where ${\tilde{\boldsymbol{\beta}}}(g)\triangleq\{\tilde{\beta}_{ij}(g)\}_{\forall i,j}$ and ${\boldsymbol{\omega}}(g)\triangleq\{\omega_{ij}(g)\}_{\forall i,j}$ are newly introduced variables. Constraint \eqref{eq:OP2c2} is convex and can be cast into a second-order cone (SOC) one. In \eqref{eq:OP2c1}, the function $\frac{1}{\tilde{\beta}_{ij}(g)}\log\bigl(1 + \frac{1}{\omega_{ij}(g)}\bigr)$ is convex which can be verified by checking the Hessian matrix. 
Applying the inequality \cite[Appendix A]{Dinh:JSAC:18}, we iteratively convexify constraint \eqref{eq:OP2c1} at iteration $\kappa +1$ as
\begin{IEEEeqnarray}{rCl}\label{eq:OPsyn4bConvex}
\mathcal{R}_{ij}^{\ul,(\kappa)}\bigr(\tilde{\beta}_{ij}(g),\omega_{ij}(g)\bigl) \triangleq a_{ij}^{(\kappa)} - b_{ij}^{(\kappa)}\omega_{ij}(g) - c_{ij}^{(\kappa)}\tilde{\beta}_{ij}(g) \geq \frac{\tau_{ij}(g)}{W^{\ul}},\ \forall  i\in\mathcal{I}, j\in\mathcal{J}_i
\end{IEEEeqnarray}
where $a_{ij}^{(\kappa)} \triangleq 2\frac{1}{\tilde{\beta}_{ij}^{(\kappa)}(g)}\log\bigl(1 + \frac{1}{\omega_{ij}^{(\kappa)}(g)}\bigr) + \frac{1}{\tilde{\beta}_{ij}^{(\kappa)}(g)(\omega_{ij}^{(\kappa)}(g)+1)}$, $b_{ij}^{(\kappa)} \triangleq \frac{1}{\tilde{\beta}_{ij}^{(\kappa)}(g)\omega_{ij}^{(\kappa)}(g)(\omega_{ij}^{(\kappa)}(g)+1)}$ and $c_{ij}^{(\kappa)}\triangleq  \frac{1}{(\tilde{\beta}_{ij}^{(\kappa)}(g))^2}\log\bigl(1 + \frac{1}{\omega_{ij}^{(\kappa)}(g)}\bigr)$ are positive constants. Here $\tilde{\beta}_{ij}^{(\kappa)}(g)$ and $\omega_{ij}^{(\kappa)}(g)$  are the feasible points of $\tilde{\beta}_{ij}(g)$ and $\omega_{ij}(g)$ obtained at iteration $\kappa$, respectively. It is clear that the function $\mathcal{R}_{ij}^{\ul,(\kappa)}\bigr(\tilde{\beta}_{ij}(g),\omega_{ij}(g)\bigl)$ is concave lower bound of $\frac{1}{\tilde{\beta}_{ij}(g)}\log\bigl(1 + \frac{1}{\omega_{ij}(g)}\bigr)$.
Next, applying \cite[Eq. (B.1)]{Dinh:JSAC:18} to  $\frac{p_{ij}(g) }{\tau_{ij}(g)}$ in \eqref{eq:OP2d} yields
\begin{equation}\label{eq:OP2dconvex}
    \frac{\mathsf{S}_{\ul}}{2}\Bigl(\frac{1}{\tau^{(\kappa)}_{ij}(g)p^{(\kappa)}_{ij}(g) }p^2_{ij}(g)  + \frac{p^{(\kappa)}_{ij}(g)}{2\tau_{ij}(g) - \tau^{(\kappa)}_{ij}(g)} \Bigr) +  L\frac{\theta_{ij}}{2}c_{ij}\mathsf{S}_{B}\mathsf{f}_{ij}^2(g)   \leq  \mathcal{E}^{\max},\ \forall  i\in\mathcal{I}, j\in\mathcal{J}_i
\end{equation}
which is the convex constraint.

Bearing all the above in mind, we solve the following inner convex approximate  program at iteration $\kappa +1$:
\begin{subequations} \label{eq:OP3}
	\begin{IEEEeqnarray}{cl}
	\quad&	\underset{\substack{\mathbf{p}(g), \mathsf{\mathbf{f}}(g), \tilde{\boldsymbol{\beta}}(g) \\ t(g), \boldsymbol{\tau}(g), \boldsymbol{\omega}(g)  }}{\mathrm{minimize}} \quad   C(g) \triangleq (1-\alpha)\frac{t(g)}{T_0} + \bar{C}(g)\label{eq:OP3a} \\
		&\st  \quad t^{\dl,co}_{ij}(g) + L\frac{c_{ij}\mathsf{S}_B}{\mathsf{f}_{ij}(g)} + \frac{\mathsf{S}_{\ul}}{\tau_{ij}(g)} \leq t(g),\ \forall  i\in\mathcal{I}, j\in\mathcal{J}_i,                        \label{eq:OP3b}\\
		&\qquad 0.5\bigl(p_{ij}(g) + \omega_{ij}(g)\bigr)  \geq \Bigl\|\sqrt{\frac{W^{\ul}N_0}{K_i\varphi_{ij}(g)}};\quad 0.5(p_{ij}(g) - \omega_{ij}(g))\Bigr\|,\ \forall  i\in\mathcal{I}, j\in\mathcal{J}_i, \label{eq:OP3c}\\
		&\qquad\sum_{i\in\mathcal{I}}\sum_{j\in\mathcal{J}_i}\frac{1}{\tilde{\beta}_{ij}(g)} \leq 1, \label{eq:OP3d}\\
	   	& \qquad \eqref{eq:OP1d}, \eqref{eq:OP1e},  \eqref{eq:OP2e}, \eqref{eq:OPsyn4bConvex}, \eqref{eq:OP2dconvex} \label{eq:OP3e}
	\end{IEEEeqnarray}
\end{subequations}
{\color{black}where the SOC constraint \eqref{eq:OP3c} is derived from \eqref{eq:OP2c2}}. We successively solve \eqref{eq:OP3} and update the optimization variables $(\mathbf{p}^{(\kappa)}(g),  \tilde{\boldsymbol{\beta}}^{(\kappa)}(g), \boldsymbol{\tau}^{(\kappa)}(g), \boldsymbol{\omega}^{(\kappa)}(g))$ until convergence.
The proposed  path-following procedure  to solve \eqref{eq:OP1RA} is summarized in Algorithm \ref{alg_2}. The initial feasible  values for   $(\mathbf{p}^{(0)}(g),  \tilde{\boldsymbol{\beta}}^{(0)}(g), \boldsymbol{\tau}^{(0)}(g), \boldsymbol{\omega}^{(0)}(g))$ are required for starting the IA procedure. We first randomly generate $p^{(0)}_{ij}(g) \in \bigr[ \SNR^{\min}\frac{W^{\ul}N_0}{K_i\varphi_{ij}(g)}, P^{\max}_{ij}\bigl], \forall i,j   $ and then set $\tilde{\beta}_{ij}^{(0)}(g) = J,\ \tau_{ij}^{(0)}(g) =  \frac{1}{J}W^{\ul}\log\bigl(1 + \frac{p^{(0)}_{ij}(g)K_i\varphi_{ij}(g)}{W^{\ul}N_0}\bigr),\  \omega_{ij}^{(0)}(g) = \frac{W^{\ul}N_0}{p^{(0)}_{ij}(g)K_i\varphi_{ij}(g)},  \forall i,j$.

\begin{algorithm}[t]
\begin{algorithmic}[1]
\fontsize{9}{9}\selectfont
\protect\caption{Proposed  Path-Following Procedure for Solving  \eqref{eq:OP1RA}}
\label{alg_2}
\global\long\def\algorithmicrequire{\textbf{Initialization:}}
\REQUIRE  Set $\kappa:=0$ and  choose initial feasible values for   $(\mathbf{p}^{(0)}(g),  \tilde{\boldsymbol{\beta}}^{(0)}(g), \boldsymbol{\tau}^{(0)}(g), \boldsymbol{\omega}^{(0)}(g))$ to constraints
 in \eqref{eq:OP3}
\REPEAT
\STATE Solve  \eqref{eq:OP3} to obtain the optimal solutions $\bigl(\mathbf{p}^{*}(g), \mathsf{\mathbf{f}}^{*}(g), \tilde{\boldsymbol{\beta}}^{*}(g), t^{*}(g), \boldsymbol{\tau}^{*}(g), \boldsymbol{\omega}^{*}(g)\bigr)$

\STATE Update\ \ $\bigl(\mathbf{p}^{(\kappa+1)}(g),  \tilde{\boldsymbol{\beta}}^{(\kappa+1)}(g), \boldsymbol{\tau}^{(\kappa+1)}(g), \boldsymbol{\omega}^{(\kappa+1)}(g)\bigr) := \bigl(\mathbf{p}^{*}(g),  \tilde{\boldsymbol{\beta}}^{*}(g), \boldsymbol{\tau}^{*}(g), \boldsymbol{\omega}^{*}(g)\bigr)$
\STATE Set $\kappa:=\kappa+1$
\UNTIL Convergence\\
\STATE{\textbf{Output:}} The optimal solutions $\bigl(\mathbf{p}^{*}(g), \mathsf{\mathbf{f}}^{*}(g), \boldsymbol{\beta}^{*}(g)\bigr)$  where $\ds \beta^{*}_{ij}(g)=1/\tilde{\beta}^{*}_{ij}(g), \forall i,j$
\end{algorithmic} \end{algorithm}

\textit{Convergence and complexity analysis:} The   path-following Algorithm \ref{alg_2} is based on the IA  framework \cite{Marks:78},  where  all approximate functions in \eqref{eq:OP3} are satisfied IA properties in \cite{Beck:JGO:10}. In particular, Algorithm \ref{alg_2} produces better solutions after each iteration, which converge to at least a local optimal solution when $\kappa \rightarrow\infty$, satisfying the Karush-Kuhn-Tucker (KKT) conditions \cite[Theorem 1]{Marks:78}. Problem \eqref{eq:OP3} includes $7J+1$ linear and conic constraints and $5J+1$ scalar decision variables. By a general interior-point method \cite[Chapter 6]{Ben:2001}, the worst-case of per-iteration complexity of Algorithm \ref{alg_2} is $\mathcal{O}\bigl(\sqrt{7J}(5J)^3 \bigr)$.

\subsection{Proposed Network-Aware Optimization Algorithms}\label{sec:FedFogOptiD}
{\color{black}We note that \eqref{eq:OP3} is a discrete convex program of \eqref{eq:OP1RA} in each  global round $g\in\mathcal{G}$. Towards a practical application, we use the theoretical results above to develop the network-aware optimization algorithms in distributed environments due to the causal setting.
\begin{assumption}\label{assp:4} In addition to Assumption \ref{assp:1},
  we further assume that the local loss function $F_{ij}(\mathbf{w})$, $\forall i\in\mathcal{I}$ and $j\in\mathcal{J}_i$ is $\lambda$-strongly convex and non-increasing.
\end{assumption}
\begin{proposition}\label{pro:1} Let Assumption \ref{assp:4} hold.
  Since the cost function $C(g)$ is a discrete convex function, there always exists $G^* > 0$ as a minimizer of the  problem: $  G^* := \argmin_{g\in\mathcal{G}} C(g),$ where $C(G^*-1) \geq C(G^*)$ and $C(G^*) \leq C(G^*+1)$.
  \end{proposition}
\noindent From Assumption \ref{assp:4}, we can show that  the global loss function $F(\mathbf{w}^g)$ is non-increasing  while the completion time function $\sum_{g\in\mathcal{G}} T(g)$ is non-decreasing over time. 
  It implies that we can stop $\FedFog$ at round $G^*$ once the stopping condition $C(G^*) - C(G^*-1) > 0$ is met, without incurring in extra costs. The optimal solution $G^*$ can be found by tracking the sign of two consecutive values of the cost function $C(g)$. However, the non-increasing sequence of the global loss function may not hold true in all global rounds due to non-i.i.d.  data and stochastic noise  of  the  random  sampling  of  mini-batchs. In this case, a few more rounds are needed to avoid an improper early convergence of $\FedFog$. This phenomenon will be empirically justified by numerical results.}

\subsubsection{Full User Aggregation}
  
\begin{algorithm}[!htbp]
	\begin{algorithmic}[1]
		\fontsize{9}{9}\selectfont
		\protect\caption{Proposed $\FedFog$-based Network-Aware Optimization Algorithm  with Full User Aggregation}
		\label{alg:networkawarealg}
		\STATE \textbf{Input:}  $L$, $G$, $I$, $J_i$,  $\mathcal{D}_{ij},  \bar{k}, B, \mathcal{E}^{\max}, \epsilon, \mathsf{SNR}^{\min}, P^{\max}_{ij}, \mathsf{f}^{\min}_{ij},$ and  $\mathsf{f}^{\max}_{ij},$    $\forall i,j$
		\STATE \textbf{Initial parameters at CS:} Initialize  $\mathbf{w}^{0}$, $\eta^{0}$, and set $G^*=G, g=0$
		\WHILE{ $g\leq G-1$}
		     \STATE Run Algorithm \ref{alg_2}  and then broadcast the optimal solutions to UEs using dedicated control channel //(S1)
		     
			 \STATE CS broadcasts $\mathbf{w}^g$ to all FSs //(S2-1)

	     \FOR{$i\in\mathcal{I}$ \textit{in parallel}}
	         \STATE  FS $i$ broadcasts $\mathbf{w}^g$ to $J_i$ UEs //(S2-2)
	         \FOR{$j\in\mathcal{J}_i$ \textit{in parallel}}
	         
	            \STATE  Overwrite $\mathbf{w}^g_{ij,0}:=\mathbf{w}^g$
	             
	             \FOR{$\ell=0,1,\cdots,L-1$}

			      \STATE UE $(i,j)$ randomly samples a new mini-batch $\mathcal{B}_{ij,\ell}^g$ with size $B$; and computes $\nabla F_{ij}(\mathbf{w}^g_{ij,\ell}|\mathcal{B}_{ij,\ell}^g)$ and  $F_{ij}(\mathbf{w}^g)$ //(S3)
			      
			      \ENDFOR
			      \STATE UE $(i,j)$ sends $\Delta\mathbf{w}_{ij}^g \triangleq \sum_{\ell\in\mathcal{L}}\nabla F_{ij}(\mathbf{w}^g_{ij,\ell}|\mathcal{B}_{ij,\ell}^g)$ and $F_{ij}(\mathbf{w}^g)$ to FS $i$ //(S4-1)
			\ENDFOR
			   \STATE  FS $i$ calculates $\mathbf{w}_i^{g} :=   \sum_{j\in\mathcal{J}_i}\Delta\mathbf{w}_{ij}^g$ and $F_{i}(\mathbf{w}^g):=\sum_{j\in\mathcal{J}_i}F_{ij}(\mathbf{w}^g)$, and  then forwards them  to CS //(S4-2 \& S4-3)
			\ENDFOR
		   \STATE CS performs global training update $\mathbf{w}^{g+1} := \mathbf{w}^{g} -  \eta^g\frac{\sum_{i\in\mathcal{I}}\Delta\mathbf{w}^{g}_{i} }{J}$; and calculates the cost function $C(g)=\alpha\frac{\sum_{i\in\mathcal{I}}F_{i}(\mathbf{w}^g)}{JF_0} + (1-\alpha)\frac{\sum_{g'=0}^{g} T(g')}{T_0}$
//(S5)
        \IF{$C(g)-C(g-1)\geq \epsilon$} 
          \IF{($k\geq \bar{k}$ \&\& $g\geq \bar{G}$)}
            \STATE Set $G^*=g-\bar{k}$; Break and go to step 28
           \ENDIF
           \STATE Set $k:=k+1$
        \ELSE
        \STATE  $k\leftarrow 0$
        \ENDIF
        \STATE Set $g:=g+1$ 
		\ENDWHILE
		
\STATE \textbf{Output:}	$\mathbf{w}^*, G^*, F(\mathbf{w}^{G^*})$ and $T^{*}_{\Sigma }=\sum_{g=0}^{G^*+\bar{k}+1}T(g)$
		\end{algorithmic} 
\end{algorithm}

The complete algorithm with the full user aggregation is summarized in Algorithm \ref{alg:networkawarealg}, where $\epsilon$ in Step 18 is a small positive constant. {\color{black}In Step 19, the condition $g\geq \bar{G}$ is added to guarantee a comparable accuracy of the learning model, where $\bar{G}$ is the required minimum  number of global rounds. The actual value of $\bar{G}$ may depend on the specific ML applications, FL algorithms and  datasets.} The large variance of the global loss value, which is due to non-i.i.d. data and  stochastic noise of the random sampling of mini-batchs, may lead to an improper early stop in Step 18. To tackle this issue, CS may wait for some more global rounds to ensure the convergence of $\FedFog$. If the condition $C(g)-C(g-1)\geq \epsilon$ is met for $\bar{k}>0$ consecutive rounds, we terminate Algorithm \ref{alg:networkawarealg}. We can see that to calculate the last cost value, an additional round of global and local updates is carried out at the end. As a result, the effective completion time for implementing $\FedFog$ in Algorithm \ref{alg:networkawarealg} is given as $T^{*}_{\Sigma }= \sum_{g=0}^{G^*+\bar{k}+1}T(g)$.

\subsubsection{Flexible User Aggregation} 
We can see that in Algorithm \ref{alg:networkawarealg}, CS needs to wait for the slowest UEs (i.e., due to low computing capability, low battery level and unfavorable links) to perform the global training update in each round, which may result in  higher training delay (so-called ``straggler effect''). As shown in \cite{LiICLR2020}, each UE is only required to activate sometime but still guarantees the convergence of $\FedAvg$. Thus, our next endeavor is to propose a flexible user aggregation  to  reduce completion time. The key idea is to train  strong UEs  first to obtain a certain accuracy level, and then more UEs will be allowed to join the training process until convergence.

To achieve the above goal, we relax problem \eqref{eq:OP2} as
\begin{subequations} \label{eq:OP2relax}
	\begin{IEEEeqnarray}{cl}
		\underset{\substack{\mathbf{p}(g), \mathsf{\mathbf{f}}(g), \boldsymbol{\beta}(g) \\ \mathbf{t}(g), \boldsymbol{\tau}(g)  }}{\mathrm{minimize}} &\quad   \hat{C}(g) \triangleq (1-\alpha)\frac{\sum_{i\in\mathcal{I}}\sum_{j\in\mathcal{J}_i}t_{ij}(g)}{JT_0} + \bar{C}(g)\label{eq:OP2relaxa} \\
		\st & \quad t^{\dl,co}_{ij}(g) + L\frac{c_{ij}\mathsf{S}_B}{\mathsf{f}_{ij}(g)} + \frac{\mathsf{S}_{\ul}}{\tau_{ij}(g)} \leq t_{ij}(g),\ \forall  i\in\mathcal{I}, j\in\mathcal{J}_i,                        \label{eq:OP2relaxb}\\
		& \quad \eqref{eq:OP1d}, \eqref{eq:OP1e}, \eqref{eq:OP1f}, \eqref{eq:OP2c}, \eqref{eq:OP2d}, \eqref{eq:OP2e}\label{eq:OP2relaxc}
	\end{IEEEeqnarray}
\end{subequations}
where $t_{ij}(g)$ is considered as a soft-latency of UE $(i,j)$ and $\mathbf{t}(g)\triangleq\{t_{ij}(g)\}_{\forall i,j}$. For the objective  \eqref{eq:OP2relaxa}, CS will favor  UEs with better conditions by allocating more resources to them, and thus achieving lower latency than other UEs. This problem can be directly solved by Algorithm \ref{alg_2}. Let $\mathcal{S}(g)$ be the set of  $S(g)=|\mathcal{S}(g)|$ UEs selected  at round $g$. Given the optimal solution $\{t_{ij}^*(0)\}_{\forall i,j}$ obtained from solving \eqref{eq:OP2relax} at the first round, CS  determines a time threshold $\mathcal{T}(0) := \mathcal{T}_{\min}$ to allow the first $S(0)=J_{\min}$ responded UEs (i.e., $J_{\min}$ UEs with the lowest delay) to participate in global updates, given as
\begin{equation}\label{eq:Tmin}
    \mathcal{T}(0):=\mathcal{T}_{\min} = \underset{(i,j)\in\mathcal{S}(0)}{\max}\{t_{ij}(0)\}
\end{equation}
where $J_{\min}\in(0,\ J]$ should be large enough to guarantee the quality of learning. CS then synchronizes $\mathcal{T}_{\min}$ to all FSs, and any UE $(i,j)$ with  higher latency (i.e., $t_{ij}(0) > \mathcal{T}_{\min}, \forall i,j $) will be ignored from the local aggregations at FSs. When the  certain accuracy level is obtained at round $g$, i.e.,
\begin{equation}\label{eq:ACClevel}
    \Bigl\|\frac{1}{S(g)}\sum_{(i,j)\in\mathcal{S}(g)}\Delta\mathbf{w}_{ij}^g \Bigr\| < \xi
\end{equation}
we increase the time threshold $\mathcal{T}(g)$  by $\Delta \mathcal{T}$  to allow  weaker UEs to join the global update, i.e., $\mathcal{S}(g):=\mathcal{S}(g-1)\cup\{\text{UE}\ (i,j)|t_{ij}(g)\leq \mathcal{T}(g)\}$, where $\xi$ is a small positive constant. This procedure is repeated untill all UEs are joined the training process. We summarize the complete procedure of flexible user aggregation  in Algorithm \ref{alg:networkawarealgFUA}.

\begin{algorithm}[!htbp]
	\begin{algorithmic}[1]
		\fontsize{9}{9}\selectfont
		\protect\caption{Proposed $\FedFog$-based Network-Aware Optimization Algorithm  with Flexible User Aggregation}
		\label{alg:networkawarealgFUA}
		\STATE \textbf{Input:}  $L$, $G$, $I$, $J_i$,  $\mathcal{D}_{ij}, J_{\min}, \bar{k}, B, \epsilon, \xi, \Delta\mathcal{T}, \mathcal{E}^{\max}, \mathsf{SNR}^{\min}, P^{\max}_{ij}, \mathsf{f}^{\min}_{ij},$ and  $\mathsf{f}^{\max}_{ij}$    $\forall i,j$
		\STATE \textbf{Initial parameters at CS:} Initialize  $\mathbf{w}^{0}$, $\eta^{0}$,   and set $\hat{G}^*=G, g=0$
		\WHILE{ $g\leq G-1$}
		     \STATE Run Algorithm \ref{alg_2}  and then broadcast the optimal solutions to UEs using dedicated control channel //(S1)
		     \IF{$g=0$}
		     \STATE Calculate $\mathcal{T}(0):=\mathcal{T}_{\min}$ and $\mathcal{S}(0)=\{\text{UE}\ (i,j)|t_{ij}(0)\leq \mathcal{T}_{\min}, \forall i,j\}$ in \eqref{eq:Tmin}; Break and go to step 10
		     \ELSIF{the condition \eqref{eq:ACClevel} is met}
		     \STATE Update $\mathcal{T}(g) :=\mathcal{T}(g-1)+\Delta\mathcal{T}$ and   $\mathcal{S}(g):=\mathcal{S}(g-1)\cup\{\text{UE}\ (i,j)|t_{ij}(g)\leq \mathcal{T}(g)\}$
		     \ENDIF
		     
			 \STATE CS broadcasts $\mathbf{w}^g$ to all FSs //(S2-1)

	     \FOR{$i\in\mathcal{I}$ \textit{in parallel}}
	         \STATE  FS $i$ broadcasts $\mathbf{w}^g$ to $J_i$ UEs //(S2-2)
	         \FOR{$j\in\mathcal{J}_i$ \textit{in parallel}}
	         
	            \STATE  Overwrite $\mathbf{w}^g_{ij,0}:=\mathbf{w}^g$
	             
	             \FOR{$\ell=0,1,\cdots,L-1$}

			      \STATE UE $(i,j)$ randomly samples a new mini-batch $\mathcal{B}_{ij,\ell}^g$ with size $B$; and computes  $\nabla F_{ij}(\mathbf{w}^g_{ij,\ell}|\mathcal{B}_{ij,\ell}^g)$ and  $F_{ij}(\mathbf{w}^g)$ //(S3)
			      
			      \ENDFOR
			      \STATE UE $(i,j)$ sends $\Delta\mathbf{w}_{ij}^g \triangleq \sum_{\ell\in\mathcal{L}}\nabla F_{ij}(\mathbf{w}^g_{ij,\ell}|\mathcal{B}_{ij,\ell}^g)$ and $F_{ij}(\mathbf{w}^g)$ to FS $i$ //(S4-1)
			\ENDFOR
			   \STATE  FS $i$ calculates $\mathbf{w}_i^{g} :=   \sum_{j\in\mathcal{J}_i(g)}\Delta\mathbf{w}_{ij}^g$  and $F_{i}(\mathbf{w}^g):=\sum_{j\in\mathcal{J}_i(g)}F_{ij}(\mathbf{w}^g)$ where  $\mathcal{J}_i(g)$ is the subset of UEs with $t_{ij}(g) \leq \mathcal{T}(g)$, and  then forwards them  to CS //(S4-2 \& S4-3)
			\ENDFOR
		   \STATE CS performs global training update $\mathbf{w}^{g+1} := \mathbf{w}^{g} -  \eta^g\frac{\sum_{i\in\mathcal{I}}\Delta\mathbf{w}^{g}_{i} }{S(g)}$; and calculates the cost function $\hat{C}(g)=\alpha\frac{\sum_{i\in\mathcal{I}}F_{i}(\mathbf{w}^g)}{S(g)F_0} + (1-\alpha)\frac{\sum_{g'=0}^{g} \mathcal{T}(g')}{T_0}$
//(S5)
        \IF{($\hat{C}(g)-\hat{C}(g-1)\geq \epsilon$ \&\& $S(g)=J$)} 
           \IF{($k\geq \bar{k}$ \&\& $g\geq \bar{G}$)}
            \STATE Set $\hat{G}^*=g-\bar{k}$; Break and go to step 33
           \ENDIF
           \STATE Set $k:=k+1$
        \ELSE 
         \STATE $k\leftarrow 0$
        \ENDIF
        \STATE Set $g:=g+1$  
		\ENDWHILE
		
\STATE \textbf{Output:}	$\mathbf{w}^*, \hat{G}^*, F(\mathbf{w}^{\hat{G}^*})$ and $\hat{T}^{*}_{\Sigma }= \sum_{g=0}^{\hat{G}^*+\bar{k}+1}\mathcal{T}(g)$
		\end{algorithmic} 
\end{algorithm}

\section{Numerical Results}\label{sec:Numericalresults} 
In this section, we numerically evaluate our proposal algorithms in several scenarios. We first present the simulation setup in Section \ref{sec:Numericalresults}-A and validate the performance of $\FedFog$  in Section \ref{sec:Numericalresults}-B. The performance comparison of Algorithms \ref{alg:networkawarealg} and \ref{alg:networkawarealgFUA} over a wireless fog-cloud network will be provided in Section \ref{sec:Numericalresults}-C.

\noindent\subsection{Simulation Setup}
 \noindent\textbf{ML Model and Data Samples:} We consider an image  classification task using a multinomial logistic regression with a convex loss function. The regularization parameter is fixed to $10^{-4}$. We evaluate $\FedFog$ by training neural networks on MNIST and CIFAR-10 datasets.
 \begin{itemize}
     \item MNIST \cite{LecunMNIST98} contains 70K images of hand-written digits 0-9 with 60K training samples and 10K testing samples. We train a fully-connected Neural Network (FCNN) with a single hidden layer using ReLU activation and a softmax layer at the end. There are $(784+1)\times 10=7,850$ optimized parameters, where the input and output sizes of the NN model are $28\times 28=784$ and 10, respectively. The initial learning rate is set to $\eta^0=0.001$, which is decayed after every global round as $\eta^g=\frac{\eta^0}{1.01^g}$.
      \item CIFAR-10 \cite{CIFAR10}  consists of 60K  colour images in 10 different classes (e.g., airplanes, cars, birds, etc.) with 50K training images and 10K testing images, where each image in CIFAR-10 is  32$\times$32 colour image. We train a convolutional NN (CNN) which has  two $3\times 3$ convolution layers followed by $2\times 2$ maxPooling, one fully-connected layer (128 units) using ReLU activation and  a softmax  at the output layer. The  learning rate is set to  $\eta^g=\frac{\eta^0}{1.005^g}$ with $\eta^0=0.001$.
  \end{itemize}
  
 
\noindent\textbf{Data Distribution:}
Due to  the limited number of samples on datasets, we consider 100 UEs concurrently participating in the training process. There are 5 BSs (or FSs), each has 20 UEs. We consider non-i.i.d. distributed data across the network, where each UE has the same number of data samples but  contains only one of the ten classes. We generate an initial global model as $\mathbf{w}^0=\boldsymbol{0}$. 

\noindent\textbf{Simulation Parameters and Benchmark Schemes over Wireless Fog-Cloud Systems:}
\begin{figure}[!ht]
	\centering
	\includegraphics[width=0.45\columnwidth,trim={0cm 0.0cm 0cm 0.0cm}]{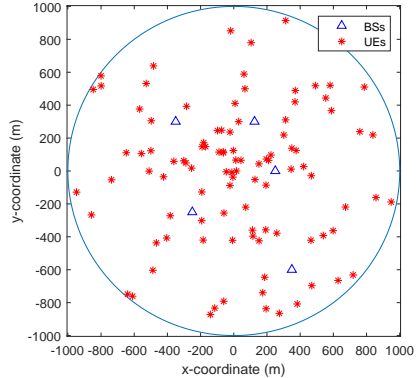}
		\vspace{-2pt}
	\caption{A system topology with $I=5$ FSs, $J=100$ UEs and $J_i=20$ UEs, $\forall i$.}
	\label{fig:Layout}
\end{figure}
\begin{table}[t]
	\centering  
	\captionof{table}{Simulation Parameters}
	\label{tab:Simulationparameter}
	\vspace{-5pt}
 	\scalebox{0.82}{
		\begin{tabular}{l|l}
			\hline
			Parameter & Value \\
			\hline\hline
			System bandwidth, $W=W^\dl=W^\ul$ & 10 MHz \\
			Noise power spectral density, $N_0$ & -174 dBm/Hz \\
			SNR threshold, $\SNR^{\min}$	&  1 dB\\
			Number of antennas at BS $i$, $K_i,\forall i$ & 8\\
		    Power budget at FSs, $P_{i}^{\max}
		    ,\forall i$ & 40 dBm\\ 
		     Effective capacitance coefficient, $\theta_{ij}/2,\forall i,j$ &	$10^{-28}$\\
		     Priority parameter, $\alpha$ & 0.7\\
		     Minimum  number of global rounds for MNIST, $\bar{G}$ & 250\\
		     Minimum  number of global rounds for CIFAR-10, $\bar{G}$ & 600\\
		    Energy   consumption requirement, $\mathcal{E}^{\max}$, for MNIST &   0.01 Joule \\
		     Energy   consumption requirement, $\mathcal{E}^{\max}$, for CIFAR-10 &   1 Joule \\
		    Reference loss and completion time, $(F_0, T_0)$, for MNIST &  (0.1,100)\\
		    Reference loss and completion time, $(F_0, T_0)$, for CIFAR-10 &  (1,1000)\\
		    Threshold for stopping condition, $\bar{k}$ & 5\\
			\hline		  				
		\end{tabular}
	}
\end{table}

We consider a system topology shown in Fig. \ref{fig:Layout}, where $5$ BSs and 100 UEs  are located within a circle of
1-km radius. The locations of BSs are fixed during the simulation. 
The large-scale fading (in dB) is generated as $\varphi_{ij}(g)=-103.8-20.9\log(d_{ij}(g))$, where $d_{ij}(g)$ (in km) is the distance between BS $i$ and UE ($i,j$) at round $g$ \cite{NguyenIoTFL2020}.
By the IEEE 754-2008 standard, we use 32-bit float type to store model weights and  the local loss value. To illustrate the heterogeneity of UEs, $P^{\max}_{ij}$ is uniformly distributed in $[10,\ 23]$ dBm, $c_{ij}$ is uniformly distributed in $[10,\ 20]$ cycles/bits, $\mathsf{f}_{ij}^{\max}$ is uniformly distributed in $[10^9,\ 3.10^9]$ cycles/s and  $\mathsf{f}_{ij}^{\min} = 10^6$ cycles/s.
The other parameters are specified in Table \ref{tab:Simulationparameter}, following  \cite{MaoJSAC16,NguyenIoTFL2020,DinhFL2019,VuCellfreeML2019,MChenFL2019}. We set $ \mathcal{E}^{\max}_{\text{CIFAR-10}} > \mathcal{E}^{\max}_{\text{MNIST}}$  since the batch size of CIFAR-10 is much larger than that of MNIST. In most cases,
Algorithm \ref{alg_2} converges in about 5 iterations. The results are averaged over 100 simulation trials.

For comparison purpose, we consider the following three schemes:
\begin{itemize}
 \item ``Equal Bandwidth (EB):'' Each UE $(i,j)$ at round $g$ is allocated the fixed portion of bandwidth as $\beta_{ij}(g)=1/J, \forall i,j$ in uplink.

\item ``Fixed Resource Allocation (FRA):'' Since the communication delay is often dominant  computation delay, we assume that UE $(i,j)$ uses its maximum transmit power (i.e., $p_{ij}(g)=P^{\max}_{ij}, \forall i,j,g$), and the frequency $\mathsf{f}_{ij}(g)$ is then computed by \eqref{eq:OP1b} and \eqref{eq:OP1e}.

\item ``Sampling scheme \cite{NguyenIoTFL2020,LiICLR2020}:''  At each global round, only a subset $J(g)$ is selected at random to participate in the training process. This scheme allows more bandwidth to be allocated to UEs in the uplink links.
\end{itemize}

\subsection{Effect of Hyperparameters on $\FedFog$ (Algorithm \ref{alg:FedFog})}
 \begin{figure}[!ht]%
\centering
\subfigure{%
\label{fig:5-a}%
\includegraphics[height=2.0in]{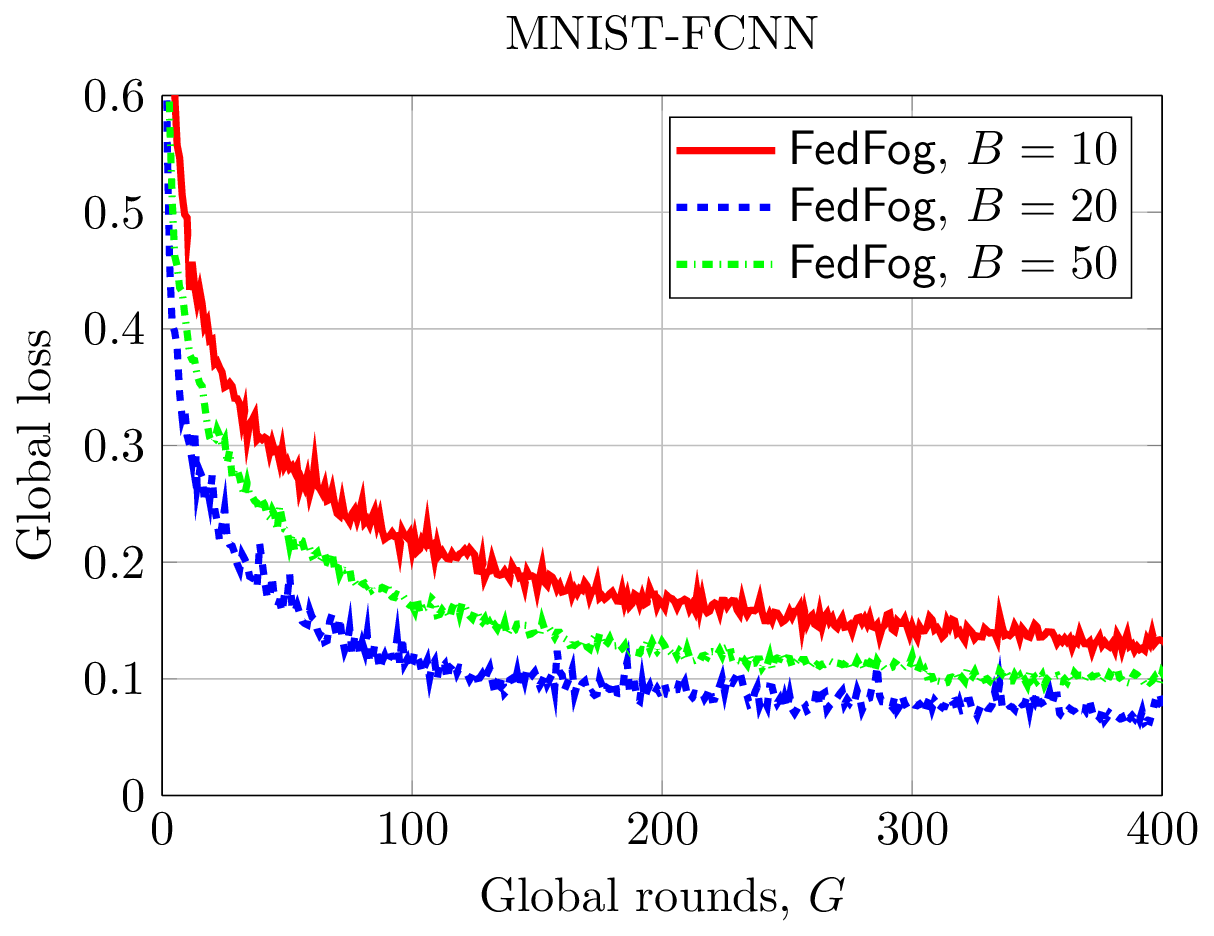}}%
\hspace{4pt}%
\subfigure{%
\label{fig:5-b}%
\includegraphics[height=2.0in]{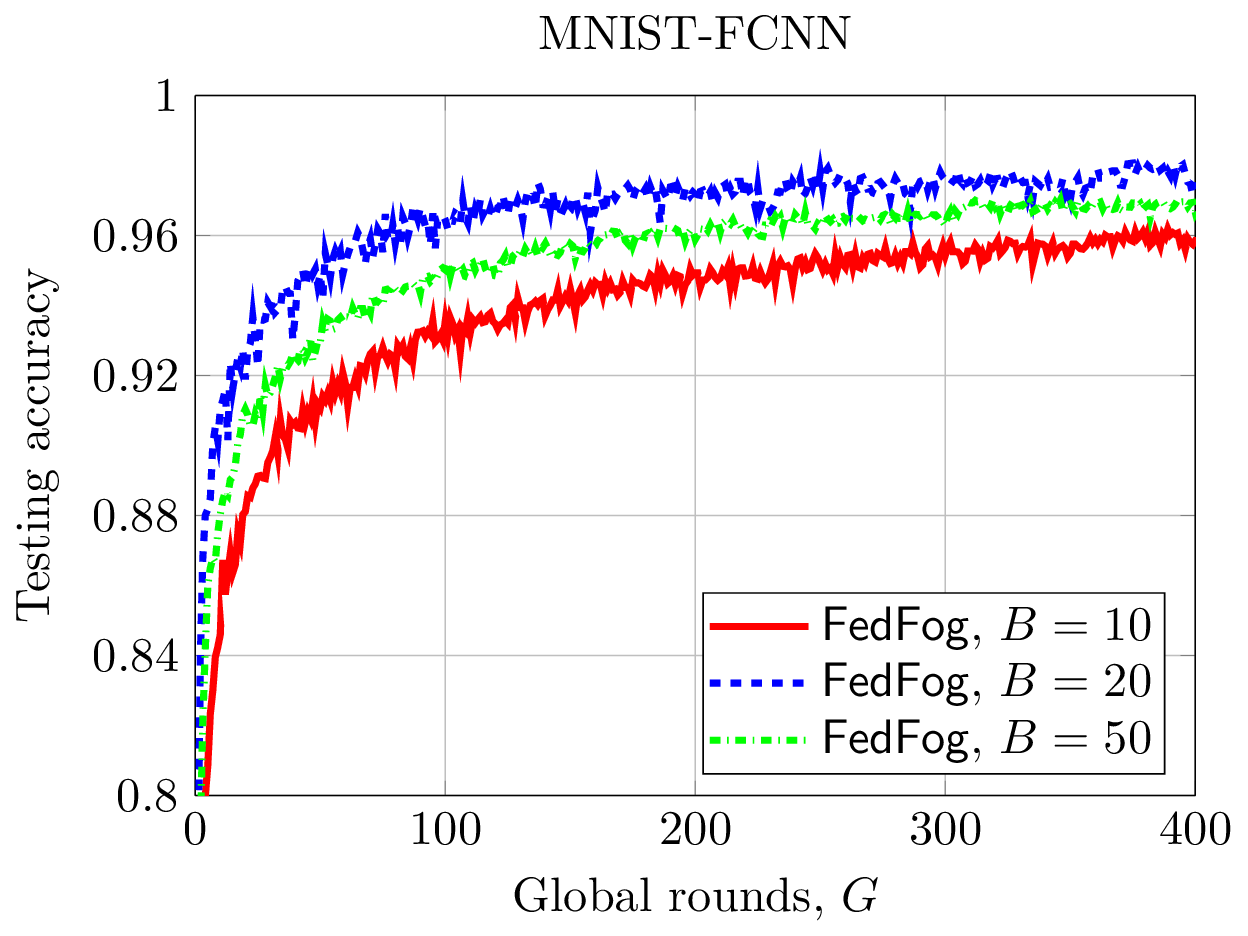}} \\
\subfigure{%
\label{fig:5-c}%
\includegraphics[height=2.0in]{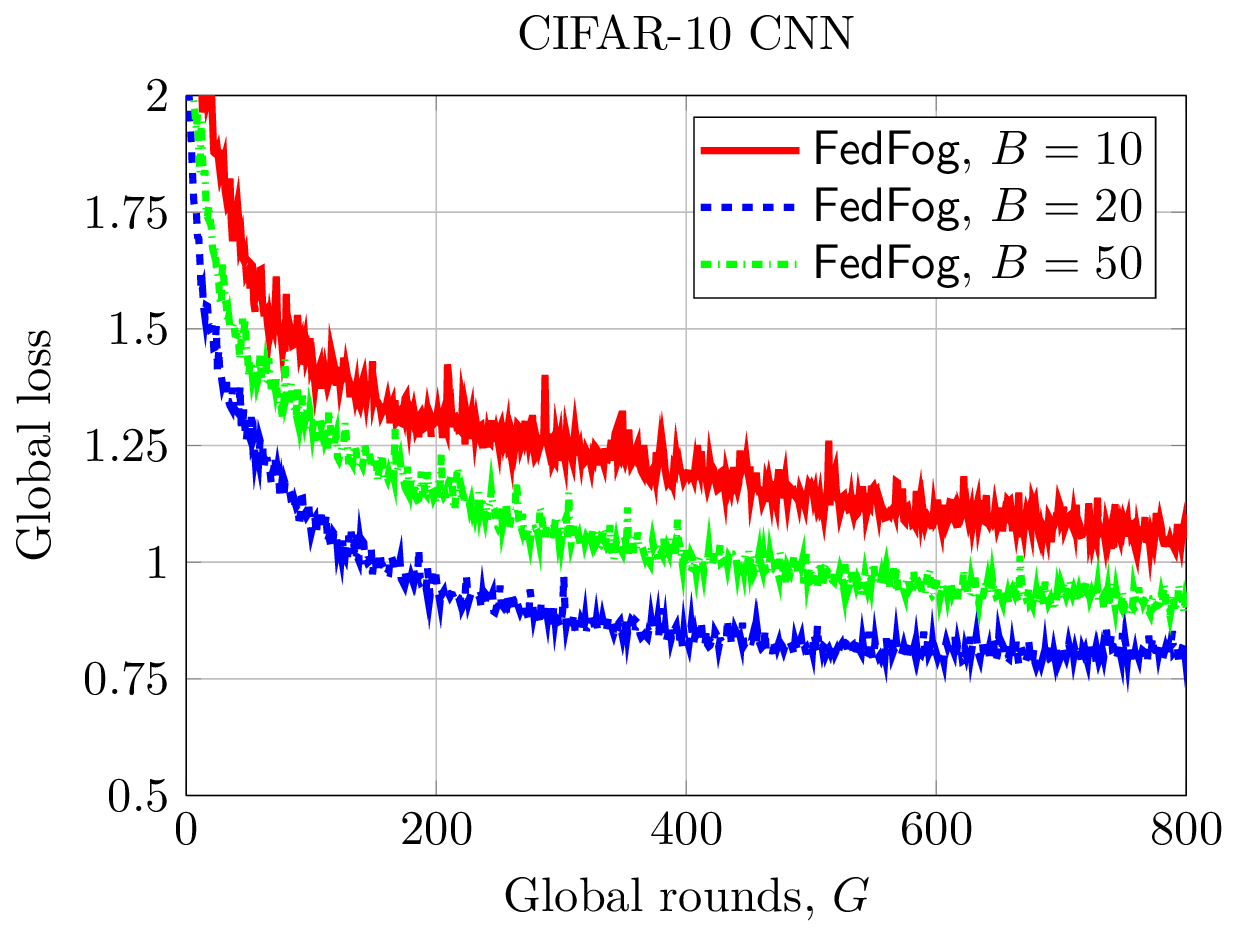}}%
\hspace{4pt}%
\subfigure{%
\label{fig:5-d}%
\includegraphics[height=2.0in]{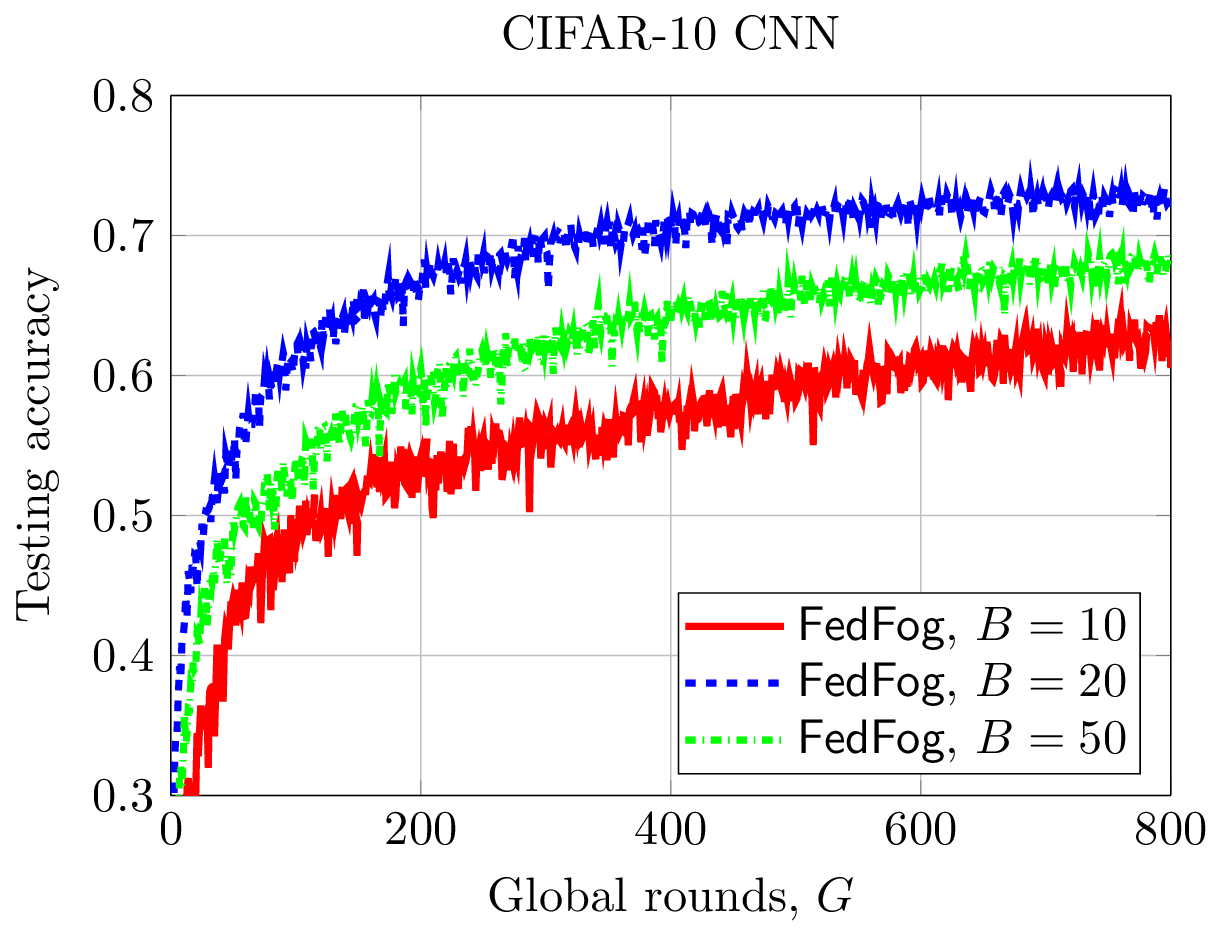}}%
\caption[]{Effect of mini-batch size $B$  on the convergence of $\FedFog$, with $L=10$.}
\label{fig:Alg1Batchsize}%
\end{figure}

In Fig. \ref{fig:Alg1Batchsize}, we investigate the effect of mini-batch size $B\in\{10,20,50\}$ on the performance of $\FedFog$ for both MNIST and CIFAR-10 datasets. It can be observed that increasing the size of the mini-batch results in a better convergence rate of $\FedFog$ since more data are trained in each iteration. {\color{black}However, a very large mini-batch size (e.g., $B=50$) slows down the convergence rate of $\FedFog$ as it requires more local iterations for the local model training to obtain the same accuracy of the learning model with the medium mini-batch size (e.g., $B=20$) in each round}. In addition,  large mini-batch sizes will consume more power and require higher computation at local UEs.

 \begin{figure}[!ht]%
\centering
\subfigure{%
\label{fig:6-a}%
\includegraphics[height=2.0in]{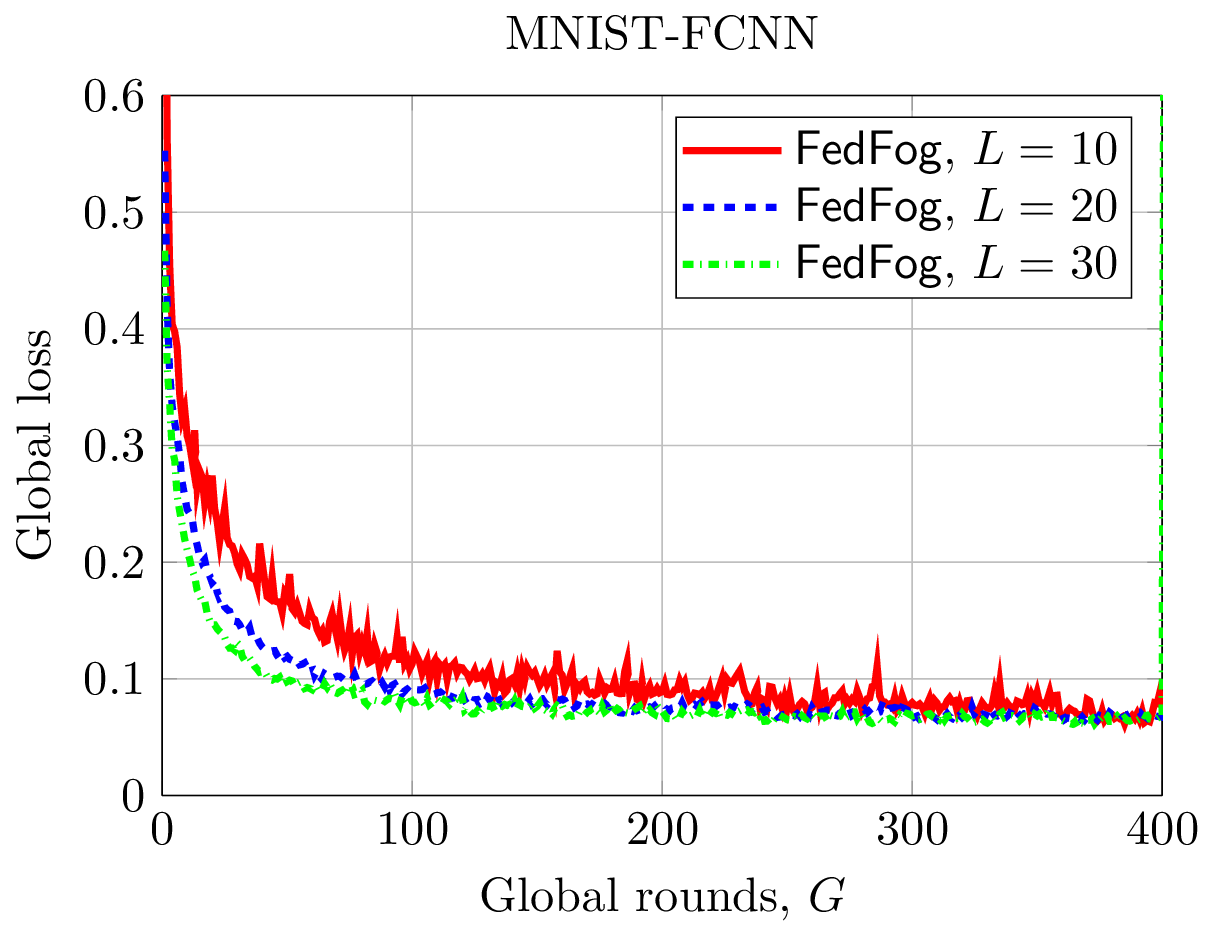}}%
\hspace{4pt}%
\subfigure{%
\label{fig:6-b}%
\includegraphics[height=2.0in]{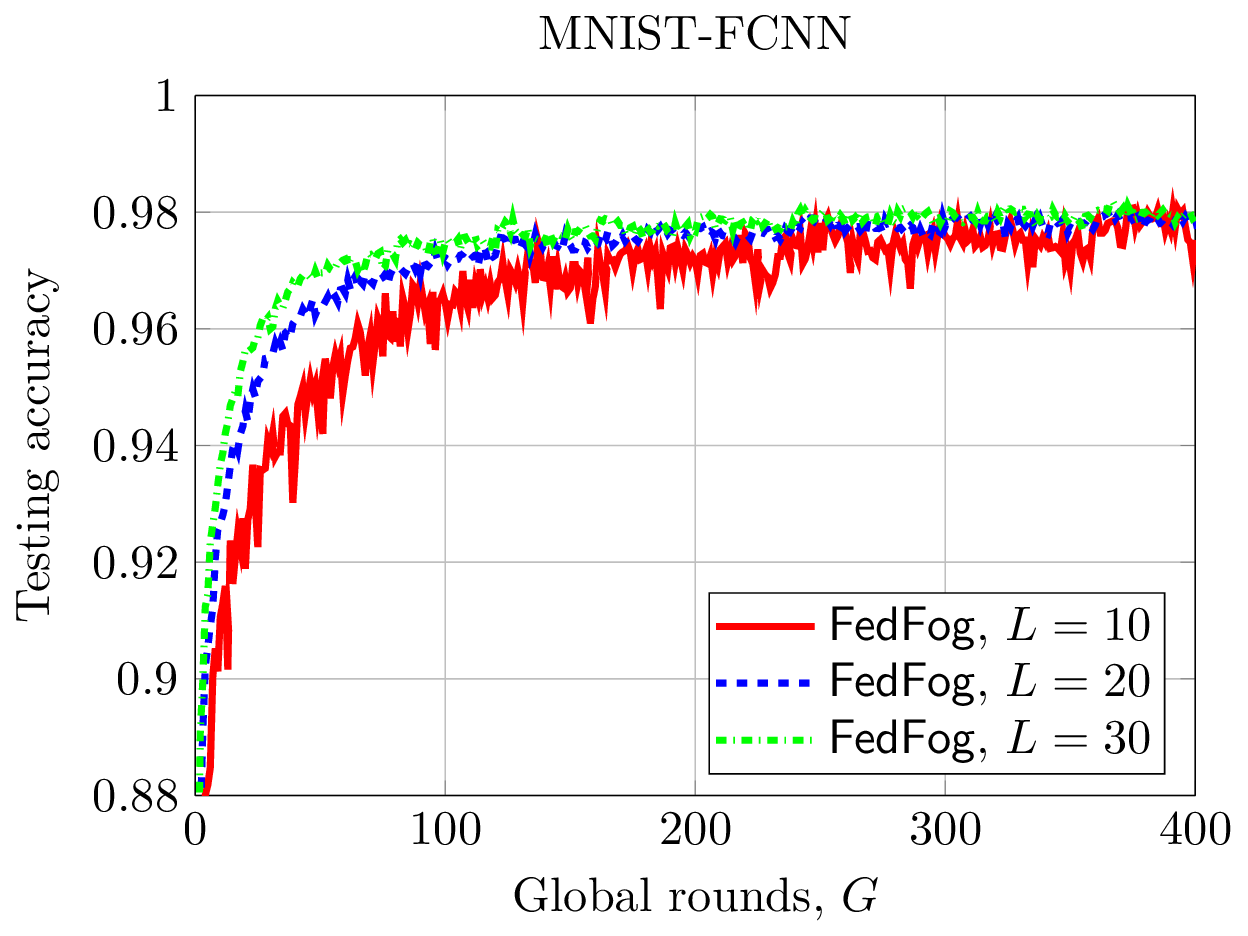}} \\
\subfigure{%
\label{fig:6-c}%
\includegraphics[height=2.0in]{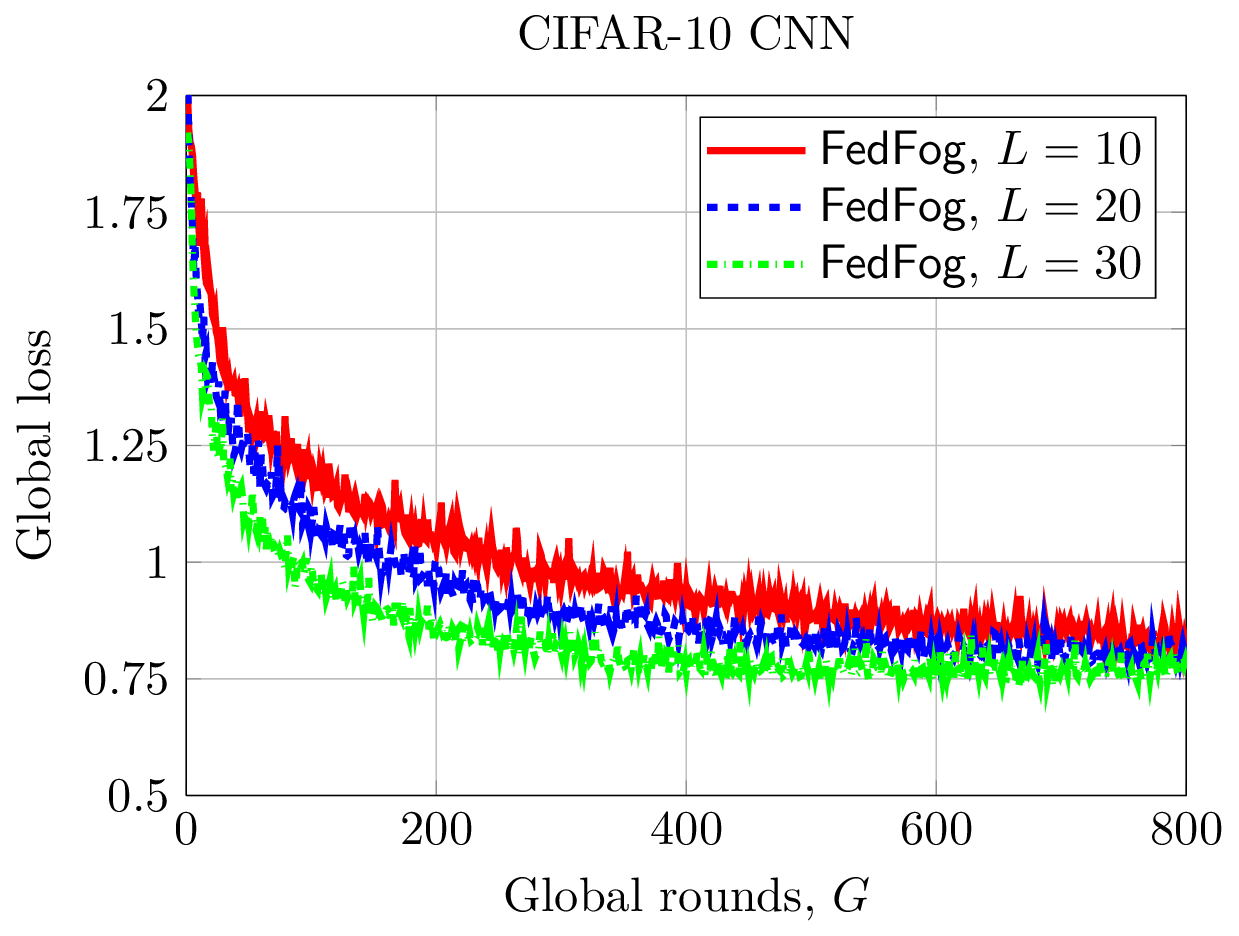}}%
\hspace{4pt}%
\subfigure{%
\label{fig:6-d}%
\includegraphics[height=2.0in]{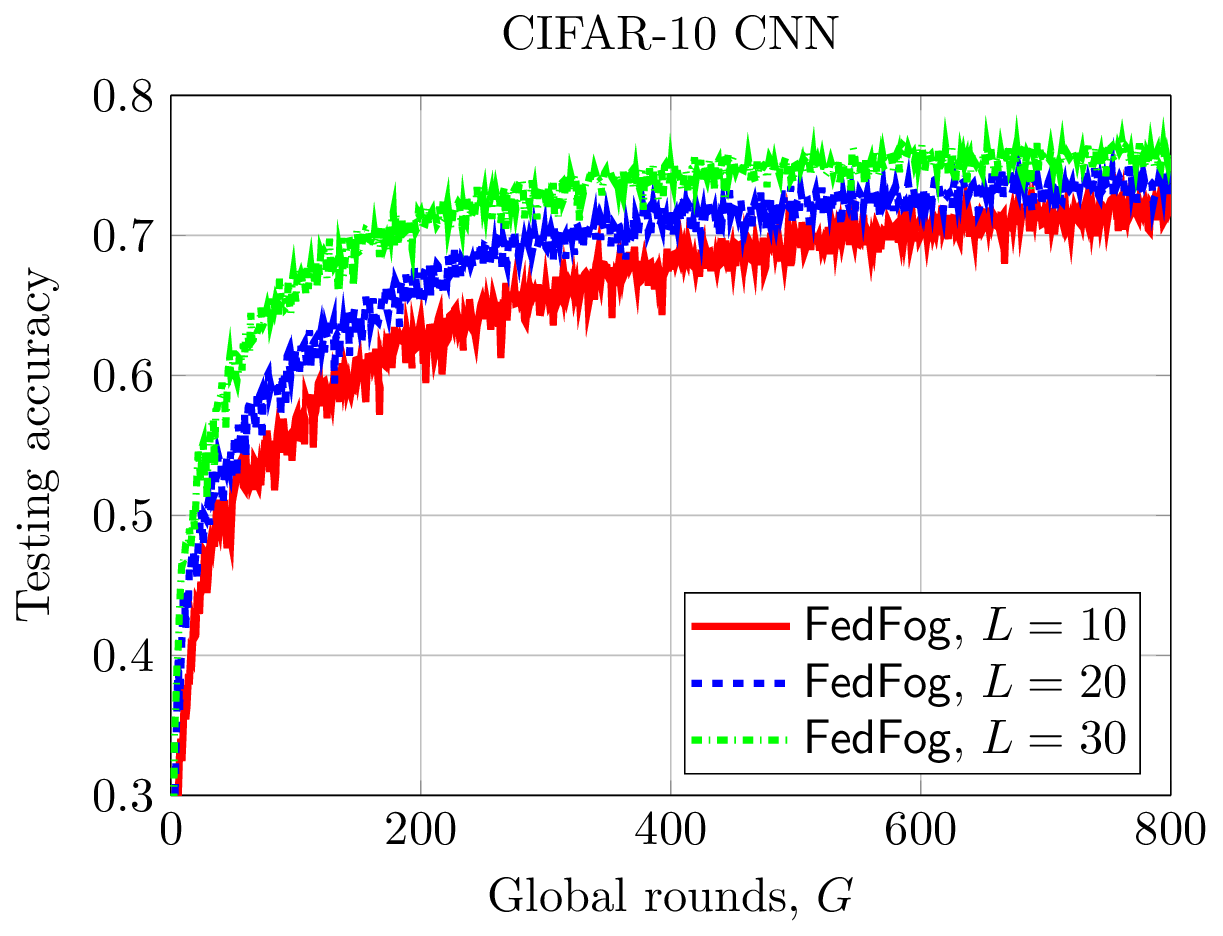}}%
\caption[]{Effect of different numbers of local iterations $L$  on the convergence of $\FedFog$, with $B=20$.}%
\label{fig:Alg1BatchLocalIter}%
\end{figure}

In wireless networks, the communication delay can dominate  computation delay, and therefore  UEs tends to  perform more local updates before sending them to CS, resulting less global model updates.
In Fig. \ref{fig:Alg1BatchLocalIter}, we show the convergence rate of $\FedFog$ with different values of local iterations $L$ and $B=20$. In all settings, the larger $L$ has a positive impact of the convergence speed of $\FedFog$; however, a very large number of local iterations also lead to high computation latency  and divergent convergence. Hence, it is beneficial to choose the appropriate values of $L$ and $B$, which not only boosts the convergence speed but also balances trade-off between  computations and communications. In the following simulations, we set 
$B=20$ and $L=20$.

\subsection{Numerical Results for $\FedFog$ over Wireless Fog-Cloud  Systems (Algorithms \ref{alg:networkawarealg} and  \ref{alg:networkawarealgFUA})}

 \begin{figure}[!ht]%
\centering
\subfigure{%
\label{fig:7-a}%
\includegraphics[height=2.0in]{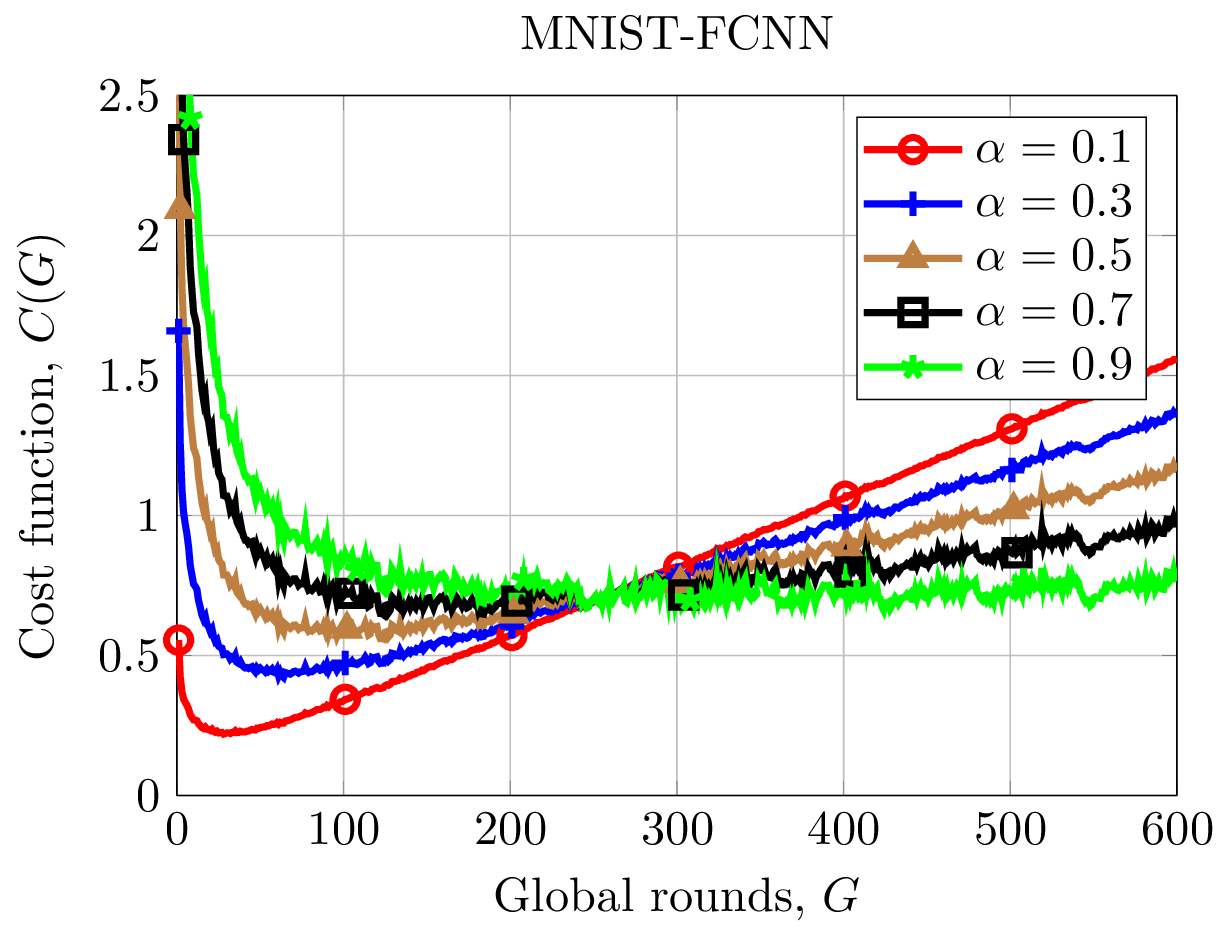}}%
\hspace{4pt}%
\subfigure{%
\label{fig:7-b}%
\includegraphics[height=2.0in]{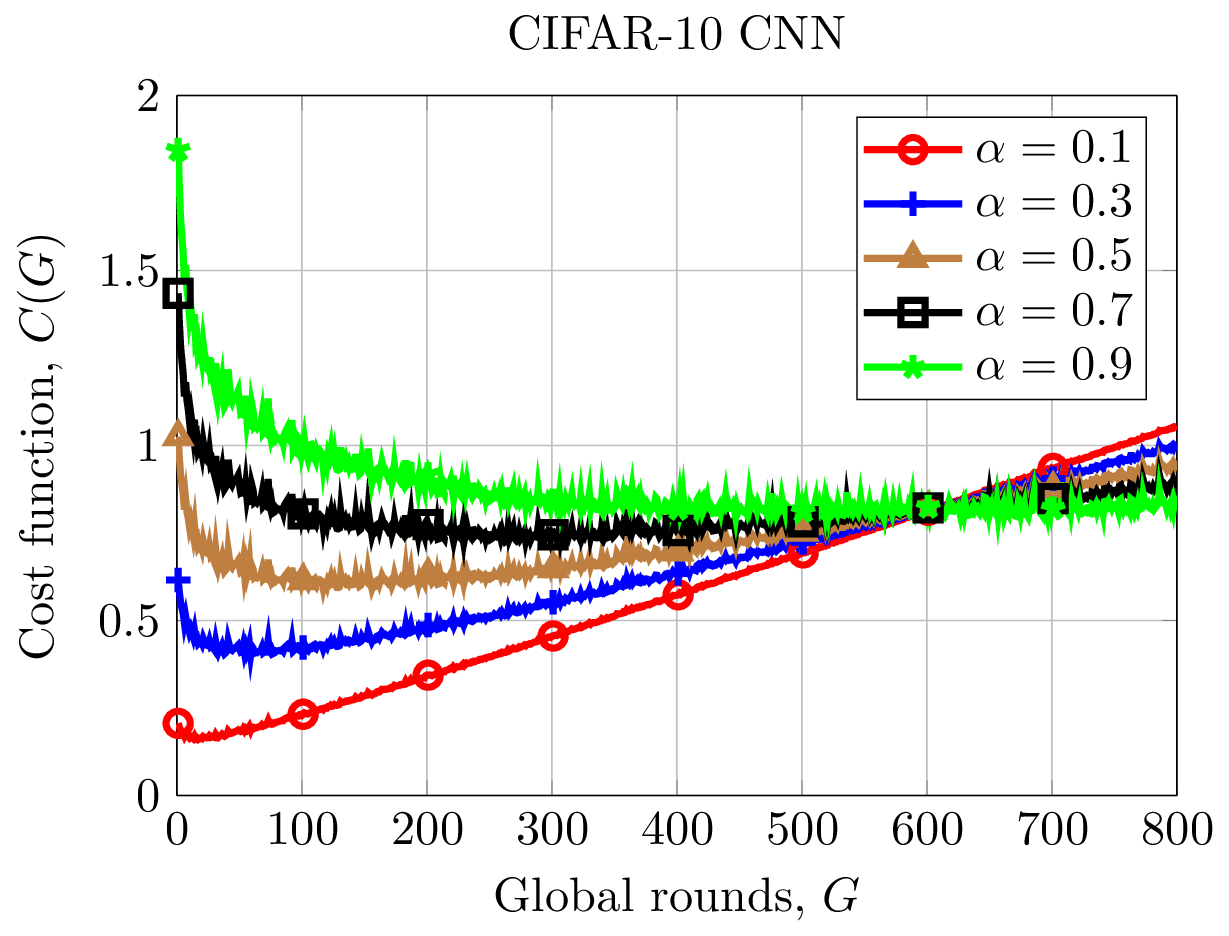}} 
\caption[]{Average $C(G)$ performances of Algorithm \ref{alg:networkawarealg} with different values of $\alpha$.}%
\label{fig:Alg3vsalpha}%
\end{figure}

Fig. \ref{fig:Alg3vsalpha} depicts the average $C(G)$ performances of Algorithm \ref{alg:networkawarealg}  with different values of the priority parameter $\alpha$.
As can be seen from this figure that, with small value of $\alpha$, Algorithm \ref{alg:networkawarealg} obtains a minimum cost function at small value of $G$, which may lead to an improper early termination. {\color{black}The reason is that, when $\alpha$ is small, the completion time, which is an increasing function of $G$,  takes more effect on the cost function than the loss function}. As expected, a larger value of $\alpha$ provides a better  balance between the accuracy of the learning model and the running cost. Therefore, we set $\alpha=0.7$ in the following results.

 \begin{figure}[!ht]%
\centering
\subfigure{%
\label{fig:8-a}%
\includegraphics[height=2.0in]{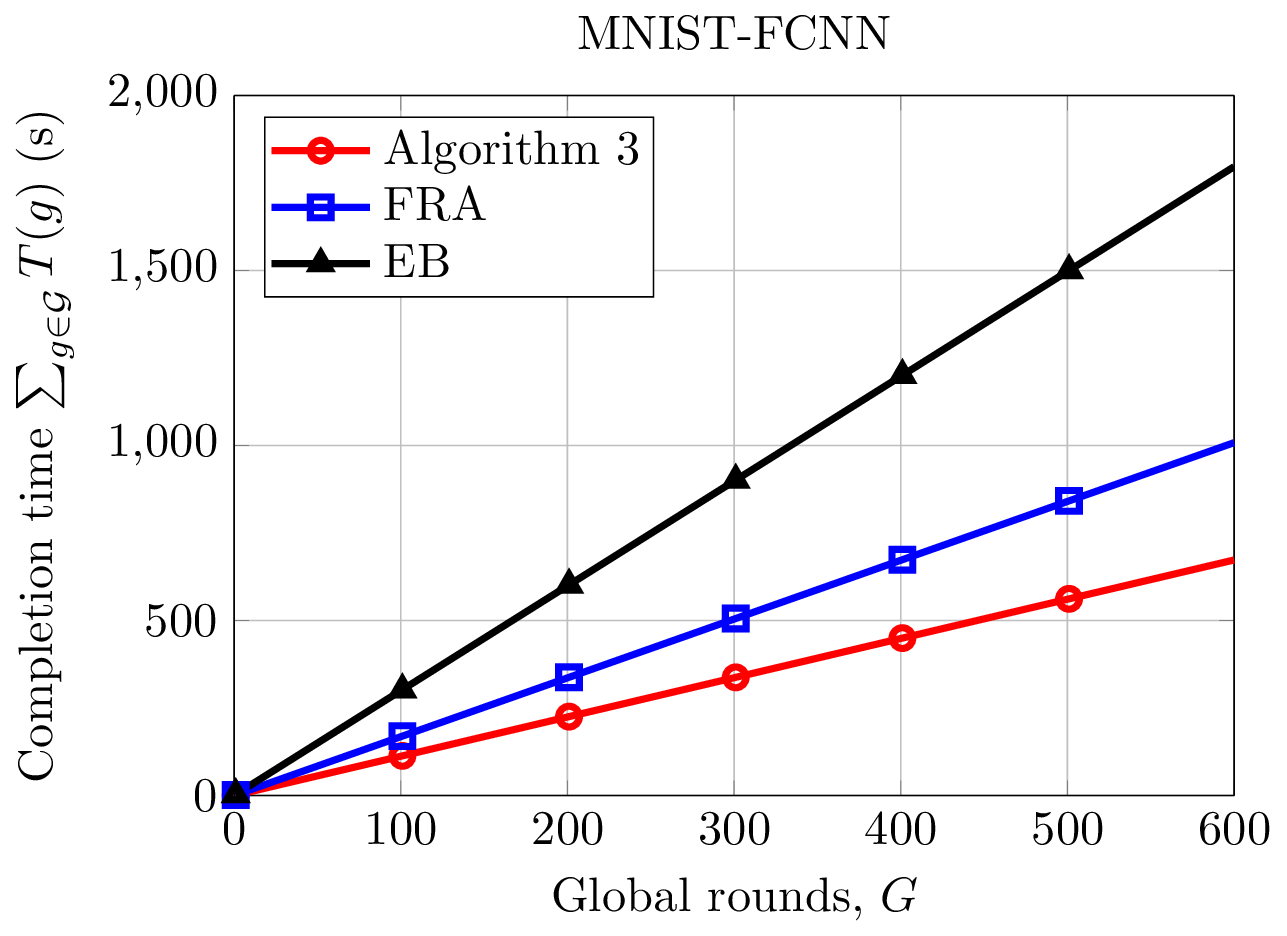}}%
\hspace{4pt}%
\subfigure{%
\label{fig:8-b}%
\includegraphics[height=2.0in]{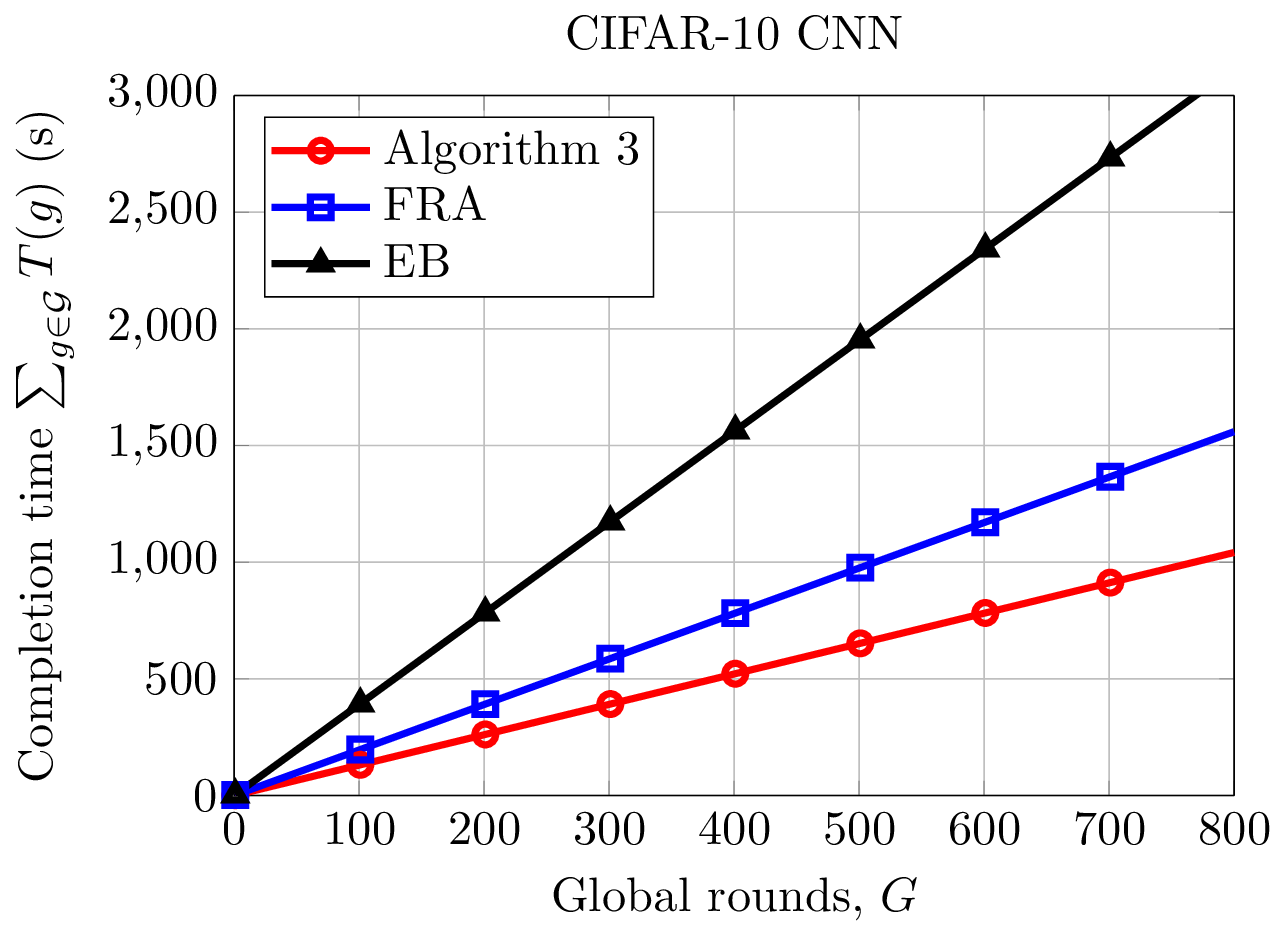}} 
\caption[]{Completion time versus $G$ for different schemes.}%
\label{fig:Alg3CompletionTimevsG}%
\end{figure}

In Fig. \ref{fig:Alg3CompletionTimevsG}, we show the performance comparison in terms of completion times among the considered schemes versus the number of global rounds. Clearly, Algorithm \ref{alg:networkawarealg} outperforms the baseline schemes in all ranges of $G$, which is even  deeper when $G$ is large. The EB, which fairly allocates the fixed bandwidth to UEs (i.e., $\beta_{ij}(g)=1/J,\, \forall i,j,g$), provides the worst performance as the  bandwidth allocated to each UE  has a great impact on both  UL and DL transmission latency, leading to the serious straggler effects. These observations demonstrate the effectiveness of the proposed Algorithm \ref{alg:networkawarealg} by jointly optimizing the transmit power, CPU-frequency and bandwidth. We can also see that the completion time of CIFAR-10 CNN is much higher than that of MNIST-FCNN since the latter has larger sizes (in bits) of data and model training than  the former. 

 \begin{figure}[!ht]%
\centering
\subfigure{%
\label{fig:9-a}%
\includegraphics[height=2.0in]{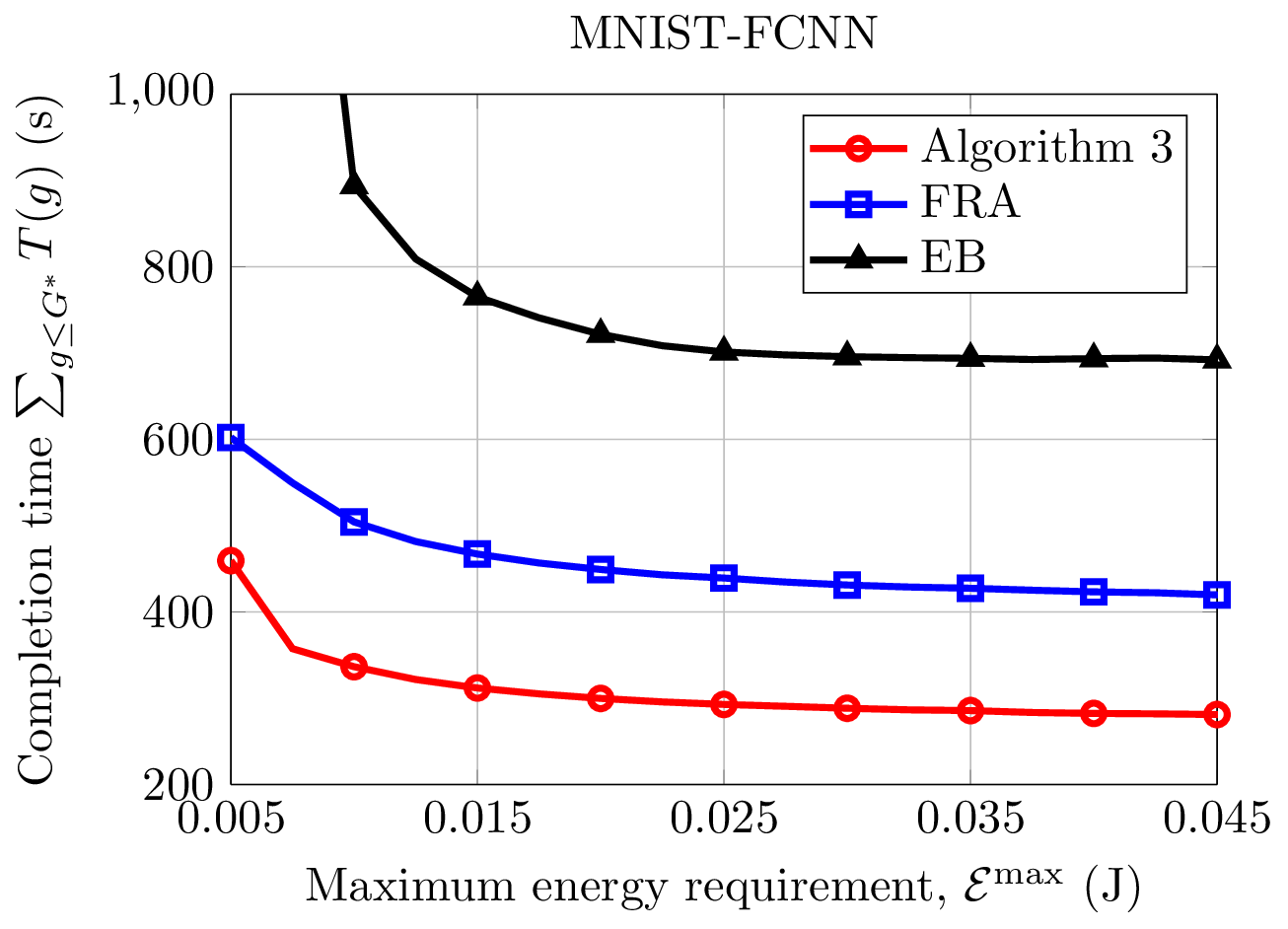}}%
\hspace{4pt}%
\subfigure{%
\label{fig:9-b}%
\includegraphics[height=2.0in]{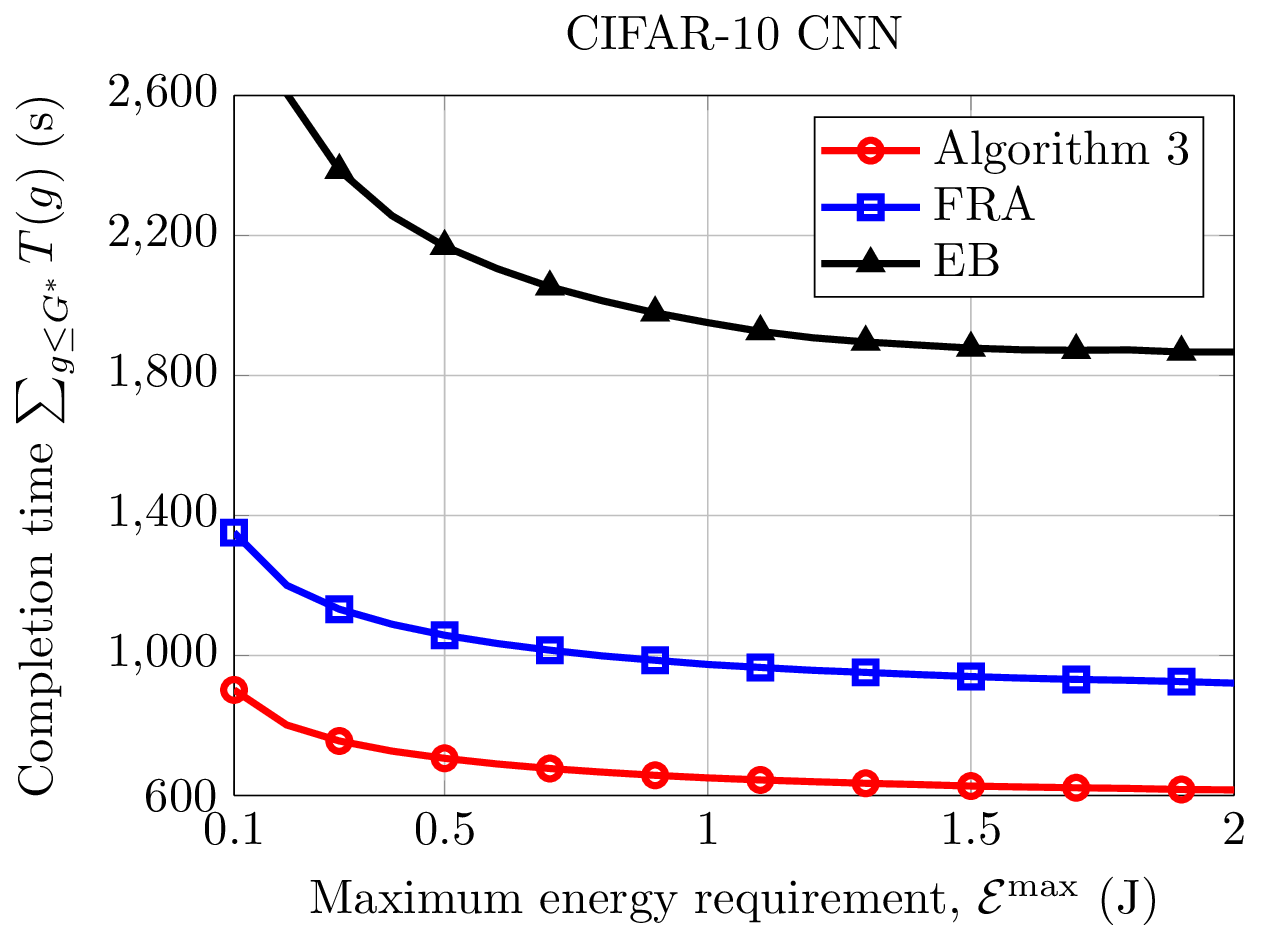}} 
\caption[]{Trade-off between completion time and  maximum  energy  consumption  requirement.}%
\label{fig:Alg3CompletionTimevsEmax}%
\end{figure}

 The impact of the maximum  energy  consumption  requirement, $\mathcal{E}^{\max}$, on the completion time is plotted in Fig. \ref{fig:Alg3CompletionTimevsEmax}, where we set $\bar{k}=5$ for the stopping condition.  Increasing the threshold $\mathcal{E}^{\max}$ results in lower completion times for all schemes. This phenomenon is not surprising because
 with a larger value of $\mathcal{E}^{\max}$, more power and CPU-frequency of UEs can be used for the global training update and local model training subject to constraint \eqref{eq:OP1b}. Again, Algorithm \ref{alg:networkawarealg} still offers the best performance out of the schemes considered.

\begin{figure}[!ht]%
\centering
\subfigure[Algorithm \ref{alg:networkawarealgFUA}: varying $\Delta\mathcal{T}$]{%
\label{fig:10-a}%
\includegraphics[height=2.0in]{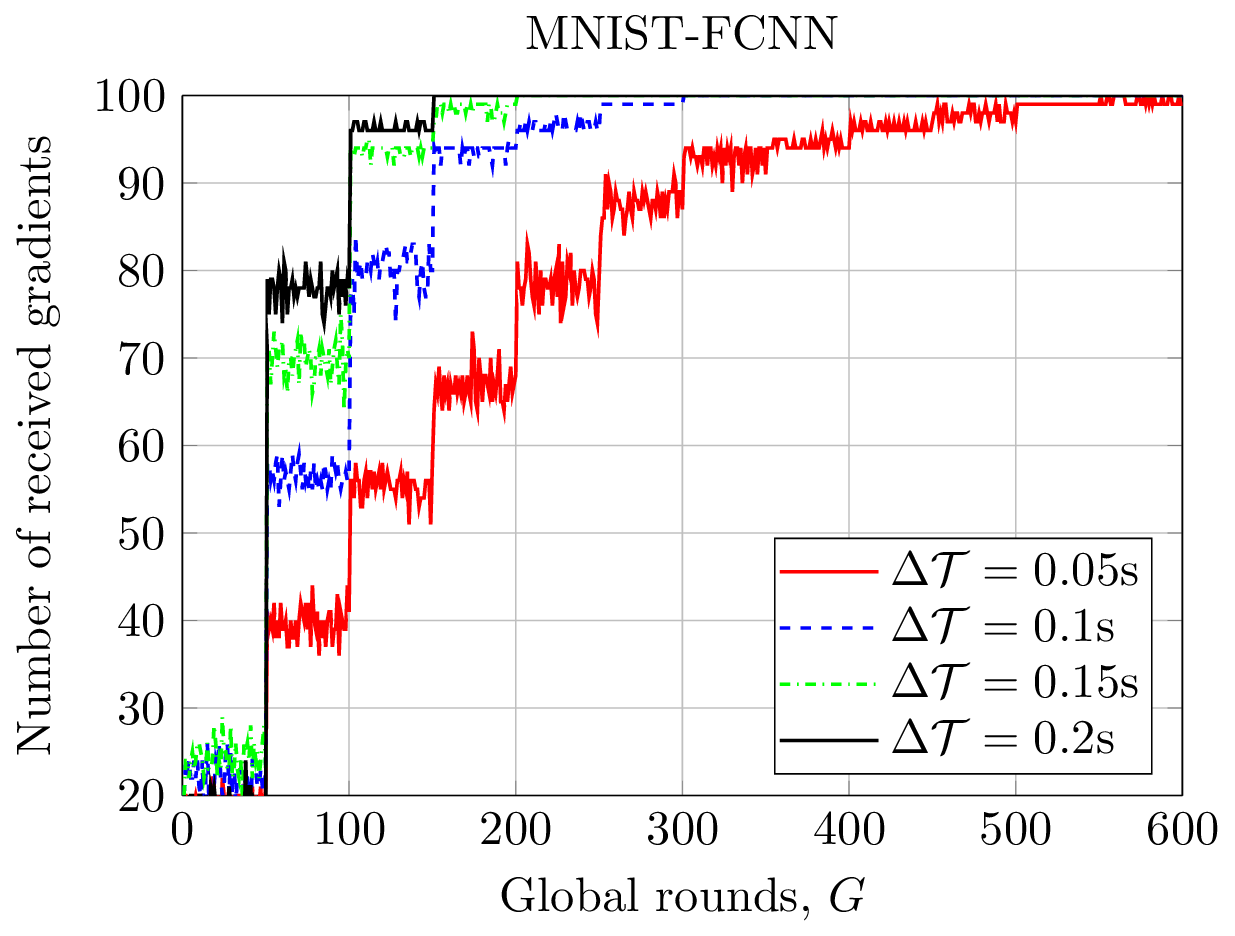}
}%
\hspace{4pt}%
\subfigure[Performance comparison between Algorithm \ref{alg:networkawarealgFUA} and baseline schemes with $\Delta\mathcal{T}=0.15$s]{%
\label{fig:10-b}%
\includegraphics[height=2.0in]{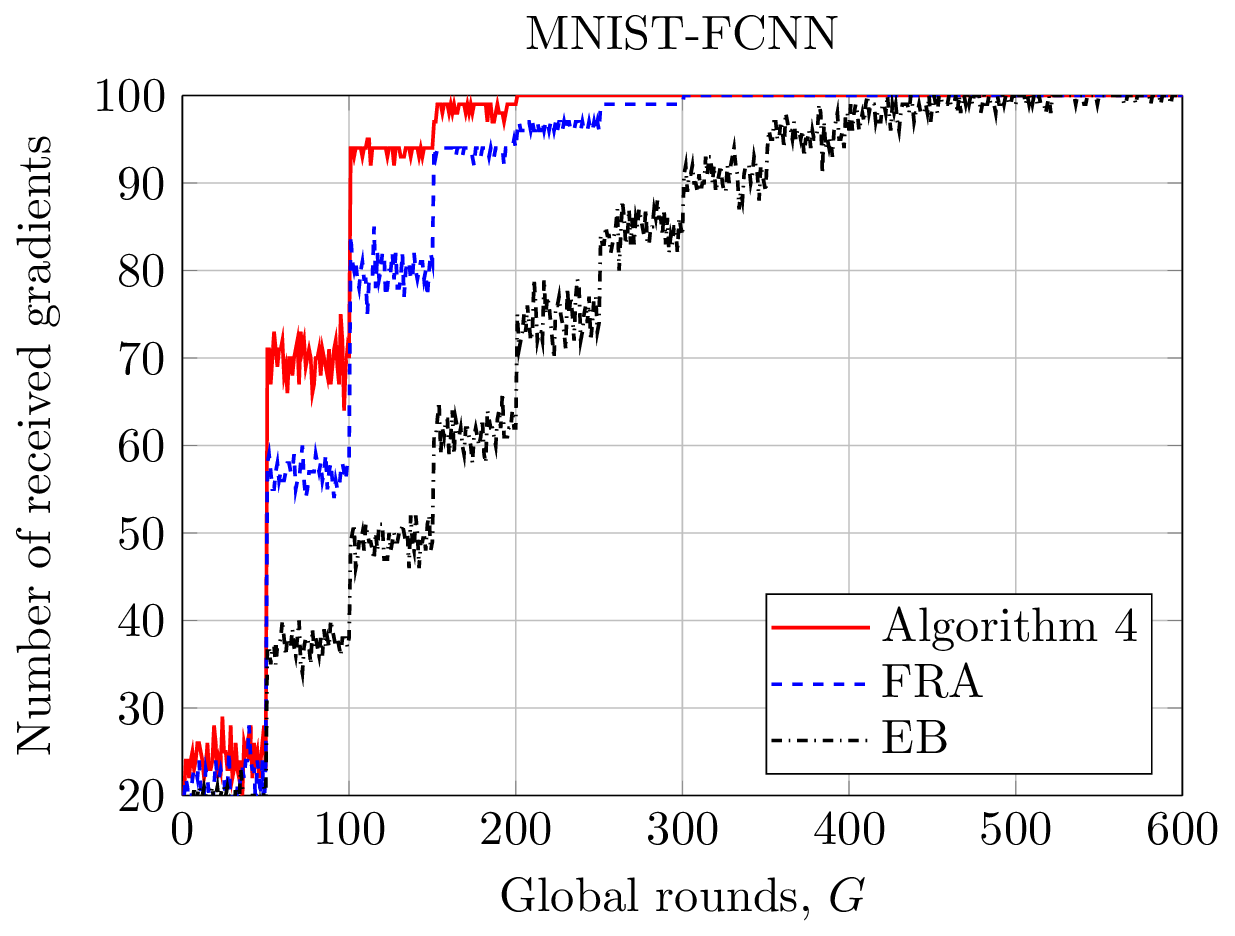}} 
\caption[]{Number of received gradients for flexible user aggregation-based schemes, with $J_{\min}=20$.}%
\label{fig:Alg4vsNoGradient}%
\end{figure}

\begin{figure}[!ht]%
\centering
\subfigure{%
\label{fig:11-a}%
\includegraphics[height=2.0in]{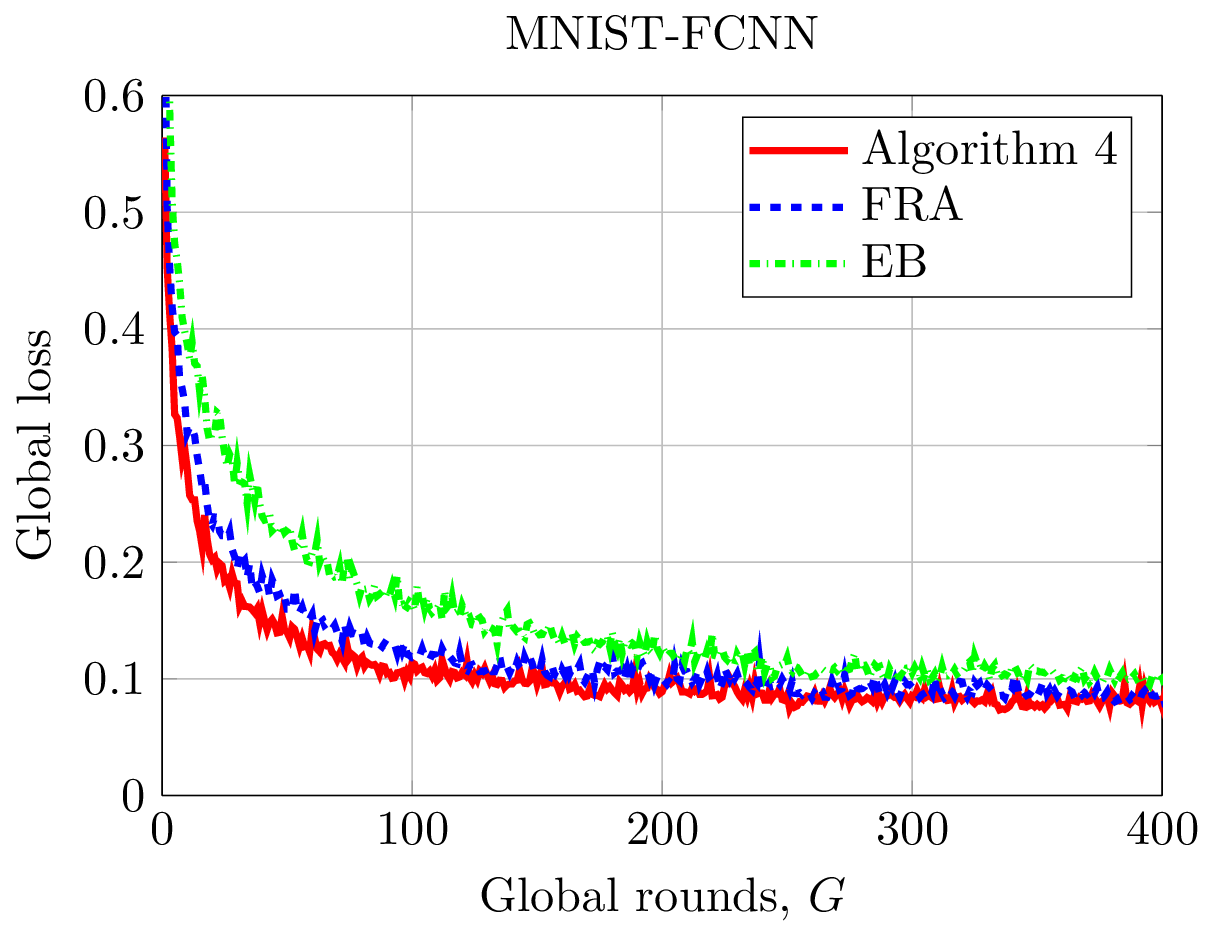}
}%
\hspace{4pt}%
\subfigure{%
\label{fig:11-b}%
\includegraphics[height=2.0in]{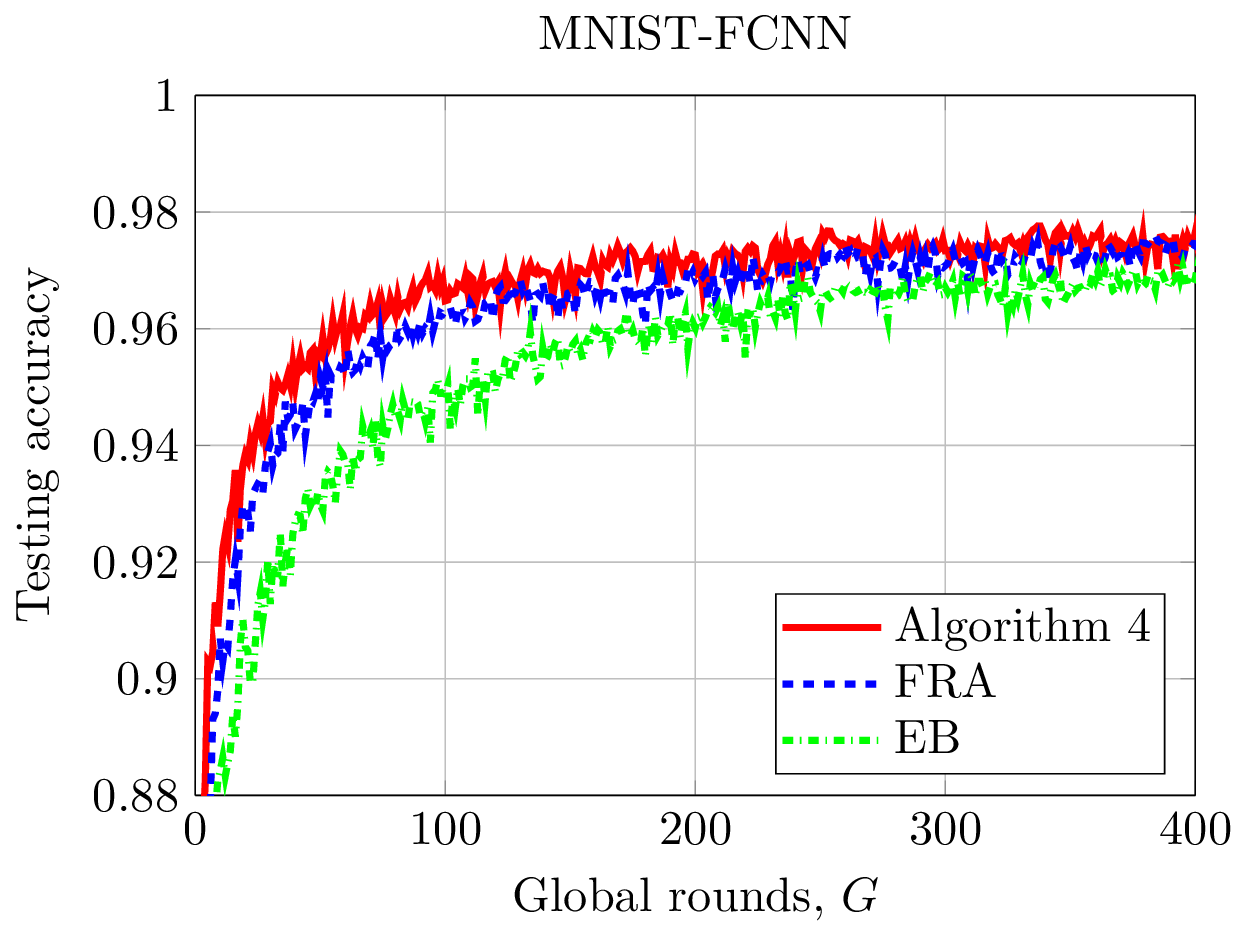}} 
\caption[]{Performance comparison between Algorithm \ref{alg:networkawarealgFUA} and baseline schemes, with $J_{\min}=20$ and $\Delta\mathcal{T}=0.15$s.}
\label{fig:Alg4vsLossAccF11}%
\end{figure}

Fig. \ref{fig:Alg4vsNoGradient} characterizes the number of received gradients for flexible user aggregation-based schemes. We note that it is often not beneficial to completely satisfy the condition \eqref{eq:ACClevel}. Therefore, weaker UEs are allowed to participate in  global updates earlier after a fixed number of global rounds, say $\Delta G$. Here we set $\Delta G=50$, which is numerically shown to  significantly accelerate the convergence rate of $\FedFog$.
 As can be seen in Fig. \ref{fig:Alg4vsNoGradient}(a) that increasing $\Delta\mathcal{T}$ results in higher number of received gradients. However, a large $\Delta\mathcal{T}$ (i.e., $\Delta\mathcal{T}=0.2$s) will not only bring less benefit in terms of the number of received gradients, but also lead to a higher training time. In addition, the results in Fig. \ref{fig:Alg4vsNoGradient}(b) show that Algorithm \ref{alg:networkawarealgFUA} can boost the number of received gradients compared to the baseline schemes with the same completion time. This results in better model training, as demonstrated in Fig. \ref{fig:Alg4vsLossAccF11}. 

\begin{figure}[!ht]%
\centering
\subfigure{%
\label{fig:12-a}%
\includegraphics[height=2.0in]{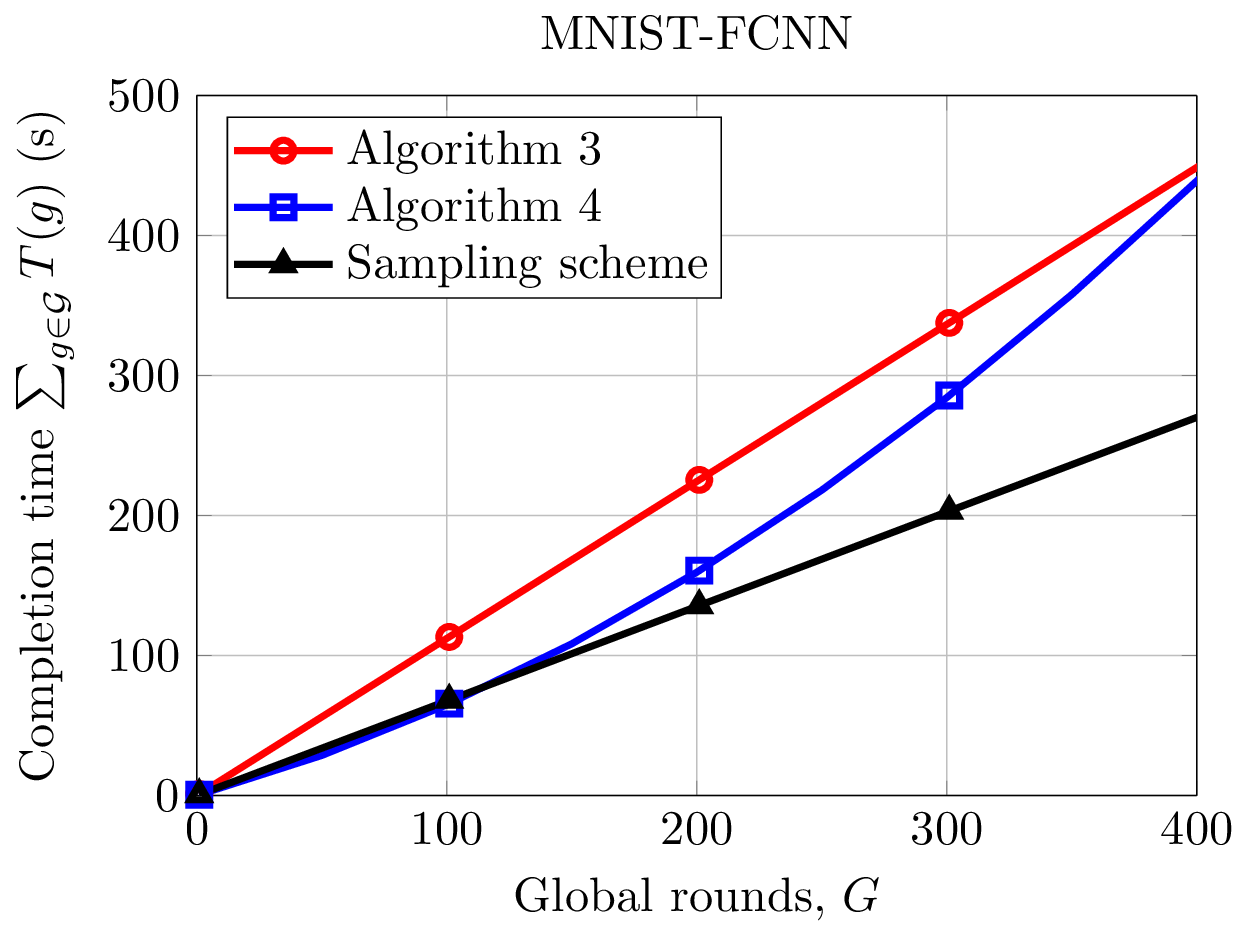}
}%
\hspace{4pt}%
\subfigure{%
\label{fig:12-b}%
\includegraphics[height=2.0in]{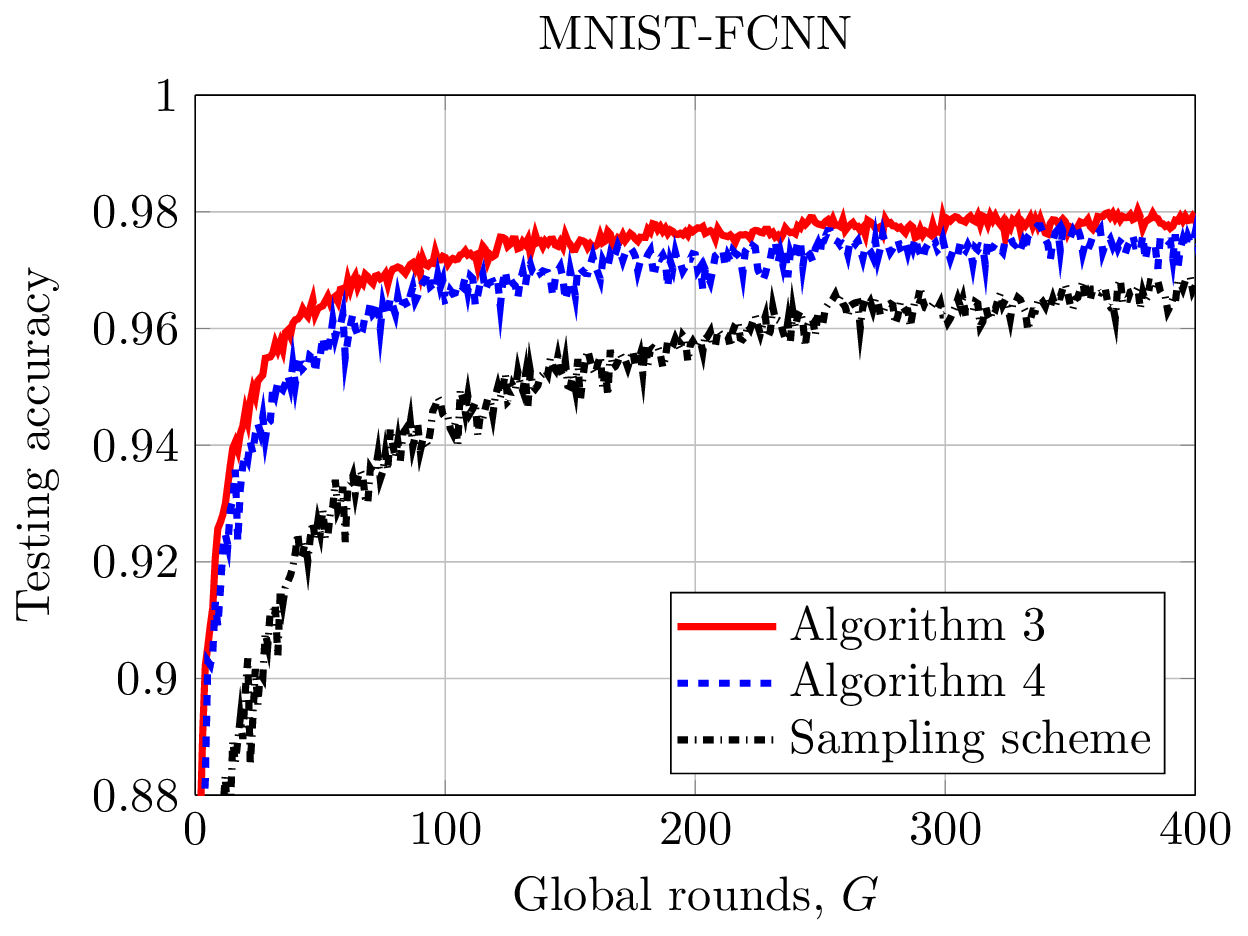}} 
\caption[]{Performance comparison between different schemes, with $J_{\min}=20, J(g)=10,\forall g$ and $\Delta\mathcal{T}=0.15$s.}
\label{fig:Alg4vsLossAccF12}%
\end{figure}

Lastly, we compare the performance between the proposed algorithms and the sampling scheme in Fig. \ref{fig:Alg4vsLossAccF12}, with $J(g) = 10, \forall g$ for the latter. As expected, Algorithm \ref{alg:networkawarealgFUA} requires much less completion time in Fig. \ref{fig:Alg4vsLossAccF12}(a) while still achieving the comparable accuracy of the learning model in Fig.  \ref{fig:Alg4vsLossAccF12}(b), compared to Algorithm \ref{alg:networkawarealg}. On the other hand, the sampling scheme has the lowest completion time due to a small number of UEs participated in the training process, but exhibiting a much slower convergence speed than Algorithms \ref{alg:networkawarealg} and \ref{alg:networkawarealgFUA}. In other words, Algorithm \ref{alg:networkawarealgFUA} offers a good balance  between the quality of learning model and the communication cost.

\section{Conclusion}\label{sec:Conclusion}
In this paper, we proposed the network-aware FL algorithm for wireless fog-cloud systems, a novel methodology for optimizing the distribution of ML tasks across users while tackling inherent issues in fog-cloud scenarios. We characterized the overall running cost of implementing the proposed $\FedFog$ algorithm in discrete time intervals, taking into account the effects of both computation and communication. The cost function of the formulated problem captures  the global loss and overall cost in terms of the training time,  which is used to design the iteration-stopping criteria to  produce  a  desirable  number of  global  rounds. The proposed scheme can avoid the redundant cost with negligible negative impact on the convergence rate and accuracy of the learning model. We also developed a flexible user aggregation strategy to mitigate the straggler effect, resulting in less completion time of $\FedFog$ over wireless fog-cloud systems. Numerical results with popular ML tasks were provided to validate the effectiveness of the network-aware FL algorithm compared to existing baseline approaches. {\color{black}For future works, it would be interesting to develop and implement prototypes to validate the efficiency of $\FedFog$ in real-time environments.}

\appendices
\renewcommand{\thesectiondis}[2]{\Alph{section}:}
\section{Proof of Lemmas \ref{lema:1}-\ref{lema:3}} \label{app: lemma123}
\renewcommand{\theequation}{\ref{app: lemma123}.\arabic{equation}}\setcounter{equation}{0}

\subsection{Proof of Lemma \ref{lema:1}}
 Since the variances of  the stochastic gradients of any two UEs are independent, it follows that
\begin{align}
     &\frac{1}{J^2}\mathbb{E}\Bigr\{\bigr\|\sum_{i\in\mathcal{I}}\sum_{j\in\mathcal{J}_i} \bigl(\nabla F_{ij}(\mathbf{w}^g_{ij,\ell})-\nabla  \bar{F}_{ij}(\mathbf{w}^g_{ij,\ell})\bigl) \bigl\|^2\Bigr\} \nonumber\\
     &=     \frac{1}{J^2}\sum_{i\in\mathcal{I}}\sum_{j\in\mathcal{J}_i}\mathbb{E}\Bigr\{\bigr\| \bigl(\nabla F_{ij}(\mathbf{w}^g_{ij,\ell})-\nabla  \bar{F}_{ij}(\mathbf{w}^g_{ij,\ell})\bigl) \bigl\|^2 \Bigl\}\nonumber\\
         &\qquad + \frac{1}{J^4}\sum_{(i,j)\neq (i',j')}\underbrace{\mathbb{E}\Bigr\{\bigr\| \bigl(\nabla F_{ij}(\mathbf{w}^g_{ij,\ell})-\nabla  \bar{F}_{ij}(\mathbf{w}^g_{ij,\ell})\bigl)\bigl(\nabla F_{i'j'}(\mathbf{w}^g_{i'j',\ell})-\nabla  \bar{F}_{i'j'}(\mathbf{w}^g_{i'j',\ell})\bigl) \bigl\|^2\Bigl\}}_{=\, 0} \nonumber\\
   & \leq \frac{1}{J^2}\sum_{i\in\mathcal{I}}\sum_{j\in\mathcal{J}_i}\gamma_{ij}^2\ (\text{by Assumption \ref{assp:2}}).
\end{align}

\subsection{Proof of Lemma \ref{lema:2}}
 From the definition of $\bar{\mathbf{w}}^g_{\ell}$ there always exists $\ell' \leq \ell,$ that satisfies $\ell-\ell' \leq L-1$. First, it is true that $\| \bar{\mathbf{w}}^g_{\ell} - \mathbf{w}^g_{ij,\ell} \bigl\|^2 = \| (\bar{\mathbf{w}}^g_{\ell} - \bar{\mathbf{w}}^g_{\ell'}) - (\mathbf{w}^g_{ij,\ell} - \bar{\mathbf{w}}^g_{\ell'})\|^2 = \|\bar{\mathbf{w}}^g_{\ell} - \bar{\mathbf{w}}^g_{\ell'}\|^2 - 2\langle\bar{\mathbf{w}}^g_{\ell} - \bar{\mathbf{w}}^g_{\ell'},\mathbf{w}^g_{ij,\ell} - \bar{\mathbf{w}}^g_{\ell'}\rangle + \|\mathbf{w}^g_{ij,\ell} - \bar{\mathbf{w}}^g_{\ell'}\|^2$. In addition, we have $\langle\bar{\mathbf{w}}^g_{\ell} - \bar{\mathbf{w}}^g_{\ell'},\sum_{i\in\mathcal{I}}\sum_{j\in\mathcal{J}_i}\mathbf{w}^g_{ij,\ell} - \bar{\mathbf{w}}^g_{\ell'}\rangle = \|\bar{\mathbf{w}}^g_{\ell} - \bar{\mathbf{w}}^g_{\ell'}\|^2$. 
Thus, it follows that
\begin{align}
 \frac{1}{J}\mathbb{E}\Bigr\{\sum_{i\in\mathcal{I}}\sum_{j\in\mathcal{J}_i}\bigr\| \bar{\mathbf{w}}^g_{\ell} - \mathbf{w}^g_{ij,\ell} \bigl\|^2\Bigr\} &= \frac{1}{J}\sum_{i\in\mathcal{I}}\sum_{j\in\mathcal{J}_i}\mathbb{E}\{\|\mathbf{w}^g_{ij,\ell} - \bar{\mathbf{w}}^g_{\ell'}\|^2\} \bigr] -\frac{1}{J}\mathbb{E}\{\|\bar{\mathbf{w}}^g_{\ell} - \bar{\mathbf{w}}^g_{\ell'}\|^2\} \nonumber\\
 &= \frac{\eta_g^2}{J}\sum_{i\in\mathcal{I}}\sum_{j\in\mathcal{J}_i}\bigl\|\sum_{t=\ell'}^{\ell-1}\nabla F_{ij}(\mathbf{w}^g_{ij,t})\bigr\|^2 - \frac{\eta_g^2}{J}\bigl\|\sum_{i\in\mathcal{I}}\sum_{j\in\mathcal{J}_i}\sum_{t=\ell'}^{\ell-1}\nabla F_{ij}(\mathbf{w}^g_{ij,t})\bigr\|^2\nonumber\\
 &\leq \frac{\eta_g^2}{J}\sum_{i\in\mathcal{I}}\sum_{j\in\mathcal{J}_i}\bigl\|\sum_{t=\ell'}^{\ell-1}\nabla F_{ij}(\mathbf{w}^g_{ij,t})\bigr\|^2
 \leq \frac{\eta_g^2L}{J}\sum_{i\in\mathcal{I}}\sum_{j\in\mathcal{J}_i}\sum_{t=\ell'}^{\ell -1}\bigl
 \|\nabla F_{ij}(\mathbf{w}^g_{ij,t})\bigr\|^2
 \nonumber\\
 &\leq (\ell -\ell')L\eta_g^2\delta^2 \leq (L-1)L\eta_g^2\delta^2\ (\text{due to}\ \ell-\ell' \leq L-1).
 \end{align}

\subsection{Proof of Lemma \ref{lema:3}}
Given $\bar{\mathbf{w}}^g_{\ell+1}=\bar{\mathbf{w}}^g_{\ell} - \eta_g\nabla F(\bar{\mathbf{w}}^g_{\ell})$ and $\nabla \bar{F}(\bar{\mathbf{w}}^g_{\ell}) \triangleq \frac{1}{J}\sum_{i\in\mathcal{I}}\sum_{j\in\mathcal{J}_i}\nabla \bar{F}_{ij}(\mathbf{w}^g_{ij,\ell}) =  \mathbb{E}\{\nabla F(\bar{\mathbf{w}}^g_{\ell})\}$, it follows that
\begin{align}\label{eq:A6}
    \|\bar{\mathbf{w}}^g_{\ell+1} - \mathbf{w}^*\|^2 &= \|\bar{\mathbf{w}}^g_{\ell}  - \mathbf{w}^* - \eta_g\nabla \bar{F}(\bar{\mathbf{w}}^g_{\ell}) + \eta_g(\nabla \bar{F}(\bar{\mathbf{w}}^g_{\ell})-\nabla F(\bar{\mathbf{w}}^g_{\ell}))\|^2 \nonumber\\
   & = \underbrace{\|\bar{\mathbf{w}}^g_{\ell}  - \mathbf{w}^* - \eta_g\nabla \bar{F}(\bar{\mathbf{w}}^g_{\ell})\|^2}_{\triangleq A} + \underbrace{\eta_g^2\|\nabla \bar{F}(\bar{\mathbf{w}}^g_{\ell})-\nabla F(\bar{\mathbf{w}}^g_{\ell})\|^2}_{\triangleq B}
\end{align}
since $\mathbb{E}\{\langle\bar{\mathbf{w}}^g_{\ell}  - \mathbf{w}^* - \eta_g\nabla \bar{F}(\bar{\mathbf{w}}^g_{\ell}), \nabla \bar{F}(\bar{\mathbf{w}}^g_{\ell})-\nabla F(\bar{\mathbf{w}}^g_{\ell}) \rangle\} = 0$. We first focus on the expected bound of $A$ by rewriting it as:
\begin{align}\label{eq:A7}
    A = \underbrace{\|\bar{\mathbf{w}}^g_{\ell}  - \mathbf{w}^*\|^2}_{\triangleq A_1}  \underbrace{-2\eta_g\langle\bar{\mathbf{w}}^g_{\ell}  - \mathbf{w}^*, \nabla \bar{F}(\bar{\mathbf{w}}^g_{\ell})\rangle}_{\triangleq A_2}  \underbrace{+\eta_g^2\|\nabla \bar{F}(\bar{\mathbf{w}}^g_{\ell})\|^2}_{\triangleq A_3}.
\end{align}
For $A_2$ and from Assumption \ref{assp:1} on $\lambda$-strongly convex, we have
{\begin{align}\label{eq:A8}
    &A_2 = -2\eta_g\frac{1}{J}\sum_{i\in\mathcal{I}}\sum_{j\in\mathcal{J}_i}\langle\bar{\mathbf{w}}^g_{\ell}  - \mathbf{w}^*, \nabla \bar{F}_{ij}(\mathbf{w}^g_{ij,\ell})\rangle \nonumber\\
     &=  -2\eta_g\frac{1}{J}\sum_{i\in\mathcal{I}}\sum_{j\in\mathcal{J}_i}\langle\bar{\mathbf{w}}^g_{\ell}  - \mathbf{w}^g_{ij,\ell}, \nabla \bar{F}_{ij}(\mathbf{w}^g_{ij,\ell})\rangle -2\eta_g\frac{1}{J}\sum_{i\in\mathcal{I}}\sum_{j\in\mathcal{J}_i}\langle\mathbf{w}^g_{ij,\ell}  - \mathbf{w}^*, \nabla \bar{F}_{ij}(\mathbf{w}^g_{ij,\ell})\rangle \nonumber\\
   &\leq \frac{1}{J}\sum_{i\in\mathcal{I}}\sum_{j\in\mathcal{J}_i}\Bigl(\|\bar{\mathbf{w}}^g_{\ell}  - \mathbf{w}^g_{ij,\ell}\|^2 + \eta^2_g \|\nabla \bar{F}_{ij}(\mathbf{w}^g_{ij,\ell})\|^2 -2\eta_g \bigl(\bar{F}_{ij}(\mathbf{w}^g_{ij,\ell}) - \bar{F}_{ij}(\mathbf{w}^*) + \frac{\lambda}{2}\|\mathbf{w}^g_{ij,\ell} - \mathbf{w}^*\|^2\bigr)\Bigr).
\end{align}}
For $A_3$ and from Assumption \ref{assp:1} on $\mu-$smooth, we have
{\small\begin{IEEEeqnarray}{rCl}\label{eq:A9}
    A_3 = \eta_g^2\|\frac{1}{J}\sum_{i\in\mathcal{I}}\sum_{j\in\mathcal{J}_i}\nabla \bar{F}_{ij}(\mathbf{w}^g_{ij,\ell})\|^2 \leq \eta_g^2\frac{1}{J}\sum_{i\in\mathcal{I}}\sum_{j\in\mathcal{J}_i}\|\nabla \bar{F}_{ij}(\mathbf{w}^g_{ij,\ell})\|^2  
    \leq \frac{2\mu\eta_g^2}{J}\sum_{i\in\mathcal{I}}\sum_{j\in\mathcal{J}_i}\bigl(\bar{F}_{ij} (\mathbf{w}^g_{ij,\ell}) - F_{ij}^*  \bigr).\qquad\
\end{IEEEeqnarray}}
Substituting \eqref{eq:A8} and \eqref{eq:A9} into \eqref{eq:A7}, it follows that
\begin{align}\label{eq:A10}
    A \leq (1-0.5\lambda \eta_g )\|\bar{\mathbf{w}}^g_{\ell}  - \mathbf{w}^*\|^2 + (1+\lambda\eta_g)\frac{1}{J}\sum_{i\in\mathcal{I}}\sum_{j\in\mathcal{J}_i}\|\bar{\mathbf{w}}^g_{\ell}  - \mathbf{w}^g_{ij,\ell}\|^2 + A_4
\end{align}
where we use the fact that $-\lambda\eta_g\frac{1}{J}\|\mathbf{w}^g_{ij,\ell} - \mathbf{w}^*\|^2 \leq -\lambda\eta_g (0.5\|\bar{\mathbf{w}}^g_{\ell}  + \mathbf{w}^*\|^2  - \sum_{i\in\mathcal{I}}\sum_{j\in\mathcal{J}_i}\|\bar{\mathbf{w}}^g_{\ell}  - \mathbf{w}^g_{ij,\ell}\|^2)$, $\|\nabla \bar{F}_{ij}(\mathbf{w}^g_{ij,\ell})\|^2 \leq 2\mu \bigl(\bar{F}_{ij} (\mathbf{w}^g_{ij,\ell}) - F_{ij}^*  \bigr)$, and $A_4 \triangleq 4\mu\eta_g^2\frac{1}{J}\sum_{i\in\mathcal{I}}\sum_{j\in\mathcal{J}_i}\bigl(\bar{F}_{ij} (\mathbf{w}^g_{ij,\ell}) - F_{ij}^*  \bigr) - 2\eta_g\frac{1}{J}\sum_{i\in\mathcal{I}}\sum_{j\in\mathcal{J}_i} \bigl(\bar{F}_{ij}(\mathbf{w}^g_{ij,\ell}) - \bar{F}_{ij}(\mathbf{w}^*)\bigr)$.

\noindent To bound $A_4$, we define $\bar{\eta}_g=2\eta_g(1-2\mu\eta_g)$ and assume $\eta_g \leq \frac{1}{4\mu}$. Thus, we have $\bar{\eta}_g\in[\eta_g\ 2\eta_g]$ and hence
\begin{align}
    A_4 &= -2\eta_g(1-2\mu\eta_g)\frac{1}{J}\sum_{i\in\mathcal{I}}\sum_{j\in\mathcal{J}_i}\bigl(\bar{F}_{ij} (\mathbf{w}^g_{ij,\ell}) - F_{ij}^*  \bigr) + 2\eta_g\frac{1}{J}\sum_{i\in\mathcal{I}}\sum_{j\in\mathcal{J}_i} \bigl( \bar{F}_{ij}(\mathbf{w}^*) - F_{ij}^*\bigr) \nonumber\\
    &= - \bar{\eta}_g\frac{1}{J}\sum_{i\in\mathcal{I}}\sum_{j\in\mathcal{J}_i}\bigr(\bar{F}_{ij} (\mathbf{w}^g_{ij,\ell}) -\bar{F}_{ij}(\mathbf{w}^*)  \bigr) + (2
    \eta_g-\bar{\eta}_g)\frac{1}{J}\sum_{i\in\mathcal{I}}\sum_{j\in\mathcal{J}_i}\bigl( \bar{F}_{ij}(\mathbf{w}^*) - F_{ij}^*\bigr) \nonumber\\
    &\leq - \bar{\eta}_g\frac{1}{J}\sum_{i\in\mathcal{I}}\sum_{j\in\mathcal{J}_i}\bigr(\bar{F}_{ij} (\mathbf{w}^g_{ij,\ell}) -\bar{F}_{ij}(\mathbf{w}^*)  \bigr) + 4\mu \eta_g^2\frac{1}{J}\sum_{i\in\mathcal{I}}\sum_{j\in\mathcal{J}_i}\varepsilon_{ij} \ (\text{by Definition \ref{def:1}}).
\end{align}
Further, it is noted that
\begin{align}\label{eq:a12}
    \frac{1}{J}\sum_{i\in\mathcal{I}}\sum_{j\in\mathcal{J}_i}\bigr(\bar{F}_{ij} (\mathbf{w}^g_{ij,\ell}) -\bar{F}_{ij}(\mathbf{w}^*)  \bigr) = \frac{1}{J}\sum_{i\in\mathcal{I}}\sum_{j\in\mathcal{J}_i}\bigr(\bar{F}_{ij} (\mathbf{w}^g_{ij,\ell}) - \bar{F}_{ij} (\bar{\mathbf{w}}^g_{\ell}) + \bar{F}_{ij} (\bar{\mathbf{w}}^g_{\ell}) - \bar{F}_{ij}(\mathbf{w}^*)  \bigr).
\end{align}
Applying Assumption \ref{assp:1} and $\|\nabla \bar{F}_{ij}(\mathbf{w}^g_{ij,\ell})\|^2 \leq 2\mu \bigl(\bar{F}_{ij} (\mathbf{w}^g_{ij,\ell}) - F_{ij}^*  \bigr)$ to \eqref{eq:a12}, after some manipulations we can obtain
\begin{align}\label{eq:a13}
    A_4 \leq& \bar{\eta}_g(\eta_g\mu-1)\frac{1}{J}\sum_{i\in\mathcal{I}}\sum_{j\in\mathcal{J}_i}\bigl(\bar{F}_{ij} (\bar{\mathbf{w}}^g_{\ell}) - \bar{F}_{ij}(\mathbf{w}^*) \bigr) \nonumber\\
    &+ \frac{1}{J}\sum_{i\in\mathcal{I}}\sum_{j\in\mathcal{J}_i}\|\mathbf{w}^g_{ij,\ell}-\bar{\mathbf{w}}^g_{\ell}\|^2
     + 6\mu \eta_g^2\frac{1}{J}\sum_{i\in\mathcal{I}}\sum_{j\in\mathcal{J}_i}\varepsilon_{ij}.
\end{align}
Substituting \eqref{eq:A10} and \eqref{eq:a13} into \eqref{eq:A6}, we have
\begin{align}\label{eq:A14}
    \|\bar{\mathbf{w}}^g_{\ell+1} - \mathbf{w}^*\|^2  \leq
     (1-0.5\lambda \eta_g )\|\bar{\mathbf{w}}^g_{\ell}  - \mathbf{w}^*\|^2 
   + \eta_g^2\Omega^g_{\ell} + 2\eta_g\frac{1}{J}\sum_{i\in\mathcal{I}}\sum_{j\in\mathcal{J}_i}\bigl(\bar{F}_{ij}(\mathbf{w}^*)-\bar{F}_{ij} (\bar{\mathbf{w}}^g_{\ell})  \bigr)
\end{align}
due to $\eta_g \leq 1/4\mu$ and  $\bar{\eta}_g(1-\eta_g\mu) \leq 2\eta_g$,   where $\Omega^g_{\ell}\triangleq \frac{(2+\lambda/4\mu)}{\eta_g^2}\frac{1}{J}\sum_{i\in\mathcal{I}}\sum_{j\in\mathcal{J}_i}\|\bar{\mathbf{w}}^g_{\ell}  - \mathbf{w}^g_{ij,\ell}\|^2+\|\nabla \bar{F}(\bar{\mathbf{w}}^g_{\ell})-\nabla F(\bar{\mathbf{w}}^g_{\ell})\|^2 +  6\mu \frac{1}{J}\sum_{i\in\mathcal{I}}\sum_{j\in\mathcal{J}_i}\varepsilon_{ij}.$ By Lemmas \ref{lema:1} and \ref{lema:2}, the expectation of $\Omega^g_{\ell}$ is $\bar{\Omega}^g_{\ell}=\mathbb{E}\{\Omega^g_{\ell}\}= (2+\lambda/4\mu)(L-1)L\delta^2 + \frac{\sum_{i\in\mathcal{I}}\sum_{j\in\mathcal{J}_i}\gamma_{ij}^2}{J^2} + 6\mu \frac{1}{J}\sum_{i\in\mathcal{I}}\sum_{j\in\mathcal{J}_i}\varepsilon_{ij}.$

\section{Proof of Theorem \ref{theo:1}} \label{app: theo1}
\renewcommand{\theequation}{\ref{app: theo1}.\arabic{equation}}\setcounter{equation}{0}
Similar to \cite{ruan2020flexible,LiICLR2020}, let us define $Q^g_{\ell}\triangleq\|\bar{\mathbf{w}}^g_{\ell} - \mathbf{w}^*\|^2$ and $\bar{Q}^g_{\ell}=\mathbb{E}\{Q^g_{\ell}\}$. From \eqref{eq:A14}, we have
\begin{align}\label{eq:B1}
 \sum_{\ell=1}^LQ^g_{\ell} \leq
     \sum_{\ell=0}^{L-1} (1-0.5\lambda \eta_g )Q^g_{\ell} 
   + \eta_g^2\sum_{\ell=0}^{L-1}\Omega^g_{\ell} + 2\eta_g\frac{1}{J}\sum_{i\in\mathcal{I}}\sum_{j\in\mathcal{J}_i}\bigl(\bar{F}_{ij}(\mathbf{w}^*)-\bar{F}_{ij} (\bar{\mathbf{w}}^g_{\ell'})  \bigr)
\end{align}
where $\bar{\mathbf{w}}^g_{\ell'}$ is a minimizer of $\frac{1}{J}\sum_{i\in\mathcal{I}}\sum_{j\in\mathcal{J}_i}\bar{F}_{ij} (\bar{\mathbf{w}}^g_{\ell})$, leading to $\sum_{i\in\mathcal{I}}\sum_{j\in\mathcal{J}_i}\mathbb{E}\{\bar{F}_{ij}(\mathbf{w}^*)-\bar{F}_{ij} (\bar{\mathbf{w}}^g_{\ell'})\}\leq 0$. By $\Omega^g\triangleq \sum_{\ell=0}^{L-1}\Omega^g_{\ell}$, we rewrite \eqref{eq:B1}  as
\begin{align}\label{eq:B2}
 Q^{g+1} \leq Q^{g}    - 0.5 \lambda \eta_g\sum\nolimits_{\ell=0}^{L-1}Q^g_{\ell} 
   + \eta_g^2\Omega^g.
\end{align}	
In addition, $\|\bar{\mathbf{w}}^g_{\ell +1} - \mathbf{w}^*\| \leq \|\bar{\mathbf{w}}^g_{\ell} - \mathbf{w}^*\| + \eta_g\|\nabla F(\bar{\mathbf{w}}^g_{\ell})\|$ or $\|\bar{\mathbf{w}}^{g+1} - \mathbf{w}^*\| \leq \|\bar{\mathbf{w}}^g_{\ell} - \mathbf{w}^*\| + \eta_g\sum_{k=\ell}^{L-1}\|\nabla F(\bar{\mathbf{w}}^g_{k})\|$. As a result, we can show that $Q^g_{\ell} \leq 0.5Q^{g+1}-\eta_g^2(\sum_{k=0}^{L-1}\|\nabla F(\bar{\mathbf{w}}^g_{k})\|)^2$. Hence, \eqref{eq:B2} is rewritten as
\begin{align}\label{eq:B3}
 (1+0.25\lambda\eta_g L)Q^{g+1} \leq Q^{g}    + 0.5 \lambda \eta_g^3 L \bigr(\sum\nolimits_{k=0}^{L-1}\|\nabla F(\bar{\mathbf{w}}^g_{k})\|\bigl)^2
   + \eta_g^2\Omega^g.
\end{align}	
Assuming the learning rate $\eta_g\leq 4/\lambda L$ and taking the expectation of both sides of \eqref{eq:B3}, we have  
\begin{align}\label{eq:B4}
 \bar{Q}^{g+1} \leq \bigl(1-\frac{\eta_g\lambda L}{8}\bigr) \bar{Q}^{g}    + \eta_g^2\bigl(2\bar{\Phi}^g + \bar{\Omega}^g \bigr) 
\end{align}	
where $\bar{\Phi}^g\triangleq\mathbb{E}\{\bigr(\sum\nolimits_{k=0}^{L-1}\|\nabla F(\bar{\mathbf{w}}^g_{k})\|\bigl)^2\} \leq L\sum\nolimits_{k=0}^{L-1}\mathbb{E}\{\|\frac{1}{J}\sum_{i\in\mathcal{I}}\sum_{j\in\mathcal{J}_i}\nabla F_{ij}(\mathbf{w}^g_{ij,\ell})\|^2\} \leq L^2\delta^2$ by Cauchy–Schwarz inequality and Assumption \ref{assp:3}, and $\bar{\Omega}^g =  (2+\lambda/4\mu)(L-1)L\delta^2 + \frac{L\sum_{i\in\mathcal{I}}\sum_{j\in\mathcal{J}_i}\gamma_{ij}^2}{J^2} + 6\mu L \frac{1}{J}\sum_{i\in\mathcal{I}}\sum_{j\in\mathcal{J}_i}\varepsilon_{ij}$.

We  consider a diminishing learning rate $\eta_g = \frac{16}{\lambda(g+1+\psi)}$ with $\psi = \max\bigl\{\frac{64\mu}{\lambda},\ 4L \bigr\} >0$, satisfying $\eta_1 \leq \min\{\frac{1}{4\mu},\ \frac{4}{\lambda L}\}$. By defining $\Psi^g \triangleq \max\bigl\{\psi^2\mathbb{E}\{\|\mathbf{w}^0 - \mathbf{w}^*\|^2\}, \frac{16^2}{\lambda^2}g\Theta \bigr\}$ for $\Theta\triangleq 2L^2\delta^2+(2+\lambda/4\mu)(L-1)L\delta^2 + \frac{L\sum_{i\in\mathcal{I}}\sum_{j\in\mathcal{J}_i}\gamma_{ij}^2}{J^2} + 6\mu L \frac{1}{J}\sum_{i\in\mathcal{I}}\sum_{j\in\mathcal{J}_i}\varepsilon_{ij}$, we next will prove that $\bar{Q}^{g} \leq \frac{\Psi^g}{(g+\psi)^2}$, which is done by induction. It follows that
\begin{align}
   \bar{Q}^{g+1} &\leq \bigl(1-\frac{16}{\lambda(g+1+\psi)}\frac{\lambda L}{8}\bigr) \frac{\Psi^g}{(g+\psi)^2}   + \bigl(\frac{16}{\lambda(g+1+\psi)}\bigr)^2\Theta \nonumber\\
   & \leq \frac{g+\psi-L}{(g+\psi-L)(g+\psi+L)}\frac{\Psi^g}{g+1+\psi}  + \frac{1}{(g+1+\psi)^2}\frac{16^2}{\lambda^2}\Theta\ (\text{due to}\  1-L \leq 0)\nonumber\\
   &= \frac{1}{(g+1+\psi)^2}\bigl(\Psi^g+\frac{16^2}{\lambda^2}\Theta\bigr) =\frac{\Psi^{g+1}}{(g+1+\psi)^2}
\end{align}
where $\Psi^g \triangleq \max\bigl\{\psi^2\mathbb{E}\{\|\mathbf{w}^0 - \mathbf{w}^*\|^2\}, \frac{16^2}{\lambda^2}(g+1)\Theta \bigr\}$.

\begingroup
\setstretch{1.10}
\bibliographystyle{IEEEtran}
\bibliography{IEEEfull}
\endgroup

\end{document}